\documentclass[11pt]{article}

\usepackage{chngcntr} 
\usepackage{amsmath}
\usepackage{enumerate}
\usepackage{fancyhdr}
\usepackage{color}
\usepackage{listings}
\usepackage{graphicx}
\usepackage{amssymb}
\usepackage{mathtools}
\usepackage{natbib}
\usepackage{hyperref}
\usepackage{fullpage}
\usepackage{booktabs}

\usepackage{xcolor}

\usepackage[utf8]{inputenc} %
\usepackage[T1]{fontenc}    %
\usepackage{hyperref}       %
\usepackage{url}            %
\usepackage{booktabs}       %
\usepackage{amsfonts}       %
\usepackage{nicefrac}       %
\usepackage{microtype}      %
\usepackage{xcolor}         %

\usepackage{macros}
\usepackage{times}
\title{Flatness and Generalization: Learning Multi-Index Models with Homogeneous Neural Networks}

\author{%
  Harsh Vardhan \\
  Department of Computer Science\\
  University of California, San Diego\\
  La Jolla, CA 92092 \\
  \texttt{hharshvardhan@ucsd.edu} \\
  \and
Hossein Taheri\\
  Department of Computer Science \\
  University of California, San Diego\\
  La Jolla, CA 92092 \\
  \texttt{htaheri@ucsd.edu}
  \and
  Arya Mazumdar \\
  Halicio\u glu Data Science Institute \\
  University of California, San Diego\\
  La Jolla, CA 92092 \\
  \texttt{arya@ucsd.edu} \\
}

\begin{document}

\maketitle

\begin{abstract}%
A common heuristic used to explain the generalization of first-order gradient methods on non-convex neural networks is that ``flat interpolators generalize well'' ~\citep{hochreiter,keskar2017on}, where flatness can be measured by the trace of the Hessian of the empirical loss.  However,  ~\citep{pmlr-v70-dinh17b} showed that any interpolator can be made sharper or flatter using the {\em symmetry of the network} that can change flatness while keeping the population and empirical losses unchanged. This result makes the earlier heuristic statement vacuous.   In this paper, we show that for learning an unknown multi-index model with $2$-layer non-convex homogeneous neural networks, there is a connection between flatness and generalization, despite the existence of symmetries. This connection pertains to the ``flattest'' interpolators, i.e., the interpolators that have orderwise minimum flatness among all interpolators.  First, we show 
that there exists a natural class of  non-generalizing interpolators 
whose flatness cannot be made closer to the flattest possible, even using symmetries. Second, we show that for data generated by a sum of single-index models, if the approximation error and label noise are low, any flattest interpolator achieves small population loss, i.e., the flattest interpolators always generalize. This establishes a direct link between flatness and generalization which applies to a large class of activations and realistic data distributions. 

\end{abstract}
\setcounter{tocdepth}{2}
\tableofcontents
\section{Introduction}
\label{sec:intro}
In machine learning, overparameterization refers to the paradigm where there are more model parameters than available data samples, such as in modern implementations of neural networks (NN). In this paradigm, the global minimizers of empirical (training) loss can interpolate, or perfectly fit the training data. However, due to the high-dimensional and non-convex landscape of the empirical loss, there are often ``bad"  interpolators, i.e., interpolators that have large population (test) loss ~\citep{zhang_understanding,nagarajan}. Surprisingly, stochastic gradient methods on these overparametrized models usually converge to ``good" interpolators, which generalize well. This empirical observation has not yet been rigorously justified, but one possible explanation for it, from ~\citep{hochreiter} initially, is the following: \textit{stochastic gradient methods converge to flat interpolators of the empirical loss and flat interpolators of empirical loss always generalize}. In this paper, sharpness of an interpolator is the trace of the Hessian of the empirical loss. Flat interpolators have small trace. In this work, we will try to theoretically justify the second part of this explanation, which states that flat interpolators generalize.

Several existing works have tried to justify the link between flatness and generalization. Experimentally, flatness is strongly correlated to generalization~\citep{Jiang*2020Fantastic,keskar2017on} and algorithms that explicitly minimize some notion of flatness generalize better~\citep{keskar2017on,chiang2023loss,DBLP:conf/iclr/ForetKMN21}. However, there is also evidence refuting this claim.  ~\cite{pmlr-v202-andriushchenko23a} shows that the correlation between flatness and generalization depends on the choice of model and ~\cite{wen2023how,schliserman2025flatminimageneralizationinsights} construct flat interpolators that do not generalize.

Notably, an important setting where this claim can be refuted is in the presence of ``symmetries'', i.e., transformations of model weights that do not change input-output relationship. The following well-known result from~\cite{pmlr-v70-dinh17b} summarizes the impact of such symmetries on flatness. 
\begin{proposition}[Symmetries for ReLU Neural Networks ~\citep{pmlr-v70-dinh17b} (see Def~\ref{def:homogen})]\label{lem:dinh}
    There exist symmetries that rescale the layer weights of a NN  such that the model output remains the same for a given input (thus, empirical and population losses remain same), however, the flatness changes.
\end{proposition}

In particular, ~\cite{pmlr-v70-dinh17b} use this property to show that one can make any interpolator sharper using symmetry, which implies that good interpolators can be sharp. Using symmetry, as in Prop.~\ref{lem:dinh}, one can also make any interpolator flatter. This contradicts the claim ``flat interpolators always generalize", as we can make bad interpolators flat. But, how flat can we make these bad interpolators? Can we make bad interpolators the flattest interpolators, i.e., those with the minimum trace of Hessian among all interpolators? Or is there a fundamental barrier that prohibits the symmetries from making certain bad interpolators achieve the minimum flatness?
A necessary condition for this claim to be true is to show there are bad interpolator
that cannot be made the flattest interpolators in the presence of symmetry. This is the first question we consider in this paper:

\begin{align*}
  \textit{\textbf{Q1}: Are there bad interpolators that  cannot be made the flattest interpolators using symmetries ?}   
\end{align*}

If the answer to this question is no, then there is a class of bad interpolators that can be made flattest, and there is no link between flatness and generalization in the presence of symmetry. However, if the answer to this question is yes, then symmetry cannot arbitrarily decrease the flatness of every interpolator. Therefore, the set of flattest interpolators does not contain a certain class of bad interpolators. But, when does the set of flattest interpolators does not contain any bad interpolator? This forms our second question, and a \emph{sufficient} condition for our claim. 

\begin{align*}
  \textit{\textbf{Q2}: Do all flattest interpolators generalize?}  
\end{align*}

We consider the feature-learning setting to answer these two questions. In our setting, the true labels come from a multi-index model, and we use a $2$-layer neural network with homogeneous non-linear activations to learn it. This is a popular regression problem, where the labels depend only on a few directions ($m^\star \ll d$) of the features, specified by the matrix $\Theta^\star \in \bR^{m^\star \times d}$, via an unknown link function $\sigma^\star$. The $2$-layer network has $m$ hidden neurons with activation $\sigma:\bR\rightarrow\bR$ and the weight vector $\vw\in \bR^{m(d+1)}$ consists of outer-layer weights $\va\in \bR^m$, and inner-layer weights $\{\theta_j\}_{j\in [m]}, \theta_j\in \bR^{d},\forall j\in [m]$. We sample $n \asymp d$ datapoints from the distribution and ensure overparametrization by setting $m \geq n$. Let $S$ be the set of training samples, and $F_{S}:\bR^{m(d+1)}\to\bR_{+}$ be the empirical loss and $F:\bR^{m(d+1)}\to \bR_{+}$ be its population loss. Then, an interpolator achieves $F_S(\vw) = 0$ and its flatness is given by $\trace(\nabla^2 F_S(\vw))$.

Even in this simplified setting, the answer to the questions posed earlier is highly non-trivial. This setup i) has both good and bad interpolators, ii) it is non-linear and non-convex, iii) it has a symmetry that can change flatness without changing population loss, however, iv) gradient-based methods provably generalize for the given sample complexity~\citep{oko2024neural,ben_arous}. Further, existing flatness-based generalization bounds either do not apply to these settings~\citep{qiao2024stable,pmlr-v272-haddouche25a, pmlr-v202-wu23r, liang2025generalizationedgestabilityrole} or are too loose~\citep{liang2025stable}. Additionally, this setting is a more representative example of real-world NN training, than the worst-case distributions considered as counter-examples in  ~\citep{wen2023how,schliserman2025flatminimageneralizationinsights}. In Section~\ref{sec:comparison}, we provide a comparison of our results to these works.

\subsection{Main Contributions}
Our main contributions are the answers to our two questions posed earlier.

\paragraph{Non-generalizing  Interpolators cannot be Flattest.}
Our first contribution is a positive answer to \textbf{Q1}, summarized informally in the following Theorem.
\begin{informal_theorem}[Thms.~\ref{thm:flattest_all}-~\ref{thm:flattest_bad}] \label{inf_thm:bad_flat}
We characterize a set of bad interpolators in Def.~\ref{def:bad_min}, such that for an interpolator $\vw \in \bR^{m(d+1)}$ in this set,  its flatness is at least $n^{\frac{c}{c+1}}$ times larger than the minimum flatness of any interpolator, with high probability, i.e., $\trace(\nabla^2 F_S(\vw))\gtrsim n^{\frac{c}{c+1}}\Upsilon^\star$. Here, $\Upsilon^\star$ is the flatness of the flattest interpolator and $c\geq 1$ quantifies the symmetry. 

\end{informal_theorem}
The class of bad interpolators considered above have their inner-layer weights, $\{\theta_j\}_{j\in [m]}$, misaligned with the true weights, $\Theta^\star$, that generated the responses. Due to overparametrization, $m\geq n$, such interpolators always exist.  However, poor alignment to $\Theta^\star$ forces a large population loss for these interpolators. Symmetry cannot decrease the flatness of these bad interpolators below a certain value. 

\paragraph{Flattest Interpolators Generalize.}
Our second contribution is a positive answer to \textbf{Q2}, but under certain assumptions. We summarize it informally in the following theorem.
\begin{informal_theorem}[Thm.~\ref{thm:gen_all_good}]\label{inf_thm:good_flat}

Under low approx. error and low label noise,  an interpolator $\vw \in \bR^{m(d+1)}$ that achieves the minimum flatness up to constants, i.e., $\trace(\nabla^2 F_S(\vw)) \asymp \Upsilon^\star$, generalizes. In particular, $F(\vw) \lesssim n^{-\min\{\frac{1}{2}, \epsilon_1, \epsilon_2\}}$, for some constants $\epsilon_1>0$ and $\epsilon_2>0$ defined in Assump.~\ref{assump:Lipschitz-like}.

\end{informal_theorem}
We note that establishing the sufficient condition for flattest interpolator to generalize requires two main assumptions beyond those used for the necessary condition -- namely it requires low approximation error and low label noise, whose precise expressions depending on $\epsilon_1$ and $\epsilon_2$ are provided in Assumption~\ref{assump:Lipschitz-like}.

Our current analysis covers both single-index and sum of single-index models and most homogeneous activations like ReLU, LeakyReLU or quadratic activation. Therefore, for our setting of learning a multi-index model with $2$-layer neural networks under symmetry, there is a relationship between flatness and generalization, as the flattest interpolators always generalize. However, the original claim in the previous works~\citep{liang2025stable,pmlr-v202-wu23r} is not precise, as interpolators with an arbitrary value of flatness might not generalize due to symmetry. Our results also show benign overfitting for the flattest interpolator for $2$-layer homogeneous NN with multi-index data, while previous works~\citep{ding24flat} could only show this for specific instances of our setting.

\subsection{Comparison to Related Works}
\label{sec:comparison}
In this section, we compare our results with the most relevant existing works and defer a more comprehensive comparison to App~\ref{sec:related_works}. The closest result to ours is ~\cite[Thm~7.1]{ding24flat}. They show that for Quadratic activations ($c=2$), when learning a sum of single-index models, where the unknown link function is the same as the activation, the flattest interpolator generalizes. Note that this is a special case of our result (cf. Inf. Thm.~\ref{inf_thm:good_flat}), with both zero approximation error and label noise. However, our main results (Thms~\ref{thm:flattest_all},\ref{thm:flattest_bad}, and \ref{thm:gen_all_good}) are much stronger than theirs, since (i) we consider homogeneous activations which includes the ReLU, LeakyReLU and Quadratic activation, (ii) we allow non-zero label noise, (iii) the link function and the model's activation can be different $\sigma^\star\neq \sigma$, and (iv) we characterize both flattest good as well as bad interpolators. One of their key insights is a ``balancedness condition'' between layer weights. Interestingly, our results exhibit a more general form of balancedness (see Prop.~\ref{rem:necessary_flatness}). Further, their proof techniques focus on equivalence to matrix factorization, however, our proofs are significantly different since we use concentration arguments to account for general non-quadratic homogeneous activations with noise and approximation errors.

Two counter-examples where there are bad interpolators that are flattest are provided in ~\citep{schliserman2025flatminimageneralizationinsights} for smooth convex optimization, and in~\citep{wen2023how} for learning a $2$-layer NN with bias for the XOR problem, a classification variant of our multi-index models. These examples are not contradictory to our findings, as both these counter-examples have specific \emph{discrete} data distributions on the features $\vx$. In particular,  ~\cite{schliserman2025flatminimageneralizationinsights} sample features from $\{0,1\}^d$, while ~\cite{wen2023how} sample features from $\{\pm1\}^d$. Their proof techniques rely heavily on these discrete distributions -- ~\cite{schliserman2025flatminimageneralizationinsights} require some coordinate of all features to be $0$, and ~\cite{wen2023how} require each feature to be an endpoint of the convex hull of all sampled features. For our Gaussian features, these events do not occur with high probability for $n\asymp d$, which is the optimal sample complexity for this problem~\citep{pmlr-v125-chen20a}. Moreover, the counter-example in ~\citep{schliserman2025flatminimageneralizationinsights} uses handcrafted loss functions that do not correspond to real tasks like regression or classification on a real data distribution. Additionally, our proofs and results in all theorems for $2$-layer NN without bias can be extended to $2$-layer networks with bias, as used in ~\citep{wen2023how}, for Gaussian features. In this case, the flattest interpolator with bias still corresponds to the case of $0$-bias for both bad and good interpolators. We provide a brief proof for this in App.~\ref{sec:wen_comparison_proof}. 

\paragraph{Organization.}
In Section~\ref{sec:setup}, we formally define our setup, including the symmetry induced by homogeneous activations. In Section~\ref{sec:bad_min_not_flattest}, we define the class of bad interpolators and characterize their flatness.In Section~\ref{sec:flattest_is_good}, we show that flattest interpolators generalize. We summarize our proof techniques in Section~\ref{sec:proof_sketch} and our conclusions in Section~\ref{sec:discussion}. We perform experiments for verifying our theoretical results in Section~\ref{sec:experiments}.

\section{Setup}
\label{sec:setup}
\paragraph{Notation.}
We use $[n]$ to denote the set $\{1,2,\ldots,n\}$. We use $\cO, o, \Omega$ and $\Theta$ to denote the usual complexity notation. For example, if a term $b=\cO(1)$, then it is order-wise at most a constant with respect to $m,n$ and $d$ . All bounds that hold with high probability hold with probability $1 - C\delta - \mathrm{poly}(m^{-b_1}, n^{-b_2}, d^{-b_3})$ for constants $C\geq 1$ and $\delta,b_1,b_2,b_3>0$. We use $\gtrsim, \lesssim$ and $\asymp$ to represent greater than, less than or equal order-wise up to poly-logarithmic factors of problem parameters and $\delta^{-1}$. $\mathrm{poly}$ is used to denote any bounded-degree polynomial in its arguments. We use $\vecspan(Q)$ to denote the linear span of a set of vectors $Q\subset \bR^{d}$, and $\bS^{d-1}$ to denote the unit sphere in $d$ dimensions. 

\paragraph{Data, Model \& Losses.} Each data point is $\vz = (\vx, y)\in  \bR^d \times \bR$,  where $\vx \sim \cN(0,\bI_d)$ are the features and $\vy$ is the label. We generate the labels as $y = \sigma^\star(\Theta^\star\vx) + \xi$ for some continuous and differentiable almost everywhere function $\sigma^\star : \bR^{m^\star} \to \bR$, a fixed  orthonormal matrix $\Theta^\star\in \bR^{m^\star \times d}$ such that $\Theta^\star(\Theta^\star)^\top  = \bI_{m^\star}$ where the $j^{th}$ row of $\Theta^\star$ is denoted by the vector $\theta_j^\star \in \bS^{d-1}$, and label noise $\xi\sim \cN(0, \zeta^2)$  is independent of $\vx$. We assume that $m^\star =\cO(1)$, as it is a multi-index model.Two common and well-known examples that are covered under our problem setup are the following --
\begin{enumerate}[nosep]
    \item Single-Index Models~\citep{oko2024neural}:  We set $m^\star=1$, and $y = \sigma^\star(\ip{\theta^\star}{\vx})+\xi$ for $\theta^\star\in \bS^{d-1}$.
    \item Sum of Single-Index Models~\citep{pmlr-v247-oko24a}: We set $y = \sum_{j=1}^{m^\star}a_j^\star \tilde{\sigma}^\star(\ip{\theta_j^\star}{\vx})+\xi$ with $\tilde{\sigma}^\star : \bR\to \bR$ being a continuous and differentiable function, and $a_j^\star \in \bR$ with $a_j^\star = \cO(1)$ $\forall j\in [m^\star]$.
\end{enumerate}    
The activation $\sigma:\bR\to \bR$ is continuous and differentiable almost everywhere, the weights of the network are $\vw^\top := [a_1, a_2, \ldots,a_m, \theta_1^\top, \theta_2^\top, \ldots, \theta_m^\top]$, and its output for a given feature $\vx\in \bR^{d}$ is $h(\vw, \vx) = \sum_{j=1}^m a_j\sigma(\ip{\theta_j}{\vx})$. We use $\psi^\star$ and $\phi^\star$ to denote the power of the signal and activation, respectively.
\begin{align*}
    \psi^\star \coloneq \Ee{\vb\sim \cN(0, \bI_{m^\star})}{(\sigma^\star)^2(\vb)}, \quad \phi^\star \coloneq \Ee{b\sim \cN(0,1)}{\sigma^2(b)}, \quad \psi^\star, \phi^\star = \Theta(1),~ \zeta^2 = o(1).
\end{align*}
We will use $\vU\in \bR^{n\times m}$ to denote the matrix of activation outputs for a fixed $\vw$. So, $\vU_{i,j} = \sigma(\ip{\theta_j}{\vx_i}), \forall i\in [n], j\in [m]$. We use $S = \{\vz_i\}_{i\in [n]}$ to denote the set of data points. Using square loss, we define the empirical loss $F_S$, population loss $F$, and their minimizers as 
\begin{align*}
    F_S(\vw) = \frac{1}{n}\sum_{\vz_i\in S} f(\vw, \vz_i) ,\quad
    F(\vw) = \Ee{\vz}{f(\vw, \vz)}, \quad \cW_S^\star \coloneq \{\vw: F_S(\vw)=0\},\quad F^\star = \min_{\vw}\quad F(\vw).
\end{align*}
If the set $\cW_S^\star$ is non-empty, then interpolators exist. Throughout this paper, we will use  ``good" and ``bad" to quantify the excess population loss of a weight $\vw \in \bR^{m(d+1)}$. A model with weight vector $\vw$, is ``good"/``bad" if
        \begin{align*}
            F(\vw) - F^\star = o(1)\quad\textbf{(Good)};\quad\quad\quad
            F(\vw) - F^\star > \kappa\cdot(\psi^\star + \zeta^2),\text{ where } \kappa \text{ is a constant. }\quad \textbf{(Bad)}.
        \end{align*}
    From the definition of bad weights, we need their excess population loss to be lower-bounded by a constant ($\kappa$) multiple of the signal power ($\psi^\star$) plus the noise variance ($\zeta^2$). For example, if a weight has excess population loss upper-bounded by $\cO\br{\mathrm{poly}(n^{-b_1}, d^{-b_2}, m^{-b_3})} $ for any constants $b_1,b_2, b_3 >0$ is a good weight.

\paragraph{Activations \& Symmetry.}
Note that the symmetry in Prop.~\ref{lem:dinh} is a by-product of using homogeneous activations. A function $g:\bR\to\bR$ is $c$-Homogeneous if $\forall \alpha >0, b\in \bR$, $g(\alpha b) = \alpha^c g(b)$. Our activation function $\sigma$ for the $2$-layer NN is a special class of homogeneous activations satisfying the following assumption.
\begin{assumption}[Piece-wise Polynomial Activation]\label{assump:activation} For all $b\geq0$, $\sigma(b) = c'\abs{b}^c$, and for all $b<0$, $\sigma(b)=c''\abs{b}^c$,  where $c', c'' \in \bR$ are some constants such that at most one of them is $0$, $c'+c'' \neq 0$, and $c \geq 1$. 
\end{assumption}
The above assumption forces $\sigma$ to be $c$-Homogeneous. All even-degree monomials of degree $c$, as well as LeakyReLU, ReLU and ReLU$^c$~\citep{he2024expressivityapproximationpropertiesdeep} satisfy the above assumption. The above assumption does not include odd-degree monomials and monomials with $c<1$ due to technical difficulty in our proof, explained in Section~\ref{sec:proof_sketch}. It also does not include non-polynomial activations like tanh and sigmoid. We define the symmetry induced by homogeneous activations used in Prop.~\ref{lem:dinh}.
\begin{definition}[Rescaling Symmetry]\label{def:homogen}
For any $\alpha_j>0, j\in [m]$, and any $\vw \in \bR^{m(d+1)}$ consisting of weights $\{a_j\}_{j\in [m]}$ and $\{\theta_j\}_{j\in [m]}$, $\widetilde{\vw} \in \bR^{m(d+1)}$ with weights $\{a_j\alpha_j^{-c}\}_{j\in [m]}$ and $\{\alpha_j \theta_j\}_{j\in [m]}$ satisfies, $h(\vw, \vx) = h(\widetilde{\vw}, \vx), \,\,\forall \vx \in \bR^d$.
\end{definition}
This symmetry is in the space of model weights, and is quantified by the scalars $\{\alpha_j\}_{j\in [m]}$. Existing works~\citep{neyshabur, zhao2026symmetry} have extensively studied the impact of this symmetry on loss landscapes. As the output of the model does not change with this symmetry, both population and empirical loss remain the same. In the next section, we will see how this symmetry changes the flatness of an interpolator.

\subsection{Flatness under Rescaling Symmetry}
\label{sec:rescaling}
The flatness of an interpolator $\vw$ is defined as $\trace(\nabla^2 F_S(\vw))$. Note that several definitions have been used in literature for flatness~\citep{pmlr-v202-wu23r,pmlr-v70-dinh17b,schliserman2025flatminimageneralizationinsights}, among which ~\cite{wu_alignment_2022,ding24flat,gatmiry2023what,wen2023how} use our definition. For our setting, the flatness of an interpolator $\vw$ simplifies to $\trace(\nabla^2 F_S(\vw))= \frac{1}{n}\sum_{i=1}^{n} \norm{\nabla_{\vw} h(\vw, \vx_i)}_2^2$.

Due to Rescaling Symmetry (Def.~\ref{def:homogen}), $h(\vw, \vx)$ remains the same for different $\{\alpha_j\}_{j\in [m]}$, however, $\nabla_{\vw}h(\vw, \vx)$ changes, thus changing flatness. We first define the flattest interpolator obtained by applying rescaling symmetry to a specific interpolator $\vw$. To define this, we need to set a reference point for a given $\vw\in \bR^{m(d+1)}$. both the inner layer weights $\{\theta_j\}_{j\in [m]}$ and the outer layer weights $\{a_j\}_{j\in [m]}$ can be changed using rescaling symmetry. For our reference, we will fix the inner-layer weights to be unit norm, $\theta_j\in \bS^{d-1}, \,\forall j\in [m]$, and the outer-layer weights $\va$ chosen such that $\vw$ interpolates. The following Lemma quantifies $\Upsilon(\vw)$, the flatness of the flattest interpolator obtained by applying the rescaling symmetry to $\vw$.
\begin{lemma}[Flattest Interpolator under Rescaling Symmetry]\label{lem:flattest_rescaling}
If Assump.~\ref{assump:activation} holds, for an interpolator $\vw$ such that $\theta_j\in \bS^{d-1}, \,\forall j\in [m]$, with high probability,
\begin{align*}
    &\Upsilon(\vw) \coloneq \min_{\hat{\vw} \in \cW_{\text{rescale}}(\vw)}\trace(\nabla^2 F_S(\hat{\vw})) \asymp  d^{\frac{c}{c+1}}\norm{\va}_{\frac{2c}{c+1}}^{\frac{2c}{c+1}}, \\
    \text{where } \cW_{\text{rescale}}(\vw) =& \{\hat{\vw}^\top = \begin{bmatrix}
      \hat{a}_1, \ldots, \hat{a}_m, \hat{\theta}_1^\top, \ldots, \hat{\theta}_m^\top
  \end{bmatrix} \text{ s.t. } \hat{\theta}_j = \alpha_j\theta_j, \hat{a}_j =  \alpha_j^{-c}a_j,\forall\alpha_j>0,\,\forall j\in [m]\}.  
\end{align*}
\end{lemma}
Here, $\cW_{\text{rescale}}$ is the set of weights (or the equivalence class) obtained from $\vw$ after rescaling. $\Upsilon(\vw)$ is the minimum flatness in this equivalence class.  The proof of this lemma, provided in App.~\ref{sec:lem_flattest_rescaling_proof}, is quite direct, as the objective is a convex and separable function in the scalars $\{\alpha_j\}_{j\in [m]}$ that define the equivalence class. 

The implications of this lemma are substantial. First, the quantity $\Upsilon(\vw)$ removes the impact of symmetry. Second, it separates the contributions of inner and outer-layer weights. Here, the additional $d^{\frac{c}{c+1}}$ term is obtained from the inner-layer weights $\{\theta_j\}_{j\in [m]}$ being unit norm. The remaining term in $\Upsilon(\vw)$ is $\norm{\va}_{p}^p$, which is a convex $\ell_p$-norm with $p=\frac{2c}{c+1}\in [1,2)$ for $c\geq 1$. Third, Lemma~\ref{lem:flattest_rescaling} does not put any assumption on $\vw$ or the data distribution, apart from the fact that $\vw$ interpolates on the training set $S$.

As we have removed the contribution of rescaling symmetry, to check if a bad interpolator can be made flattest by rescaling, we need to evaluate  $\Upsilon(\vw)$ from Lemma~\ref{lem:flattest_rescaling} for a bad interpolator. In the next section, we consider a specific class of bad interpolators, where we can appropriately bound $\Upsilon(\vw)$ to answer \textbf{Q1}.

\section{Bad Interpolators of Multi-Index Models are Not Flattest}
\label{sec:bad_min_not_flattest}
In this section, we answer \textbf{Q1}, thus proving Inf. Thm.~\ref{inf_thm:bad_flat}. The proof involves two steps, the first being a bound on the minimum flatness of an interpolator (Thm~\ref{thm:flattest_all}) and the second being a lower bound on the flatness of bad interpolators (Thm.~\ref{thm:flattest_bad}).  We first define a natural class of bad interpolators for multi-index data distributions.

\begin{definition}[$(\rho^\star, \rho',\kappa)$-Bad Interpolator]\label{def:bad_min}
For $\rho^\star,\rho' \in [0,\frac{1}{2}]$, and a constant $\kappa\in (0,\frac{1}{2})$, we say that an interpolator $\vw\in \bR^{m(d+1)}$ is $(\rho^\star,  \rho',  \kappa)$-bad if,
\begin{enumerate}[nosep]
    \item $\forall j\in [m],\quad \theta_j \in \bS^{d-1}, \quad  \ip{\theta_j}{\theta_{j'}} \leq \rho', \forall j\neq j' \in [m]$.
    \item $F(\vw) - F^\star \geq \kappa(\zeta^2 + \psi^\star)$.
    \item $\abs{\ip{\vv_1}{\vv_2}} \leq \rho^\star, \quad \forall \;\vv_1\in \vecspan(\{\theta_j\}_{j\in [m]})\cap \bS^{d-1}, \forall\;  \vv_2 \in \vecspan(\{\theta_j^\star]_{j\in [m^\star]})\cap \bS^{d-1}$.
\end{enumerate}
\end{definition}

Let us understand each condition in the above definition. The first condition says that inner-layer weights are unit-norm and well-separated. The separation ensures that $\vw$ can provably interpolate (cf. Lemma~\ref{lem:interpolation}). The second condition ensures that the population loss is large, by a constant $\kappa$ for this interpolator, so that it is indeed a bad interpolator. The third condition requires any unit vector in the span of the inner-layer weights, $\vecspan(\{\theta_j\}_{j\in [m]})$, to be misaligned with all vectors in the span of true directions, $\vecspan(\{\theta_j^\star\}_{j\in [m]})$. A consequence of this condition is that each layer weight is also misaligned with the true direction, i.e., $\abs{\ip{\theta_j}{\theta_{j'}^\star}}\leq \rho^\star, \forall j\in [m], j'\in [m^\star]$. For multi-index models, this condition also ensures that the population loss is large. Note that the bad interpolators defined above always exist: we provide sufficient conditions for their existence and an example in App.~\ref{sec:proof_flattest_bad}.

 To obtain a lower bound on $\Upsilon(\vw)$ for an interpolator $\vw$, we state additional assumptions required on the unknown link function $\sigma^\star$ that generates the labels $y$ and the width of the network.

\begin{assumption}[Link Function]\label{assump:link}
    For Gaussian random vector, $\vb\sim \cN(\mu, \varsigma)$ where $\mu\in \bR^{m^\star}$ and $\varsigma\in \bR^{m^\star \times m^\star}$ is a Positive Semi-Definite matrix, the link function satisfies the following properties:
    \begin{align*}
        &\textbf{Mean}: \|\Ee{\vb}{\sigma^\star(\vb)}\|_2 =\Theta\br{\norm{\mu}_2^{c} + \norm{\varsigma}_F^{\frac{c}{2}}},\quad \quad \textbf{Variance} : \Varr{\vb}{\sigma^\star(\vb)}  = \Omega\br{\norm{\varsigma}_F \norm{\mu}_2^{2c-2} +  \norm{\varsigma}_F^{c}},\\
        &\textbf{Tails} : \sigma^\star(\vb)\text{ is  a }(\frac{c}{2}, \cO(\norm{\mu}_2^c + \norm{\varsigma}_F^{\frac{c}{2}}))\text{-Sub-Weibull random vector.}
    \end{align*}    
\end{assumption}

Note that a $(\varrho, K)$-Sub-Weibull random variable has its $k^{th}$ central moment bounded by $\cO(K k^{\varrho})$ ~\citep{Vladimirova_2020,sub_weibull,sharper_sub_weibull} for $\rho,K>0$. Sub-Gaussian and Sub-exponential distributions are also $(\varrho, K)$-Sub-Weibull with $\varrho = \frac{1}{2}$ and $\varrho = 1$ respectively, with $K$ being their Sub-exponential or Sub-Gaussian parameters. We provide an overview of Sub-Weibull distributions in App.~\ref{sec:link_func}. The above assumption forces a polynomial bound on the link function, which is in fact satisfied by our activations in   Assump.~\ref{assump:activation}. 
\begin{assumption}[Width and Inner-Layer Dependence]\label{assump:width}
    Let $C\leq m \leq C' n$ for some constants, $C'>C>1$ and $\Omega(1) = \lambda_{\min}(\bar{\Phi}) < \lambda_{\max}(\bar{\Phi}) = \cO(1)$, where $\bar{\Phi}\in \bR^{m\times m}$ is the covariance matrix of activations, $\bar{\Phi} = \frac{1}{n}\E{(\vU-\E{\vU})^\top(\vU-\E{\vU})}$.
\end{assumption}
Note that the above assumption ensures that the width of the network is only a constant times larger than $n$, which ensures interpolation. Further, we also assume that the covariance matrix of activations is well-conditioned, which forces all inner-layer weights to not be too similar. This is a technical condition which we explain in Section~\ref{sec:proof_sketch}, and is satisfied by Def.~\ref{def:bad_min} automatically.

Our first result is a lower bound on $\Upsilon(\vw)$  for any interpolator.
\begin{theorem}[Flattest Interpolator]\label{thm:flattest_all}
If Assump.~\ref{assump:activation},~\ref{assump:link} and ~\ref{assump:width} hold, then, for an interpolator $\vw$, with high probability,   
$\Upsilon(\vw) \gtrsim \Upsilon^\star\coloneq d^{\frac{c}{c+1}}m^{-\frac{c-1}{c+1}}$.
\end{theorem}
We will compare $\Upsilon^\star$ derived above with $\Upsilon(\vw)$ for a bad interpolator as defined in Def.~\ref{def:bad_min}.
\begin{theorem}[Flattest Bad Interpolator]\label{thm:flattest_bad}
Suppose Assump.~\ref{assump:activation},~\ref{assump:link} and \ref{assump:width} hold . Then, for a  $(\rho^\star, \rho', \kappa)$-bad interpolator $\vw$ (Def. ~\ref{def:bad_min}), with high probability,
\begin{align*}
    \Upsilon(\vw)  \gtrsim 
        (dn)^{\frac{c}{c+1}}m^{-\frac{c-1}{c+1}},~~ \text{if~~ } \rho^\star = o((m^\star)^{-\frac{1}{4}}).
\end{align*}
\end{theorem}
Note that $\Upsilon^\star$ is a lower bound, however, for noiseless case $\zeta=0$ and when learning activation $\sigma^\star =\sigma$ for single-index or sum of single-index models, this lower bound is achievable(cf. Thm~\ref{thm:flattest_good}).  When the correlation $\rho^\star$ is strictly smaller than a constant, the flatness of a bad interpolator in Def.~\ref{def:bad_min} is at least $n^{\frac{c}{c+1}}$ times larger than the $\Upsilon^\star$. This proves Inf. Thm.~\ref{inf_thm:bad_flat}, thus satisfying the necessary condition for our claim ``flat interpolators generalize".  Since $n = \Omega(1)$, these bad interpolators are indeed very sharp, and even using rescaling symmetry, we cannot make them as flat as possible.  This completes the picture for the argument of ~\citep{pmlr-v70-dinh17b} with respect to rescaling symmetry: while good interpolators can be made as sharp as possible, not all bad interpolators can be made as flat as possible.

Note that Thms.~\ref{thm:flattest_all} and ~\ref{thm:flattest_bad} do not prevent other bad interpolators, for instance those with large $\rho^\star$ or not satisfying Def.~\ref{def:bad_min}, from achieving the minimum flatness $\Upsilon^\star$. To answer \textbf{Q2}, we need to fully characterize the class of interpolators achieving the minimum flatness $\Upsilon^\star$. We do this in the next section.

\section{Flattest Interpolators of Sum of Single Index Models Generalize}
\label{sec:flattest_is_good}
In this section, we will answer \textbf{Q2}, showing that under some assumptions, flattest interpolators always generalize. To prove this result, we provide a tight characterization of the set of flattest interpolators in terms of the outer-layer weights, $\va$. 

\begin{proposition}[Necessary and Sufficient Condition for Flattest Interpolator]\label{rem:necessary_flatness}
Suppose Assump.~\ref{assump:activation},~\ref{assump:link} and ~\ref{assump:width} hold. Then, for any interpolator $\vw$ with $\theta_j\in \bS^{d-1}, \forall j\in [m]$, it holds with high probability, $\Upsilon(\vw) \asymp \Upsilon^\star$ if and only if $\norm{\va}_\infty =\cO(m^{-1})$.
\end{proposition}

This condition on $\norm{\va}_\infty$ allows us a clean characterization of $\Upsilon^\star$. However, as $\Upsilon^\star$ is a lower bound on $\Upsilon(\vw)$, we do not know if there are interpolators that actually achieve this lower bound. To show that such interpolators do exist, we need additional assumptions on the data distribution, namely small label noise and approximation error, quantified by the following assumption.
\begin{assumption}[Low Approximation Error and Label Noise]\label{assump:Lipschitz-like}
For some constants, $\epsilon_1, \epsilon_2>0$,
\begin{itemize}[nosep]
    \item \textbf{Approx. Error:} there exists a $2$-layer NN of width $\breve{m}\geq 1$, with unit norm inner-layer weights $\{\breve{\theta}_j\}_{j\in [\breve{m}]}$, and bounded outer-layer weights, $\breve{\va} \in \bR^{\breve{m}},\norm{\breve{\va}}_\infty \leq (2\breve{m})^{-1}\sqrt{\psi^\star(\phi^\star)^{-1}}$, such that,$\forall \vx\in \bR^{d}$,  $\abs{\sigma^\star (\Theta^\star\vx) - \sum_{j\in [\breve{m}]} \breve{a}_j\sigma(\langle\breve{\theta}_j, \vx\rangle)} \leq  \Delta = \cO(m^{-\frac{1}{2}}n^{-\frac{1}{2}-\epsilon_1})$.
\item \textbf{Label Noise:} $\zeta = \cO(m^{-\frac{1}{2}}n^{-\frac{1}{2}-\epsilon_2})$.
\end{itemize}
\end{assumption}
This assumption forces the link function, $\sigma^\star$, to be well-approximated by a $2$-layer network of width $\breve{m}$, activation $\sigma$, and bounded outer-layer weights. A straightforward example for the above assumption is when $\Delta=0$, and the unknown link function is a sum of activations. This is often referred to as the teacher-student setup~\citep{Goldt_2020}. The above assumption can handle cases beyond the teacher-student example, as long as the approximation error, $\Delta$, is sufficiently small. As for the bounds on the outer-layer weights, $\norm{\va^\star}_\infty$, note that these are obtained to ensure that the signal power is at most $\psi^\star$. Further, their scale  $\breve{m}^{-1}$ matches that required by Prop.~\ref{rem:necessary_flatness}. Existing works have shown that approximating well-behaved functions with a $\bar{m}$ width $2$-layer network results in approximation error of the order of $\cO(\mathrm{poly}(\breve{m}^{-1}))$~\citep{SIEGEL20221,Yang2024}. Therefore, our bound on approximation error, $\Delta$, is reasonable as long as $\breve{m}$ is large.

While the label noise requirement might seem restrictive, as it implies the SNR, $\frac{\sqrt{\psi^\star}}{\zeta}$, grows with $n$, we note that similar conditions between SNR and $n$ are often required in prior works for benign overfitting of overparametrized NNs with isotropic $\vx$ and single-index labels, see ~\cite[(A1), (A4)]{pmlr-v178-frei22a}.

We now present our main result answering \textbf{Q2}. This is divided into $2$ results-- proving the existence of an interpolator achieving flatness $\Upsilon^\star$, and bounding the population loss of any interpolator achieving flatness $\Upsilon^\star$.
\begin{theorem}[Flattest Interpolators Generalize]\label{thm:gen_all_good}
Suppose Assump.~\ref{assump:activation},~\ref{assump:link} and \ref{assump:Lipschitz-like} hold and $m\geq 2n$. Then, there exists an interpolator $\vw\in\bR^{m(d+1)}$ with unit-norm inner-layer weights that satisfies Eq.~\eqref{eq:necessary} with high probability. Further, for any interpolator $\vw\in \bR^{m(d+1)}$ with unit-norm inner-layer weights, that satisfies Eq.~\eqref{eq:necessary}, with high probability, $F(\vw) -F^\star\lesssim n^{-\min\bc{\frac{1}{2}, \epsilon_1, \epsilon_2}}$.
\begin{align}\label{eq:necessary}
    \norm{\va}_\infty \leq m^{-1}\br{\sqrt{\psi^\star(\phi^\star)^{-1}} + \gamma n^{-\min\{\epsilon_1,\epsilon_2\}}}, \quad \text{ where } \quad \epsilon_1, \epsilon_2>0, \gamma=\cO(1)
\end{align}
\end{theorem}
Note that both the existence and generalization result correspond to the flattest interpolators $\vw$ with $\Upsilon(\vw) \asymp \Upsilon^\star$ via Prop.~\ref{rem:necessary_flatness}. This Theorem answers \textbf{Q2}, as any interpolator $\vw$ that is asymptotically flattest has a population loss that decreases with $n$. Note that we make no assumptions on the actual correlation between the inner-layer weights $\{\theta_j\}_{j\in [m]}$ and the true direction $\Theta^\star$ for the above Theorem. Existing works on single-index and sum of single-index models~\citep{oko2024neural,damian2023smoothing,pmlr-v178-damian22a} explicitly force this correlation to be large via training algorithms, as it is necessary for a small population loss. In our case, the flattest interpolators, via the balancedness property of Prop.~\ref{rem:necessary_flatness}, implicitly satisfy this condition. In the next section, we provide brief proof sketches for our main results.

\section{Proof Sketches}
\label{sec:proof_sketch}
In this section, we now explain the key theoretical tools used to prove all the main results. The complete proofs are provided in App.~\ref{sec:main_proofs}. At a high level, we need to quantify the relationship between flatness and population loss. For Thms.~\ref{thm:flattest_all} and ~\ref{thm:flattest_good}, this is via lower bounds on $\Upsilon(\vw)$ in terms of the correlation $\rho^\star$ (Def. ~\ref{def:bad_min}) between inner-layer weights and the true direction $\Theta^\star$. For multi-index models, this correlation is a strong indicator of population loss.  For Thm.~\ref{thm:gen_all_good}, this is via upper bounds on population loss based on $\norm{\va}_\infty$, without using any intermediate connections to the correlation. At a high-level, all these results rely on a few fundamental techniques -- to handle terms containing $\va$, we use convex duality ~\cite[Chapter~5]{boyd2004convex} and for all other terms, we appropriately use concentration of measure.

 For all our proofs, the core ingredient is Lemma~\ref{lem:flattest_rescaling}, that encodes the contribution of symmetry in terms of a convex $\ell_p$ norm of $\va$, with $p=\frac{2c}{c+1}$, decided by the symmetry (Def.~\ref{def:homogen}). If $\vy\in \bR^{n}$ is the vector of labels, then, $\vw$ is an interpolator, iff $\vU \va = \vy$. Note that this is a linear constraint in terms of $\va$. Therefore, flatness is a convex function in $\va$ with a linear constraint in $\va$ for interpolation.

\textbf{Proof Sketch of Thms.~\ref{thm:flattest_all} and \ref{thm:flattest_bad}.}
We present this proof sketch in Fig.~\ref{fig:proof_bad}. At the core of this proof, we want to find a lower bound on $\Upsilon(\vw)$ that is independent of $\va$. Duality can exactly compute the expression $\min_{\va}\Upsilon(\vw)$, however, there is no closed-form expression for it, except in the case of quadratic activations ~\citep{ding24flat}. To obtain a closed-form expression for all activations in Assump.~\ref{assump:activation}, we use a tight lower bound by a novel connection to overparametrized linear regression. Consider the interpolation constraint $\vU\va = \vy$. Fixing the inner-layer weights fixes the activation, and for $m\geq n$, this problem resembles overparametrized linear regression in $\va$. From ~\citep{bartlett}, we know the minimum $\ell_2$ norm interpolator for overparametrized linear regression is  $\va_{\min,\ell_2} = \vU^\top(\vU\vU^\top)^{-1}\vy$. As flatness $\Upsilon(\vw)$ is an $\ell_p$ norm, we can obtain a tight lower bound on it via the $\ell_2$ norm of the minimum $\ell_2$ norm interpolator in Lemma~\ref{lem:min_l2_lb}. 

The remainder of the proof utilizes two different bounds on  $\norm{\va_{\min, \ell_2}}_2$ for all interpolators and bad interpolators respectively.  For bad interpolators that are not aligned with the true direction, $\Theta^\star$, the labels $\vy$ are almost independent of the activation matrix $\vU$. Hence, a lower bound on $\norm{\va_{\min,\ell_2}}$ depends on $\trace(\vU\vU^\top)^{-1})$, by a Hanson-Wright bound (Lemma~\ref{lem:hw_sub_weibull}). In contrast, the flattest interpolator is achieved when $\vy$ and activation are completely dependent variables, obtaining a bound of $\sigma_{\max}^{-2}(\vU)$, by Cauchy-Schwarz.

Finally, Lemma~\ref{lem:eigen_U} establishes a difference in the eigenvalues of the activation matrix. From Assump.~\ref{assump:activation}, the activation matrix $\vU$ has non-zero mean. Hence, $\vU\vU^\top$ has a very large maximum eigenvalue, $\sigma_{\max}^2(\vU)\asymp mn$, while the rest of its eigenvalues are much smaller, $\asymp m$. This separation between eigenvalues of $\vU$ allows us to establish a separation between bad interpolators and the flattest interpolator. To establish these bounds on eigenvalues of the activation matrix $\vU$, whose columns might be correlated, we require Assump.~\ref{assump:width}.

{\small
\begin{figure}[t!]
    \centering
\begin{tikzcd}[row sep=small, column sep=small]
    |[draw, rectangle, inner sep=4pt]| {\begin{array}{c} \text{Lemma~\ref{lem:flattest_rescaling}} \\ \Upsilon(\vw) \asymp d^{\frac{c}{c+1}}  \norm{\va}_{\frac{2c}{c+1}}^{\frac{2c}{c+1}} \end{array}} & \\ 
     |[draw, rectangle, inner sep=4pt]| {\begin{array}{c} \text{Lemma ~\ref{lem:min_l2_lb} (Lower bound)}\\ \norm{\va}_{\frac{2c}{c+1}}^{\frac{2c}{c+1}} \gtrsim d^{\frac{c}{c+1}} m^{\frac{1}{c+1}}\norm{\va_{\min, \ell_2}}_2^{\frac{2c}{c+1}}\\
     \va_{\min, \ell_2} := \argmin_{\va \in \bR^{m}, \vU\va = \vy}\norm{\va}_2 \end{array}} 
    & |[draw, rectangle, inner sep=4pt]| {\begin{array}{c} \text{Theorem ~\ref{thm:flattest_bad} (Bad Interpolators)}\\ \norm{\va_{\min, \ell_2}}_2^2 \gtrsim \sum_{i=1}^n\sigma_i^{-2}(\mathbf{U}),\\ \implies \Upsilon(\mathbf{w}) \gtrsim (nd)^{\frac{c}{c+1}}m^{-\frac{c-1}{c+1}}. \end{array}} \\
     |[draw, rectangle, inner sep=4pt]| {\begin{array}{c} \text{Lemma ~\ref{lem:eigen_U} (Eigenvalues)}\\ \sigma_{\max}^{-2}(\mathbf{U})\asymp \frac{1}{mn}, \\\sum_{i=1}^n \sigma_{i}^{-2}(\mathbf{U}) \asymp \frac{n}{m} \end{array}} 
    & |[draw, rectangle, inner sep=4pt]| {\begin{array}{c} \text{Theorem ~\ref{thm:flattest_all} (Flattest Interpolators)}\\ \norm{\va_{\min, \ell_2}}_2^2 \gtrsim n\sigma_{\max}^{-2}(\mathbf{U}),\\  \implies \Upsilon^\star \asymp d^{\frac{c}{c+1}}m^{-\frac{c-1}{c+1}}. \end{array}}
    \arrow[from=1-1, to=2-1]
    \arrow[from=2-1, to=2-2]
    \arrow[from=2-1, to=3-2]
    \arrow[from=3-1, to=2-2]
    \arrow[from=3-1, to=3-2]
 \end{tikzcd}    
 \caption{Proof Sketch of Thms.~\ref{thm:flattest_all} and \ref{thm:flattest_bad}. }
    \label{fig:proof_bad}
\end{figure}
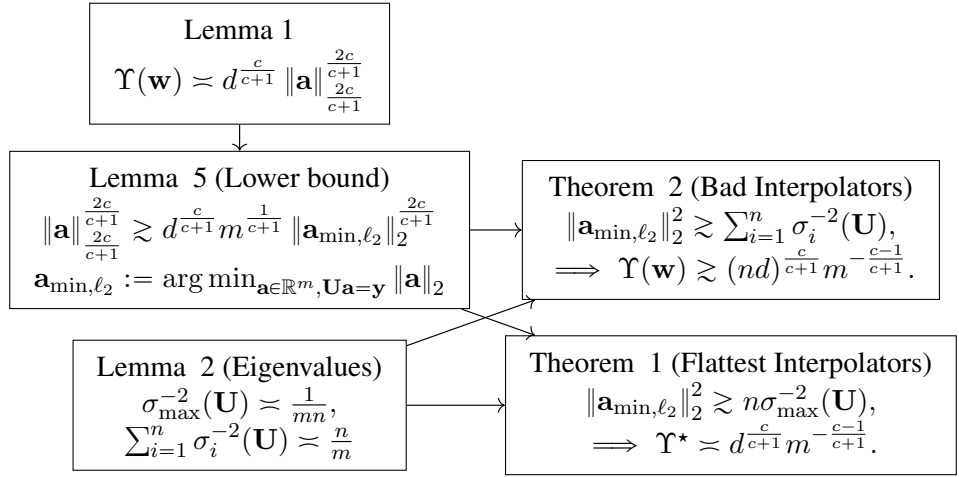}

\textbf{Proof Sketch of Thm.~\ref{thm:gen_all_good}.}
We focus on the generalization of flattest interpolators in Theorem~\ref{thm:gen_all_good}. To show generalization, we need to bound the population loss, $F(\vw)$, for the flattest interpolator. From App~\ref{sec:prelims},
\begin{align}\label{eq:pop_loss}
    2F(\vw) = \zeta^2 + \psi^\star + \sum_{j'=1}^m\sum_{j=1}^m a_j a_{j'}\E{\sigma(\ip{\theta_j}{\vx})\sigma(\ip{\theta_{j'}}{\vx}))} - 2\sum_{j=1}^m a_j \E{\sigma^\star(\Theta^\star \vx)\sigma(\ip{\theta_j}{\vx})}.
\end{align}
Note that the population loss $F(\vw)$ is a quadratic function in $\va$, with its coefficients depending on the inner-layer weights. The flattest interpolator imposes no conditions on inner-layer weights, only a condition on $\norm{\va}_{\infty}$ and interpolation $\vU\va=\vy$. Interestingly, using the worst-case bounds for all terms that depend on the inner-layer weights yields our required bound of $n^{-\min\{\frac{1}{2}, \epsilon_1, \epsilon_2\}}$.

First, consider the quadratic term in Eq.~\eqref{eq:pop_loss}. Note that each $a_j\leq \norm{\va}_\infty$ and the terms of $\sigma$ are upper bounded by $\phi^\star$. This gives an upper bound of $m^2\frac{\psi^\star}{m^2 \phi^\star}\phi^\star = \psi^\star$ on this term. 

To make the population loss small, we want the linear term in $\va$ in Eq.~\eqref{eq:pop_loss} to be large and negative. Instead of using a trivial bound here, we find the maximum value of this linear term,  under the linear interpolation constraint and the convex constraint for minimum flatness $\norm{\va}_\infty = \cO(m^{-1})$. By duality, this bound is,
\begin{align*}
    -\sum_{j=1}^m a_j  \E{\sigma^\star(\Theta^\star \vx)\sigma(\ip{\theta_j}{\vx})} \leq -\frac{1}{n}\norm{\vy}_2^2 + \norm{\va}_{\infty}\sum_{j=1}^m\abs{\frac{1}{n}\sum_{i=1}^n \vU_{i,j}\vy_i - \E{\sigma^\star(\Theta^\star \vx)\sigma(\ip{\theta_j}{\vx})}}.
\end{align*}
Note that $\frac{1}{n}\norm{\vy}_2^2 \approx \psi^\star$, which cancels out the terms of $\psi^\star$, the signal power, in the population loss. The remaining term is the difference between a sample-mean and expectation, as $\E{\vU_{i,j}\vy_i} = \E{\sigma(\ip{\theta_j}{\vx_i})\sigma^\star(\Theta^\star\vx_i)}, \forall i\in [n], j\in [m]$. This is small by concentration, and the term $\norm{\va}_{\infty}$ cancels out the summation over $m$.

\section{Experiments}
\label{sec:experiments}
In this section, we empirically verify our theoretical results for several single-index and multi-index problems. We first describe our experimental setup.

\subsection{Setup}
For all our experiments, we set $d=100, n=500$ and $m=5000$. We use $4$ different activation functions. For $c=1$, we use ReLU and LeakyReLU with coefficient $c'' = -0.1$. For $c=2$, we use quadratic activation. 

For the data distribution, we consider single-index and sum of single-index models as described in Section~\ref{sec:setup}. In each of these distributions, the link functions $\sigma^\star$ and $\tilde{\sigma}^\star$ are obtained by adding a scalar sampled from $\text{Unif}([-\Delta, \Delta])$ to a sum of activations, thus satisfying Assump.~\ref{assump:Lipschitz-like}. For all cases other than ReLU activation, we set $\Delta = m^{-\frac{1}{2}}n^{-\frac{1}{2}-\epsilon_1}$ and $\zeta = m^{-\frac{1}{2}}n^{-\frac{1}{2}-\epsilon_2}$ with $\epsilon_1, \epsilon_2\sim \text{Unif}([0,0.1])$. For ReLU activation, we consider $\zeta=\Delta=0$, so the activation matches the link function. For the sum of single-index models, we select the coefficients $a_j^\star \sim \text{Unif}([-1,1])$ and set $m^\star = 5$.

We use $3$ additional data distributions to cover a varied set of examples. First, we use a single-index model with a linear link function $\sigma^\star(b) = b, \forall b\in \bR$ and a ReLU activation. Note that for this model, $\ReLU(b) - \ReLU(-b) = \sigma^\star(b), \forall b\in \bR$. Therefore, using a ReLU network should be able to learn this data model. Second, we use a single-index model with linear link function and linear activation. Note that linear activation has $c'=1, c''=-1$, so it doesn't satisfy our requirements for Piecewise-Polynomial activations in Assump.~\ref{assump:activation} and our analysis should not cover it. This example corresponds to the case of matrix factorization with $r=1$, matching the example of ~\cite{gatmiry2023what}, however, as $m>n$, RIP property should not hold. Finally, we use an example from ~\cite[Section~4.2]{nichani2023provable}, where $\sigma^\star(\Theta^\star \vx) = \ReLU(\vx^\top (\Theta^\star)^\top \Theta^\star \vx)$ where $m^\star = \frac{d}{2}$. Note that ~\cite{nichani2023provable} show that such problems are harder to learn by $2$-layer NNs than $3$-layer NNs. Further, as $m^\star = \Theta(d)$, this problem is no longer a multi-index problem, so our results should not hold. We use a $\ReLU$ activation for this link. Additionally, while the link function $\sigma^\star$ satisfies Assump.~\ref{assump:link} for $c=2$, learning it with ReLU activation implies a mismatch in homogeneity of the activation and link as $c=1$ for ReLU activation.

For all our data distributions and activations, we first sample a set of random inner-layer weights $\{\theta_j\}_{j\in [m]}$ each of unit norm such that the correlation of each weight to the subspace spanned by the vectors in $\Theta^\star$, is exactly $\rho^\star$. Note that this is similar to our definition of bad minima in Def.~\ref{def:bad_min}. We then find the vector $\va\in \bR^m$ that is the solution of the following optimization problem,
\begin{align}
    \min_{\va\in \bR^{m}} \Upsilon(\vw)\coloneq\sum_{j=1}^m B(\theta_j) \abs{a_j}^{\frac{2c}{c+1}}, \text{ such that } \sum_{j=1}^m a_j\sigma(\ip{\theta_j}{\vx_i}) = y_i, \forall i\in [n]\label{eq:flattest_min} 
\end{align}
The solution to this problem should give us the outer-layer weights that correspond to the flattest interpolator after rescaling according to Lemma~\ref{lem:flattest_rescaling}. For fixed inner-layer weights, the above problem is a convex program with linear constraints. Therefore, we use the convex optimization solver CVXPY~\citep{diamond2016cvxpy,agrawal2018rewriting} with MOSEK~\citep{mosek}, to solve it.

We choose $25$ different logarithmically equally spaced values for $\rho^\star$ in the range $[d^{-2}, \frac{1}{2}]$. For each value of $\rho^\star$, we solve Eq.~\eqref{eq:flattest_min} to compute the outer-layer weights $\va$ and minimum flatness $\Upsilon(\vw)$ corresponding to this $\va$. Then, we compute the population loss for this choice of inner-layer and outer-layer weights for the $2$-layer network by measuring the squared error on a fresh sample of $1500$ datapoints. We provide a scatter plot of the flatness $\Upsilon(\vw)$ and the population loss $F(\vw)$ for each $\rho^\star$, where the color denotes the value of $-\log(\rho^\star)$. The plots are averaged over $5$ random seeds, and the experiments took 5 hours on a single CPU with $15$ cores and $20GB$ RAM.

We plot the results for single-index, sum of single-index, and the $3$ special data distributions in Figures ~\ref{fig:single_index}, \ref{fig:sum_single_index} and ~\ref{fig:special} respectively.

\begin{figure*}[t!]
    \centering
    \begin{subfigure}[t]{0.32\textwidth}
        \centering
        \includegraphics[width=\textwidth]{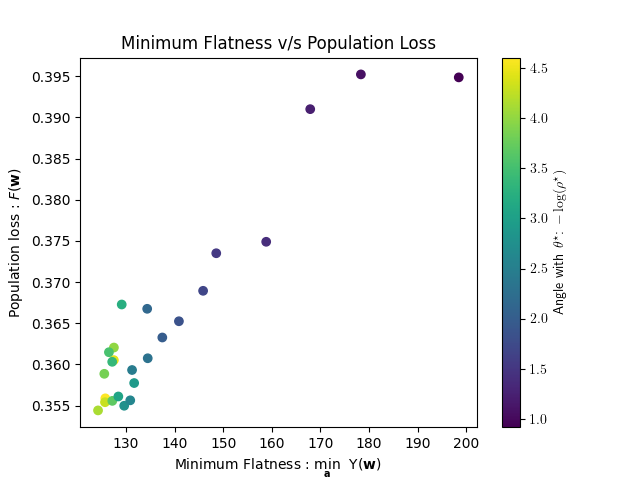}
        \caption{Activation: Leaky ReLU}
    \end{subfigure}%
    \begin{subfigure}[t]{0.32\textwidth}
        \centering
        \includegraphics[width=\textwidth]{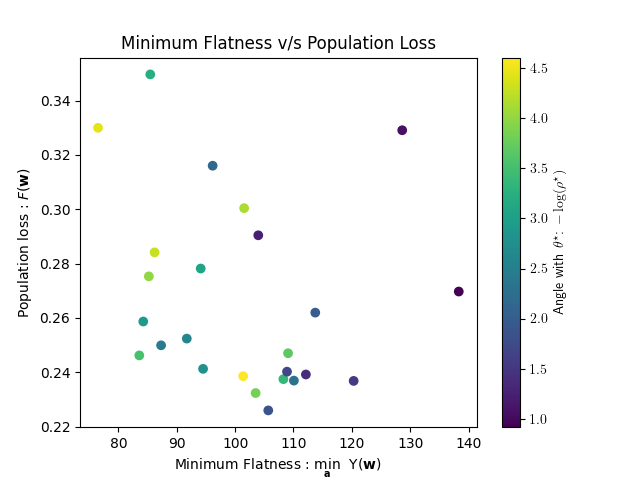}
        \caption{Activation: ReLU}
    \end{subfigure}%
    \begin{subfigure}[t]{0.32\textwidth}
        \centering
        \includegraphics[width=\textwidth]{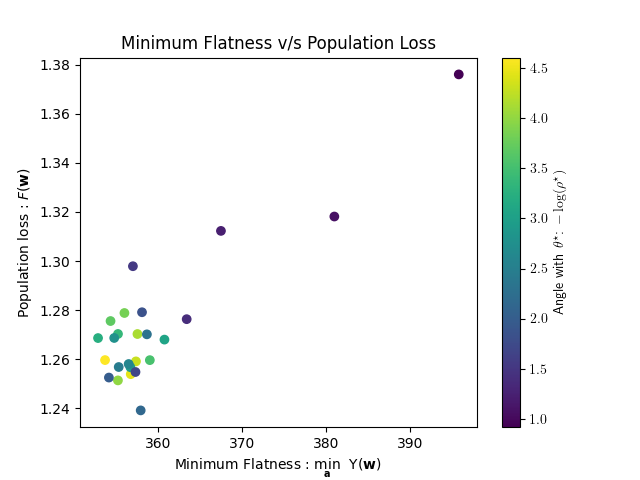}
        \caption{Activation: Quadratic}
    \end{subfigure}%
    \caption{Flatness and population loss for learning single-index link functions close to activations. In all cases except ReLU, the flattest solution generalizes.}
    \label{fig:single_index}
\end{figure*}
\begin{figure*}[t!]
    \centering
    \begin{subfigure}[t]{0.32\textwidth}
        \centering
        \includegraphics[width=\textwidth]{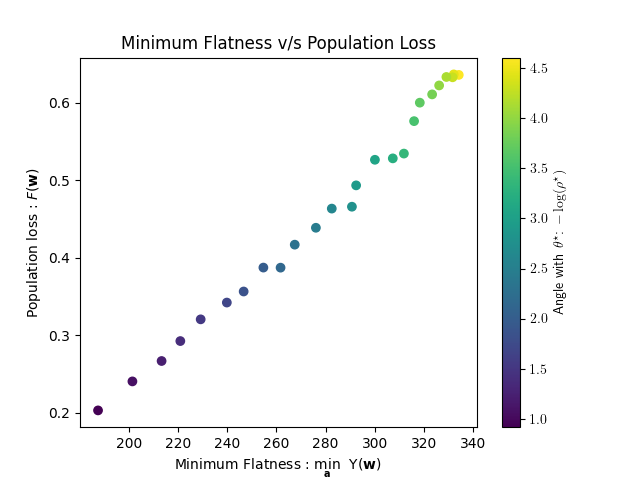}
        \caption{Activation: Leaky ReLU}
    \end{subfigure}%
    \begin{subfigure}[t]{0.32\textwidth}
        \centering
        \includegraphics[width=\textwidth]{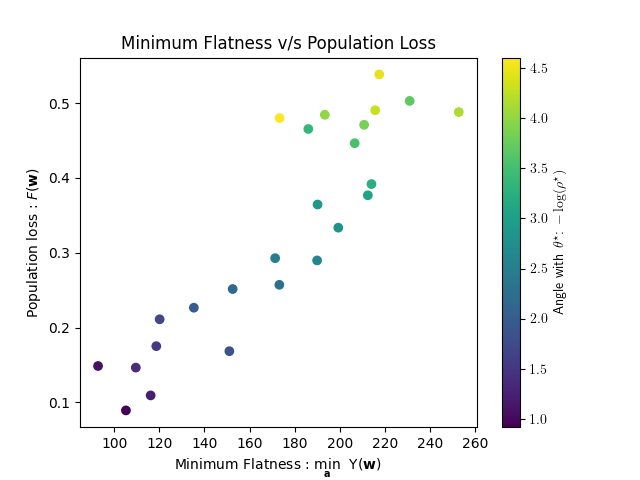}
        \caption{Activation: ReLU}
    \end{subfigure}%
    \begin{subfigure}[t]{0.32\textwidth}
        \centering
        \includegraphics[width=\textwidth]{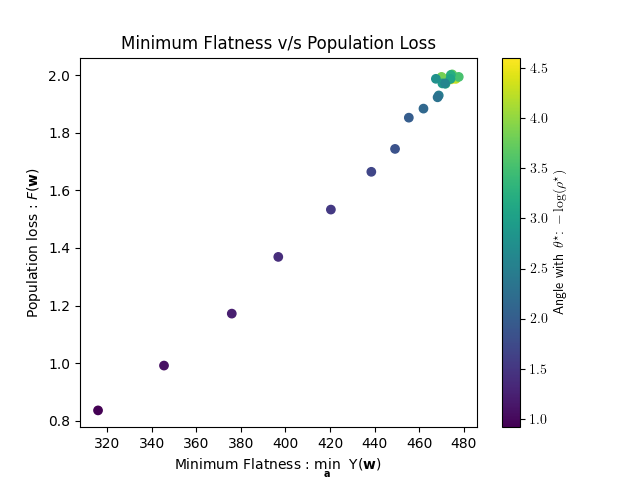}
        \caption{Activation: Quadratic}
    \end{subfigure}%
    \caption{Flatness and population loss for learning sum of single-index link functions, each close to activations.In all cases, the flattest solution generalizes.}
    \label{fig:sum_single_index}
\end{figure*}

\begin{figure*}[t!]
    \centering
    \begin{subfigure}[t]{0.32\textwidth}
        \centering
        \includegraphics[width=\textwidth]{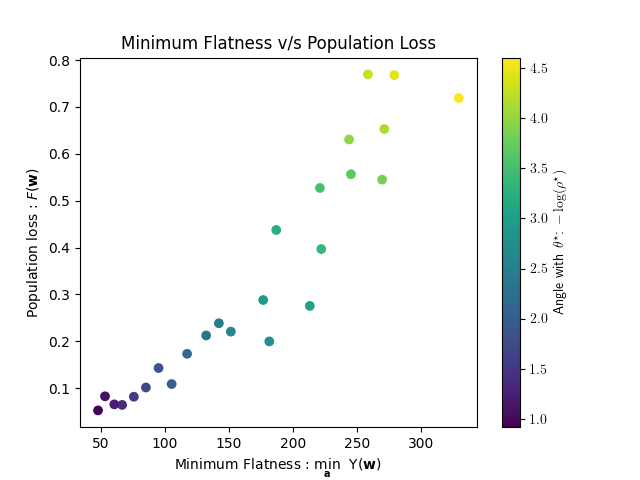}
        \caption{Link: Linear,\\ Activation: ReLU}
        \label{fig:relu_sum_relu_lin}
    \end{subfigure}%
    \begin{subfigure}[t]{0.32\textwidth}
        \centering
        \includegraphics[width=\textwidth]{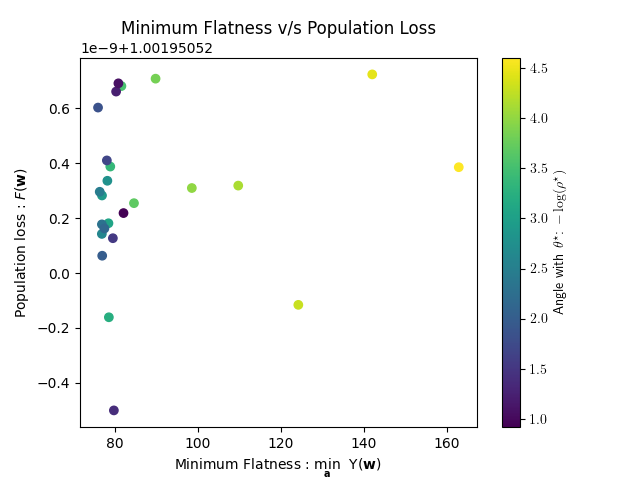}
        \caption{Link: Linear,\\ Activation: Linear}
        \label{fig:linear_single_index}
    \end{subfigure}%
    \begin{subfigure}[t]{0.32\textwidth}
        \centering
        \includegraphics[width=\textwidth]{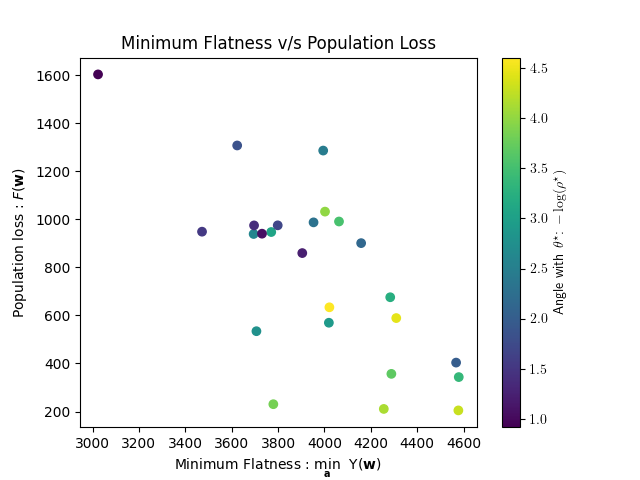}
        \caption{Link: $\ReLU(\vx^\top (\Theta^\star)^\top \Theta^\star \vx)$,\\ Activation: ReLU}
        \label{fig:relu_quadratic_type}
    \end{subfigure}%
    \caption{Flatness and population loss for special data distributions. For learning linear link with ReLU activation, flattest interpolator generalizes. For linear activation with linear link all presented interpolators, achieve small flatness and generalize, however for the last case, flatttest interpolator doesn't generalize. }
    \label{fig:special}
\end{figure*}

\subsection{Results}

\paragraph{Single-Index Models.} For single-index models, there is a link between flattest interpolators and their population loss. From Fig~\ref{fig:single_index}, we can see that for LeakyReLU and Quadratic activation, if the minimum flatness after rescaling, $\min_{\va}\Upsilon(\vw)$, is low, then population loss is also low. However, extremely high correlation with the true direction, i.e., large $\rho^\star$ forces large $\norm{\va}_\infty$ to ensure interpolation. Thus,  large $\rho^\star$ does not achieve the minimum flatness and, in turn, has a high population loss. This phenomenon results in a poor connection between flatness and generalization for ReLU activation, where very small $\rho^\star$ can easily interpolate with small $\norm{\va}_\infty$, but as $\rho^\star$ is small, its population loss is large. For single-index models, this is a result of poor conditioning of activation, thus violating Assump.~\ref{assump:width}.

\paragraph{Sum of Single-Index Models.} For the sum of single-index models, the connection between flatness and generalization is much stronger. Further, there is also a connection between the correlation $\rho^\star$ and flatness. From Fig~\ref{fig:sum_single_index}, this connection holds for all activations, even ReLU, where it seemed to be absent for single-index models. If we increase $\rho^\star$, the correlation to the true direction, the population loss decreases, and the minimum flatness $\min_{\va}\Upsilon(\vw)$ also decreases. This stronger connection is precisely because multi-index models with diverse features can interpolate much better for a large $\rho^\star$ than single-index models, as they can still satisfy Assump.~\ref{assump:width}.

\paragraph{Special Cases.} For the special case of learning a linear function using ReLU in Fig~\ref{fig:relu_sum_relu_lin}, although this is a single-index model, it behaves similarly to our results for multi-index models, with a strong connection between the correlation, population loss and flatness. For a linear link function with linear activation in Fig~\ref{fig:linear_single_index}, it appears that most interpolators obtain the same flatness. While their population losses can be different, all these population losses are very small, even for different $\rho^\star$. This is because for linear activation, the outer-layer $\va$ weights can be selected so that $\sum_{j=1}^m a_j \theta_j$ is parallel to $\theta^\star$. Therefore, for linear link function and linear activation, empirically flattest interpolators seem to generalize, however the theoretical justification in ~\citep{gatmiry2023what} is still loose. For multi-index model in Fig~\ref{fig:relu_quadratic_type}, note that a $2$-layer network with bounded outer-layer weights cannot approximate it~\citep{nichani2023provable}. Therefore, the approximation error $\Delta$ is always large and thus, there is no connection between flatness and generalization, as the flattest solution seems to have a high population loss.

We conclude our paper with a summary of our main contributions, promising extensions and limitations in the next Section. 

\vspace{-2mm}
\section{Conclusion}
\label{sec:discussion}
\vspace{-2mm}

In this paper,  we have shown that despite symmetry that can change the flatness of an interpolator, for multi-index data, some bad interpolators are not the flattest interpolators. Further, if labels are generated as a sum of single-index models with low approximation error and low label noise, flattest interpolators generalize. Therefore, the original claim of ``flat interpolators generalize" might be false with symmetry, but ``flattest interpolators still generalize". This provides a connection between flatness and generalization which was unknown for general noisy data distributions in the presence of symmetry with Homogeneous NNs. 

We conclude our paper with a discussion of its limitations that form promising directions for future work: 

\textbf{Higher label noise:} A large label noise makes the flatness of every interpolator the same as that of a bad interpolator in Thm.~\ref{thm:flattest_bad}, as the label $\vy$ effectively has a large component independent of activation. A more precise analysis of $\Upsilon(\vw)$ in terms of the signal and noise terms ($\psi^\star$ and $\zeta^2$) should ideally permit a larger label noise as long as SNR also grows with $n$. To extend our results to both constant $\zeta^2$ and constant SNR, which is the case for gradient-based methods on single-index and sum of single-index data~\citep{oko2024neural,pmlr-v247-oko24a}, we need to consider weights beyond interpolators, $F_S(\vw)=0$, to approximate interpolators, $F_S(\vw) \lesssim \zeta^2$, as none of the gradient methods converge to interpolators in finite steps.  

\textbf{Size of the set of flattest interpolators:} All our results (Thms.~\ref{thm:flattest_all}, \ref{thm:flattest_bad} and \ref{thm:gen_all_good}) hold for a single model with high probability. To extend these results to all models in a given set, for instance that of flattest interpolators, we need to quantify their sizes for a high probability union bound. Unfortunately, there are no known tight estimates for these sizes, and the only loose estimate that we have is the size of set of all model weights.

\textbf{Broader necessary and sufficient conditions:} While our setting is more general than that of existing works~\citep{ding24flat}, it is still focused on a particular form of symmetry, a particular class of multi-index data, and a particular NN architecture. In light of negative results from ~\cite{schliserman2025flatminimageneralizationinsights,wen2023how}, there is still hope for discovering a more general set of necessary and sufficient conditions for flattest interpolators to generalize, of which our problem setup is a special instance.   

\textbf{Algorithms that converge to flattest interpolators:} Our work does not provide any algorithms that can converge to these flattest interpolators. Existing optimizers like SAM~\citep{DBLP:conf/iclr/ForetKMN21, pmlr-v139-kwon21b} and its variants, or (S)GD with large step-sizes~\citep{cohen2021gradient}, attempt to reduce flatness; however, these are not theoretically guaranteed to reach the flattest interpolators~\citep{schliserman2025flatminimageneralizationinsights,wen2023how} except in specific models~\citep{wu2023implicit}.

\bibliographystyle{apalike}
\bibliography{references}

\appendix
\crefalias{section}{appendix} %
\newpage

\section{Related Works}
\label{sec:related_works}

Existing works focus on flatness of even local minima of empirical loss, not just interpolators. Throughout this section, we will discuss flatness of an arbitrary minima of the empirical loss. Interpolators are global minima of the empirical loss, with $0$ empirical loss.

\paragraph{Experimental Evidence Linking Flatness and Generalization}
There have been several attempts to create algorithms that seek flatter minima in non-convex neural networks. ~\cite{hochreiter} were the first to use variants of gradient descent to explicitly search for flat minima. For a large class of algorithms to find flat minima, some form of noise is added to the updates of gradient descent, either independent noise like SGLD~\citep{Chaudhari_2019} or dependent noise from stochastic gradients~\citep{keskar2017on}. In practice, these algorithms also converge to good minima. ~\cite{DBLP:conf/iclr/ForetKMN21} propose the Sharpness-Aware Minimization (SAM) as an alternative to SGD. This algorithm intuitively minimizes sharpness and in practice converges to good minima. ~\cite{chiang2023loss} implement non-gradient based  algorithms that only depend on flatness of the empirical landscape also converge to good minima. ~\cite{Jiang*2020Fantastic} and ~\cite{pmlr-v202-andriushchenko23a} benchmark the correlation between flatness and generalization for a large class of real-world NN architectures and datasets. ~\cite{Jiang*2020Fantastic} compute flatness by measuring the sensitivity of the loss to perturbations , while ~\cite{pmlr-v202-andriushchenko23a} use a version of this flatness metric that is invariant under rescaling symmetry. ~\cite{Jiang*2020Fantastic} show strong correlation between flatness and generalization for most cases, while ~\cite{pmlr-v202-andriushchenko23a} show that this correlation is high only for specific architecture and dataset combinations.

\paragraph{Theoretical Flatness-based Generalization bounds}    
Most flatness-based generalization bounds can be classified into $2$ categories -- minima stability bounds and PAC-Bayesian bounds. Minima stability~\citep{wu_how_sgd_18, wu_alignment_2022,mulayoff_exact_2024, NacsonMOMS23,chemnitz_characterizing_2024} shows that (S)GD with a large step size can stably converge to a flat minima whose flatness is proportional to the step size. The goal of minima stability-based generalization is to bound the generalization of the set of flat minima. For specific examples, the set of flat minima has weights with small norm. Then, they use generalization bounds for models with small weight norm~\citep{neyshabur}. These bounds are available for a few specific settings:classification and regression on $2$-layer ReLU networks~\citep{mulayoff_implicit_2021,qiao2024stable,liang2025stable,qiao2025doesflatnessimplygeneralization, liang2025generalizationedgestabilityrole} and diagonal neural networks~\citep{pmlr-v202-wu23r}. Some of these bounds become vacuous under rescaling symmetry~\citep{pmlr-v202-wu23r}, while others do not apply to our setting as they need  $d=1$~\citep{mulayoff_implicit_2021,qiao2024stable,liang2025stable,qiao2025doesflatnessimplygeneralization} or $d\geq 1$, but $\vx$ is a finite mixture of low-dimensional subspaces~\citep{liang2025generalizationedgestabilityrole}. The only bounds that are applicable for $d\geq 1$ are from ~\cite{liang2025stable}, however these have excess population risk $\cO(d^{-\frac{1}{d}}) = \cO(1)$ for our setting with optimal sample complexity $n=\Tilde{\Theta}(d)$. The core idea of PAC-Bayesian approaches is to separate a term in the upper bound on excess population loss that is related to the flatness, $\nabla^2 F_S(\vw)$. ~\cite{pmlr-v80-dziugaite18a} show that this happens for Entropy-SGD algorithm in ~\citep{Chaudhari_2019}, while ~\cite{pmlr-v272-haddouche25a} show this more generally. However, these bounds are still sensitive to rescaling symmetry, and are not tight for arbitrary neural networks with our data distribution. In fact, note that any flatness-based generalization bound that measures the population loss of all empirical minima that are $\Upsilon$-flat for some $\Upsilon>0$, i.e., $\trace(\nabla^2 F_S(\vw)) \leq \Upsilon$, is vacuous for $\Upsilon \gtrsim (dn)^{\frac{c}{c+1}}m^{-\frac{c-1}{c+1}}$. This is due to the existence of bad interpolators in Definition~\ref{def:bad_min}, which can achieve this flatness using rescaling symmetry (Theorem~\ref{thm:flattest_bad}).

\paragraph{Flattest Minima under Rescaling}
~\cite{pmlr-v70-dinh17b} show that good minima can be made sharp for most reasonable definitions of sharpness using this rescaling symmetry. Several attempts have been made to handle this rescaling symmetry. For simpler models like matrix factorization, $2$-layer NN with quadratic activation~\citep{ding24flat}, and deep linear NN~\citep{pmlr-v119-mulayoff20a}, the flattest minima under rescaling can be computed. Our settings cover most of their results, apart from the case of deep linear networks, where they establish a connection between depth and flatness. ~\cite{pmlr-v119-tsuzuku20a} and ~\cite{adilova2023famrelativeflatnessaware} use generalization bounds to obtain algorithms that find flat minima invariant to rescaling symmetry. To the best of our knowledge, only ~\cite{petzka} provide a generalization bound invariant under scaling for classification with a general loss function, model architecture and data distribution. Their generalization bound encodes flatness using a form of robustness to features. However, their assumptions are very restrictive and not applicable to our setting. In particular, they cannot be extended to regression, or non-zero label noise, or if the labels are not locally constant.  Additionally, their generalization bound for $n\asymp d$, which is the optimal sample complexity in our setting, is $\cO(d^{-1/d}) = \cO(1)$, making their bounds vacuous. As for deep linear networks ~\citep{gatmiry2023what}, under regularity conditions like the restricted isometry property (RIP)~\cite[Chapter~7]{Wainwright_2019}, the flattest interpolator of deep linear networks for low-rank matrix factorization recovers the ground truth and thus perfectly generalizes. However, the regularity conditions, like RIP, are not satisfied in our settings, as we have $m \geq n$ ~\cite[Corollary~1]{gatmiry2023what}. This result crucially uses uniform convergence bounds ~\cite[Theorem~6]{gatmiry2023what} for generalization, which are vacuous in our overparametrized settings.

\paragraph{Multi-Index Data Distributions.}
Multi-index data distributions have been extensively studied in traditional statistics~\citep{Li91}; however, more recently, these have become an appropriate test-bed to study learning in $2$-layer neural networks. For $2$-layer neural networks, learning guarantees for gradient-based algorithms, for instance, Online-SGD or SGD with batch-reuse, exist for single-index~\citep{ben_arous,oko2024neural} and sum of single-index models~\citep{pmlr-v247-oko24a}. For deeper networks, there has been significant work on multi-index models with a hierarchical structure~\citep{damian2026the, pmlr-v195-abbe23a} or $3$-layer networks~\citep{nichani2023provable}. All these works require much smaller sample complexity than the NTK regime~\citep{jacot_ntk}, often dependent on Hermite coefficients of the link function $\sigma^\star$, for instance, the information~\citep{ben_arous}, generative~\citep{damian2023smoothing} and leap exponents~\citep{pmlr-v195-abbe23a, damian2026the}. For more general multi-index models, existing works can guarantee that gradient-based algorithms can recover the true direction $\Theta^\star$; however, they cannot show generalization without additional assumptions~\cite{bietti23,pmlr-v258-simsek25a,defilippis2026optimal}. Note that all these works characterize the population loss of the output of gradient-based methods; however, they don't consider the flatness of the loss landscape. A detailed survey of existing works in multi-index models is provided in ~\citep{hsu_multi_index}.

\section{Preliminaries}
\label{sec:prelims}

In this section, we describe core technical results required for our results. Most of their proofs are deferred to App.~\ref{sec:misc}.

\paragraph{Additional Notation}
Recall that $\vw\in \bR^{m(d+1)}$ is used to denote the weights of an interpolator with unit-norm inner-layer weights $\{\theta_j\}_{j\in [m]}$, succintly represented by the matrix $\Theta\in \bR^{m\times d}$, and outer layer weights $\va\in \bR^{m}$. Further, $\vy\in \bR^{n}$ denotes the vector of labels, with $\vy = \vy^\star + \vn$, where $\vn$ is the vector of label noise $\{\xi_{i}\}_{i\in [n]}$ and $\vy^\star$ is the vector of true responses, $\{\sigma^\star(\Theta^\star\vx_i)\}_{i\in [n]}$.  We use $\widetilde{\vU} \coloneq (\vU\vU^\top)^{-1}$. We use $\vone_{k}, \vzero_k\in \bR^{k}$ to represent the vectors with all its coordinates $1$ or $0$.

\subsection{Population Loss}
We first provide the expression of population in terms of easier to handle quantities.
We define the functions $\widetilde{\psi}:[-1,1]^{m^\star}\to \bR$ and $\phi : [-1,1] \to \bR$ as the following,
\begin{align*}
   \forall \Xi \in [-1,1]^{m^\star}, \quad \widetilde{\psi}(\Xi) &= \E{\sigma^\star(\vq_2)\sigma(q_1)} , \text{ where } \begin{bmatrix}
        q_1\\
        \vq_2\\
    \end{bmatrix}\sim \cN\br{\vzero, \begin{bmatrix}
        1 & \Xi^\top \\
        \Xi & \bI_{m^\star}
    \end{bmatrix}}\\
   \forall \Xi \in [-1,1], \quad \phi(\Xi) &= \E{\sigma(q_2)\sigma(q_1)} , \text{ where } \begin{bmatrix}
        q_1\\
        q_2\\
    \end{bmatrix}\sim \cN\br{\vzero, \begin{bmatrix}
        1 & \Xi \\
        \Xi & 1
    \end{bmatrix}}.    
\end{align*}
Note that $\phi(1) = \phi^\star = \Ee{q\sim \cN(0,1)}{\sigma^2(q)}$. Using these functions, we can compute the population loss of any weight $\vw\in \bR^{d}$ with inner layer weights $\{\theta_j\}_{j\in [m]}$, $\theta_j\in \bS^{d-1}, \forall j\in [m]$, and outer-layer weights $\va\in \bR^m$.
\begin{align*}
    2F(\vw) &= \E{(y - \sum_{j=1}^{m}a_j\sigma(\ip{\theta_j}{\vx}))^2} = \Ee{\vx, \xi}{(\sigma^\star(\Theta^\star\vx) + \xi - \sum_{j=1}^{m}a_j\sigma(\ip{\theta_j}{\vx}))^2}\\
    &= \Ee{\vx}{(\sigma^\star(\Theta^\star\vx) - \sum_{j=1}^{m}a_j\sigma(\ip{\theta_j}{\vx}))^2}  + \Ee{\xi}{\xi^2} + 2\Ee{\xi}{\xi}\Ee{\vx}{(\sigma^\star(\Theta^\star\vx) - \sum_{j=1}^{m}a_j\sigma(\ip{\theta_j}{\vx}))}\\
\end{align*}
We first expand the terms corresponding to the noise $\xi$ and plug in its mean and variance. 

Now, we expand the terms of activation.
\begin{align*}
    2F(\vw)    &= \Ee{\vx}{(\sigma^\star(\Theta^\star\vx) - \sum_{j=1}^{m}a_j\sigma(\ip{\theta_j}{\vx}))^2}  + \zeta^2\\    
    &= \Ee{\vx}{(\sigma^\star(\Theta^\star\vx))^2} + \Ee{\vx}{\br{\sum_{j=1}^{m}a_j\sigma(\ip{\theta_j}{\vx})}^2} - 2\sum_{j\in [m]} a_j\Ee{\vx}{\sigma^\star(\Theta^\star\vx)\sigma(\ip{\theta_j}{\vx})}  + \zeta^2\\    
    &= \psi^\star +\zeta^2 + \underset{j,j'\in [m]}{\sum\sum}a_j a_{j'}\Ee{\vx}{\sigma(\ip{\theta_j}{\vx}))\sigma(\ip{\theta_{j'}}{\vx})} - 2\sum_{j\in [m]} a_j\widetilde{\psi}(\Theta^\star\theta_j)  \\    
    &= \psi^\star +\zeta^2 + \underset{j,j'\in [m]}{\sum\sum}a_j a_{j'}\phi(\ip{\theta_j}{\theta_{j'}}) - 2\ip{\va}{\widetilde{\Psi}}  \\    
    2 F(\vw) &= \psi^\star +\zeta^2 + \va^\top\Phi\va - 2\ip{\va}{\widetilde{\Psi}}
\end{align*}
Here, $\widetilde{\Psi}\in \bR^m$ is a vector with its $j^{th}$ coordinate being $\widetilde{\psi}(\Theta^\star\theta_j), \forall j\in [m]$ and $\Phi\in \bR^{m\times m}$ is a matrix with its $j,j^{th}$ entry being $\Phi_{j,j'} = \phi(\ip{\theta_{j}}{\theta_{j'}}, \forall j,j'\in [m]$. Next, we find the optimal outer-layer weights that minimize the population loss for single-index data distributions.

\subsection{Flatness and Rescaling Symmetry}
In this section, we discuss the impact of symmetry on flatness. First, we state the full expression of flatness for any $\vw\in \bR^{m(d+1)}$.
\begin{align*}
    \nabla F_S(\vw) &= -\frac{1}{n}\sum_{i=1}^n (y_i - h(\vw, \vx_i))\nabla_{\vw} h(\vw, \vx_i)\\
    \nabla^2 F_S(\vw) &= \frac{1}{n}\sum_{i=1}^n \br{\nabla_{\vw}h(\vw, \vx_i)(\nabla_{\vw}h(\vw, \vx_i))^\top  - (y_i - h(\vw, \vx_i))\nabla^2 h(\vw, \vx_i)}.
\end{align*}
If $\vw$ is an interpolator, $F_S(\vw)=0$, so $h(\vw, \vx_i) = y_i, \forall i\in [n]$. This removes the second term in $\nabla^2 F_S(\vw)$. Therefore, for any interpolator, flatness is,
\begin{align*}
    \trace(\nabla^2 F_S(\vw)) = \frac{1}{n}\sum_{i=1}^n \trace\br{\nabla_{\vw} h(\vw, \vx_i)(\nabla_{\vw} h(\vw, \vx_i))^\top} = \frac{1}{n}\sum_{i=1}^n \norm{\nabla_{\vw} h(\vw, \vx_i)}_2^2.
\end{align*}
Now, we compute the value of $\nabla_{\vw} h(\vw, \vx)$.
Note that $h(\vw, \vx) = \sum_{j=1}^m a_j \sigma(\ip{\theta_j}{\vx})$. Therefore,
\begin{align*}
    \nabla_{a_j} h(\vw, \vx) = \sigma(\ip{\theta_j}{\vx}), \quad \nabla_{\theta_j} h(\vw, \vx) = a_j \sigma'(\ip{\theta_j}{\vx})\vx .
\end{align*}
Plugging these values into $\trace(\nabla^2 F_S(\vw))$, we obtain the flatness of any interpolator as follows,
\begin{align*}
    \trace(\nabla^2 F_S(\vw)) = \frac{1}{n}\sum_{i=1}^n\sum_{j=1}^m \br{\sigma^2(\ip{\theta_j}{\vx_i}) + a_j^2 (\sigma')^2(\ip{\theta_j}{\vx_i})\norm{\vx_i}_2^2}.
\end{align*}

Now, applying rescaling symmetry from Def.~\ref{def:homogen}, for rescaling coefficients $\{\alpha_{j}\}_{j\in [m]}$, the flatness of the new interpolator is,

\begin{align*}
    \trace(\nabla^2 F_S(\widetilde{\vw})) = \frac{1}{n}\underset{i\in [n],j\in [m]}{\sum\sum} (\alpha_j^{2c}\sigma^2(\ip{\theta_j}{\vx_i}) + \alpha_j^{-2} a_j^2 (\sigma')^2(\ip{\theta_j}{\vx_i})\norm{\vx_i}_2^2) 
\end{align*}
We use the fact that $\sigma$ and $\sigma'$ are $c$ and $(c-1)$ homogeneous functions from Assump.~\ref{assump:activation}. Here, $\widetilde{\vw}$ is the interpolator obtained after applying the rescaling coefficients. Therefore,
$\widetilde{\vw} = \begin{bmatrix}
    \alpha_1^{-c}a_1, \ldots, \alpha_m^{-c}a_m, \alpha_1\theta_1^\top, \ldots, \alpha_m \theta_m^\top.
\end{bmatrix}$

To show that $\tilde{\vw}$ is still an interpolator, note that,
\begin{align*}
    h(\widetilde{\vw}, \vx) = \sum_{j=1}^m \alpha_j^{-c}a_j \sigma(\ip{\alpha_j \theta_j}{\vx}) = \sum_{j=1}^m \alpha_j^{-c}a_j \alpha_{j}^c \sigma(\ip{\theta_j}{\vx}) = \sum_{j=1}^m a_j \sigma(\ip{\theta_j}{\vx}) = h(\vw, \vx), \quad \forall \vx\in \bR^{d}.
\end{align*}
Therefore, $\widetilde{\vw}$ is also an interpolator and has the same population loss as $\vw$.

Note that changing the values of $\{\alpha_j\}_{j\in [m]}$ changes the value of $\trace(\nabla^2 F_S(\widetilde{\vw})$.  The only case when this remains constant for any $\alpha_j$ is if $\sigma(\ip{\theta_j}{\vx_i}) = 0, a_j,\forall j \in [m], i\in [n]$. Since $\vw \in \cW_S^\star$, this implies, $h(\vw, \vx_i) = y_i = 0, \forall y_i \in [n]$. If $\zeta = 0$, $y_i = \sigma(\ip{\theta^\star}{\vx_i})$. For piecewise polynomial activations, $\sigma(b) = 0$ for either $b\geq0$ or $b\leq 0$. Since $\ip{\theta^\star}{\vx_i} \sim \cN(0,1)$, this event can happen with probability atmost $\frac{1}{2}$ for each $i\in [n]$. Since we would need $y_i = 0, \forall i\in [n]$, the probability of such an event is bounded by $2^{-n}$ by the independence of $\vx_i$. If $\zeta \neq 0$, then $y_i = 0, \forall i\in [n]$ if $\xi_i = - \sigma(\ip{\theta^\star}{\vx_i}), \forall i\in [n]$. Since $\xi_i$ is a gaussian independent of $\vx_i$, this event happens with probability $0$ for each $i\in [n]$.

Note that by setting any $\alpha_j \to \infty$ or $\alpha_j \to 0$, we can increase $\trace(\nabla^2 F_S(\widetilde{\vw})\to \infty$. Therefore, we can arbitrarily make any interpolator sharper. However, it is a convex function in each $\alpha_j$, hence, we cannot make any interpolator flatter.

\subsection{Properties of Piece-wise Polynomial Activations (Assump.~\ref{assump:activation})}
\begin{proposition}[Activations and Sum of Activations are Valid Link Functions]\label{prop:activation_valid}
Any Piece-wise Polynomial activation function $\sigma$ (Assump.~\ref{assump:activation}) satisfies Assump~\ref{assump:link} with $m^\star = 1$. Further, $\hat{\sigma}:\bR^{m^\star}\to \bR$ defined as the following also satisfies Assump.~\ref{assump:link}. 
\begin{align*}
    \hat{\sigma}(\vb) = \sum_{j=1}^{m^\star} \hat{a}_j \sigma(b_j),\quad \forall \vb \in \bR^{m^\star} \text{ with } \abs{\hat{a}_j} = \cO(1), \forall j\in [m^\star]
\end{align*}
\end{proposition}
The proofs of this Prop. are provided in App.~\ref{sec:link_func}.
This proposition ensures that the sum of activations and activations themselves are valid link functions, and thus have Sub-Weibull tails.

This provides us with the following Lem. on the behaviour of their eigenvalues, which we crucially use to establish a gap between bad and good interpolators.
\begin{lemma}[Eigenvalues of Activation matrix]\label{lem:eigen_U}
If Assump~\ref{assump:width} holds, with probability $1-\delta$,
\begin{align*}
\sigma_{\max}(\vU\vU^\top) \asymp mn,\quad\text{ and  } \forall i \in \{2,3,\ldots, n\}, \;\sigma_{i}(\vU\vU^\top) \asymp m    
\end{align*}
where $\sigma_{\max}$ is the largest singular value and $\sigma_i$ is the $i^{th}$ largest singular value of a matrix.    
\end{lemma}
The proof of this Lem. is provided in App.~\ref{sec:eigen_U_proof}.

\section{Main Proofs}
\label{sec:main_proofs}
In this section, we describe our main proofs for Informal Thms.~\ref{inf_thm:bad_flat} and \ref{inf_thm:good_flat}. Based on the proof sketch from Section~\ref{sec:proof_sketch}, we break down our core proof into $4$ important parts.
\begin{enumerate}
    \item  Derivation of $\Upsilon(\vw)$ (App.~\ref{sec:lem_flattest_rescaling_proof}).
    \item  Lower Bounds on $\Upsilon(\vw)$ for All Interpolators (App.~\ref{sec:upsilon_lb_proof}).
    \item  Lower Bounds on $\Upsilon(\vw)$ for Bad Interpolators (App.~\ref{sec:proof_flattest_bad}).
    \item  Connection of $\Upsilon^\star$ to population loss (App.~\ref{sec:proof_flattest_pop}).
\end{enumerate}
We will prove Lem.~\ref{lem:flattest_rescaling} in the first part to derive $\Upsilon(\vw)$. Then, we obtain a lower bound on $\Upsilon(\vw)$ in terms of only the inner-layer weights, $\{\theta_j\}_{j\in [m]}$, in  Lem.~\ref{lem:min_l2_lb}. This allows us to compute $\Upsilon^\star$, proving Thm.~\ref{thm:flattest_all}. Then, we use Lem.~\ref{lem:min_l2_lb} along with the definition of bad interpolators from Def.~\ref{def:bad_min} to obtain the lower bound on flatness for bad interpolators, proving Thm.~\ref{thm:flattest_bad}. Finally, we utilize the value of $\Upsilon^\star$ to characterize the population loss of the set of flattest interpolators, thus proving Thm.~\ref{thm:gen_all_good}. Proofs of any intermediate Lems. that are not provided in Appendices ~\ref{sec:lem_flattest_rescaling_proof} - \ref{sec:proof_flattest_pop} are provided in App.~\ref{sec:misc}.

\subsection{Derivation of $\Upsilon(\vw)$ : Proof of Lem.~\ref{lem:flattest_rescaling}}
\label{sec:lem_flattest_rescaling_proof}

Our proof heavily relies on the expression of $\Upsilon(\vw)$ from Lem.~\ref{lem:flattest_rescaling}. We first provide a proof for this Lem. by using the following two Lems.
\begin{lemma}\label{lem:flattest_rescaling_full}
    For any interpolator $\vw\in \bR^{m(d+1)}$ with unit norm inner-layer weights,
    \begin{align*}
        &\Upsilon(\vw) = \min_{\hat{\vw}\in \cW_{\mathrm{rescale}(\vw)}}\trace(\nabla^2 F_{S}(\hat{\vw})) = \sum_{j=1}^m \abs{a_j}^{\frac{2c}{c+1}} B(\theta_j),\\
        \text{where } &B(\theta_j) = (c^{-\frac{c}{c+1}} + c^{\frac{1}{c+1}})B_1^{\frac{1}{c+1}}(\theta_j)B_2^{\frac{c}{c+1}}(\theta_j).
    \end{align*}
    Here, $B_1, B_2:\bS^{d-1}\to \bR_{+}$ are functions given by $B_1(\theta) \coloneq \frac{1}{n}\sum_{i=1}^n \sigma^2(\ip{\theta}{\vx_i})$ and $B_2(\theta) \coloneq \frac{1}{n}\sum_{i=1}^n (\sigma')^2(\ip{\theta}{\vx_i})\norm{\vx_i}_2^2$.
\end{lemma}

\begin{lemma}[Bounds on $B(\theta)$]\label{lem:bounds_b_theta}
    For Piece-wise Polynomial Activation $\sigma$ (Assump.~\ref{assump:activation}), from Prop.~\ref{prop:activation_valid}, with probability $1-6\delta$, for a set $\{\theta_j\}_{j\in [m]}$, with $\theta_j\in \bS^{d-1}, \forall j\in [m]$,
    \begin{align*}
        \max_{j\in [m]}B(\theta_j) \lesssim d^{\frac{c}{c+1}}, \quad \min_{j\in [m]}B(\theta_j) \gtrsim d^{\frac{c}{c+1}}, \implies B(\theta_j) \asymp d^{\frac{c}{c+1}}, \forall j\in [m]
    \end{align*}
\end{lemma}

Using Lem.~\ref{lem:flattest_rescaling_full} and ~\ref{lem:bounds_b_theta}, we can show that with probability $1-6\delta$, each $B(\theta_j) \asymp d^{\frac{c}{c+1}}$, therefore, 
\begin{align*}
    \Upsilon(\vw) \asymp d^{\frac{c}{c+1}}\sum_{j=1}^m \abs{a_j}^{\frac{2c}{c+1}} = d^{\frac{c}{c+1}} \norm{\va}_{\frac{2c}{c+1}}^{\frac{2c}{c+1}}.
\end{align*}
This completes the proof of Lem.~\ref{lem:flattest_rescaling}. The proof of the intermediate Lem.~\ref{lem:bounds_b_theta}  is deferred to App.~\ref{sec:b_theta_bounds}. We provide a proof for Lem.~\ref{lem:flattest_rescaling_full} as it extends the flatness beyond ReLU ($c=1$ in ~\cite{wen2023how}) and quadratic activations ($c=2$ in ~\cite{ding24flat}) to any piece-wise Polynomial activations.

\begin{proof}[\textbf{Proof of Lem.~\ref{lem:flattest_rescaling_full}}]
    To find the flattest empirical minimizer in $\cW_{\text{rescale}}(\vw)$ for an empirical minimizer $\vw\in \bR^{d}$, we need to solve the optimization problem,
\begin{align*}
    &\min_{\{\alpha_j\}_{j\in [m]}, \alpha_j>0} \frac{1}{n}\underset{i\in [n],j\in [m]}{\sum\sum} (\alpha_j^{2c}\sigma^2(\ip{\theta_j}{\vx_i}) + \alpha_j^{-2} a_j^2 (\sigma')^2(\ip{\theta_j}{\vx_i})\norm{\vx_i}_2^2)\\
    &= \min_{\{\alpha_j\}_{j\in [m]}, \alpha_j>0}\sum_{j\in [m]} \alpha_j^{2c} B_1(\theta_j) + \alpha_j^{-2} a_j^2 B_2(\theta_j).
\end{align*}
Here, $B_1(\theta) \coloneq \frac{1}{n}\sum_{i\in [n]}\sigma^2(\ip{\theta}{\vx_i})$ , and $B_2(\theta) \coloneq \frac{1}{n}\sum_{i\in [n]}(\sigma')^2(\ip{\theta}{\vx_i})\norm{\vx_i}_2^2$, for any $\theta\in \bS^{d-1}$. Note that $B_1, B_2 \geq 0$, $\forall \theta \in \bS^{d-1}$. As the objective is separable in $\alpha_j$, we can minimize each function of $\alpha_j$ individually. 

Consider the objective function $Q:\bR_{+}\setminus \{0\}\to \bR_{+}$. Then, 
\begin{align*}
    Q(\alpha_j) &= B_1(\theta_j)\alpha_j^{2c} + a_j^2 \alpha_j^{-2} B_2(\theta_j)\\
    \nabla Q(\alpha_j)  &= 2c B_1(\theta_j)\alpha_j^{2c-1} -2 a_j^{2}B_2(\theta_j) \alpha_j^{-3}\\
    \nabla^2 Q(\alpha_j) &= 2c(2c-1)B_1(\theta_j)\alpha_j^{2c-2} + 6 a_j^{2} B_2(\theta_j) \alpha_j^{-4}.
\end{align*}
Note that $\nabla^2 Q(\alpha_j) >0, \forall \alpha_j >0$, so the function is minimized at $\alpha_j$ where $\nabla Q(\alpha_j) = 0$. This value corresponds to 
\begin{align*}
 \alpha_j = \left(\frac{a_j^2 B_2(\theta_j)}{cB_1(\theta_j)}\right)^{\frac{1}{2c+2}}.   
\end{align*}

Therefore, the flatness of the flattest minima obtained by purely rescaling $\va$ and $\theta_j$ is given by,
\begin{align}
\sum_{j\in [m]} \abs{a_j}^{\frac{2c}{c+1}}B(\theta_j), \text{ where } B(\theta_j) \coloneq (c^{-\frac{c}{c+1}} + c^{\frac{1}{c+1}}) B_1^{\frac{1}{c+1}}(\theta_j) B_2^{\frac{c}{c+1}}(\theta_j). \label{eq:symm_opt}
\end{align}
\end{proof}

\subsection{Lower bounds on $\Upsilon(\vw)$: Proof of Thm.~\ref{thm:flattest_all}}
\label{sec:upsilon_lb_proof}

\begin{lemma}[Lower bound on $\Upsilon(\vw)$]\label{lem:min_l2_lb}
If the conditions of Lem.~\ref{lem:eigen_U} hold, for an interpolator $\vw\in \bR^{m(d+1)}$, with unit-norm inner-layer weights, $\{\theta_j\}_{j\in [m]}$, with probability $1-6\delta$, 
\begin{align*}
    \Upsilon(\vw) &\gtrsim d^{\frac{c}{c+1}} \br{\frac{\norm{\va_{\min,\ell_2}}_2^2}{\norm{\va_{\min, \ell_2}}_{\frac{2c}{c-1}}}}^{\frac{2c}{c+1}},\\
    \text{ where, } \va_{\min, \ell_2} &= \vU^\top\widetilde{\vU}\vy = \min_{\va \in \bR^{m}, \vU\va = \vy} \norm{\va}_2 .
\end{align*}
If $\frac{\norm{\va_{\min, \ell_2}}_\infty}{\norm{\va_{\min, \ell_2}}_2} = \cO(\frac{1}{\sqrt{m}})$,
\begin{align*}
    \Upsilon(\vw) &\gtrsim d^{\frac{c}{c+1}} m^{\frac{1}{c+1}} \norm{\va_{\min,\ell_2}}_2^{\frac{2c}{c+1}}.    
\end{align*}
\end{lemma}

\begin{lemma}[Lower bound on $\norm{\va_{\min, \ell_2}}_2$]\label{lem:weak_bound_l2}
    If $\zeta^2 = o(1)$, and the conditions of Lem.~\ref{lem:eigen_U} hold, then with probability $1-4\delta$,
    \begin{align*}
        \norm{\va_{\min,\ell_2}}_2 \gtrsim \sigma_{\max}^{-1}(\vU)\sqrt{n} \gtrsim \frac{1}{\sqrt{m}} .
    \end{align*}
    Further, for $\norm{\va_{\min, \ell_2}}_2 \asymp \frac{1}{\sqrt{m}}$, $\frac{\norm{\va_{\min, \ell_2}}_\infty}{\norm{\va_{\min, \ell_2}}_2} = \cO(\frac{1}{\sqrt{m}})$. 
\end{lemma}

From Lems.~\ref{lem:min_l2_lb} and \ref{lem:weak_bound_l2}, for any interpolator, with probability $1-10\delta$, we have,
\begin{align*}
    \Upsilon(\vw) \gtrsim d^{\frac{c}{c+1}} m^{-\frac{1}{c+1}} \norm{\va_{\min,\ell_2}}_2^{\frac{2c}{c+1}} \gtrsim d^{\frac{c}{c+1}} m^{-\frac{1}{c+1}} \br{\frac{1}{\sqrt{m}}}^{\frac{2c}{c+1}} = d^{\frac{c}{c+1}}m^{-\frac{c-1}{c+1}} = \Upsilon^\star.
\end{align*}
This proves Thm.~\ref{thm:flattest_all}. In the remainder of this Section, we provide the proofs for Lems.~\ref{lem:min_l2_lb} and \ref{lem:weak_bound_l2}.

\begin{proof}[\textbf{Proof of Lem.~\ref{lem:min_l2_lb}}]
    We find a lower bound on $\Upsilon(\vw)$ for a fixed set of inner-layer weights $\{\theta_j\}_{j\in [m]}$. 
Note that $\vw$ is an interpolator, so, 
\begin{align*}
    &h(\vw, \vx_i) = y_i, \forall i\in [n]\\
    &\sum_{i=1}^n a_j \sigma(\ip{\theta_j}{\vx_i}) = y_i\\
    & (\vU\va)_{i} = y_i\\
    \implies& \vU\va = \vy.
\end{align*}

From Lem.~\ref{lem:flattest_rescaling}, a lower bound on flatness in terms of only inner-layer weights requires minimizing $\Upsilon(\vw)$ in terms of $\va$. Note that this is a constrained minimization, as interpolation forces the linear constraint, $\vU\va=\vy$. The following statement holds with probability $1-9\delta$,
\begin{align*}
    \Upsilon(\vw) \geq \min_{\va\in \bR^{m}, \vU\va=\vy} \gtrsim d^{\frac{c}{c+1}} \min_{\va \in \bR^{m}, \vU\va=\vy} \norm{\va}_{\frac{2c}{c+1}}^{\frac{2c}{c+1}}.      
\end{align*}
Note that the above optimization in $\va$ minimizes a convex $\ell_p$ norm, with a linear equality constraint. Therefore, strong duality holds~\cite[Chapter~5]{boyd2004convex}, and we can use duality to find a lower bound on this objective. Using dual variables $\vr\in \bR^n$ for the interpolation constraint $\vU\va = \vy$, the primal problem has the following form.
\begin{align}
        \min_{\va \in \bR^{m}} \max_{\vr\in \bR^{n}} \norm{\va}_{{\frac{2c}{c+1}}}^{\frac{2c}{c+1}}  + \ip{\vr}{\vU \va - \vy}.\label{eq:primal}
\end{align}

As strong duality holds, the above optimal primal objective is equal to its optimal dual objective. Its dual problem is given by the following.
\begin{align}
     \max_{\vr\in \bR^{n}} \min_{\va \in \bR^{m}} \norm{\va}_{{\frac{2c}{c+1}}}^{\frac{2c}{c+1}}  + \ip{\vr}{\vU \va - \vy} \label{eq:dual}
\end{align}

We first solve the inner unconstrained minimization in terms of $\va$. Note that the Lagrangian is a convex function, therefore, we can minimize it by setting its gradient to $0$. We consider $2$ cases, when $c=1$ and $c>1$, to handle the $\ell_1$ norm and arbitrary $\ell_p$ norms for $p>1$ separately.
We first find the lower bound for $c>1$.
\paragraph{Case I : $c>1$ .}
By setting the first derivative of the Langrangian in Eq~\eqref{eq:dual} to $0$ for $c>1$, we obtain, $\forall j\in [m]$,
\begin{align*}
    &\sign(a_j)\frac{2c}{c+1} \abs{a_j}^{\frac{c-1}{c+1}}   + \sum_{i\in [n]}r_i \vU_{i,j} = 0, \\
    & \abs{a_j} = \left(-\sign(a_j)\frac{(c+1)\sum_{i\in [n]} r_i \vU_{i,j}}{2c }\right)^{\frac{c+1}{c-1}}.
\end{align*}
Plugging this value into Eq~\eqref{eq:dual}, the Langrangian becomes, 
\begin{align}
    &\norm{\va}_{{\frac{2c}{c+1}}}^{\frac{2c}{c+1}}  + \ip{\vr}{\vU \va - \vy} 
    = -\widetilde{c}\sum_{j\in [m]}\abs{\sum_{i\in [n]} r_i \vU_{i,j}}^{\frac{2c}{c-1}}  -\ip{\vr}{\vy} = -\widetilde{c}\norm{\vU^\top \vr}_{\frac{2c}{c-1}}^{\frac{2c}{c-1}} - \ip{\vr}{\vy}.  \label{eq:opt_dual}
\end{align}
Here $\widetilde{c} = (c-1)(c+1)^{\frac{c+1}{c-1}}(2c)^{-\frac{2c}{c-1}}$. The above function is a concave objective in $\vr$ as it is negative of an $\ell_p$ norm of $\vr$ with $p=\frac{2c}{c-1}$. As $p=\frac{2c}{c-1}\geq 2$, the $p$-norm is convex, and the negative of a convex function is concave.

To maximize Eq~\eqref{eq:opt_dual} with respect to $\vr$, we can set the first derivative with respect to $\vr$ to $0$. However, computing a closed-form solution to $\vr$ is not feasible for any arbitrary $c>1$. Instead, we find a lower bound to Eq~\eqref{eq:opt_dual}. As the original dual problem, Eq~\eqref{eq:dual} maximizes the objective in Eq~\eqref{eq:opt_dual} in terms of $\vr$, if we plug in any value of $\vr\in \bR^{n}$, it should be a lower bound Eq~\eqref{eq:dual}. We plug in $\vr = -\beta(\vU\vU)^{-1}\vy$ in Eq~\eqref{eq:opt_dual} for some $\beta>0$.
\begin{align*}
    &\max_{\vr\in \bR^{n}} -\widetilde{c}\norm{\vU^\top \widetilde{\vU}\vy}_{\frac{2c}{c-1}}^{\frac{2c}{c-1}} - \ip{\vr}{\vy} \nonumber\\
    &\geq \max_{\beta\in \bR} -\widetilde{c}\beta^{\frac{2c}{c-1}}\norm{\vU^\top \widetilde{\vU}\vy}_{\frac{2c}{c-1}}^{\frac{2c}{c-1}} +\beta\norm{\widetilde{\vU}^{\frac{1}{2}}\vy}_2^2 
\end{align*}
To obtain a tight lower bound, we can now maximize this lower bound in terms of $\beta$ by setting the first derivative with respect to $\beta$ to $0$, as this is a concave function $\beta$. This gives us the optimal value of $\beta^\star$ as the following. 
\begin{align*}
 \beta^\star = \br{\frac{(c-1)\norm{\widetilde{\vU}^{\frac{1}{2}}\vy}_2^2}{2c\widetilde{c}\norm{\vU^\top \widetilde{\vU}\vy}_{\frac{2c}{c-1}}^{\frac{2c}{c-1}}}}^{\frac{c-1}{c+1}}.   
\end{align*}
This is a concave function as its hessian with respect to $\beta$ is $-\widetilde{c}\frac{2c(c+1)}{(c-1)}\beta^{\frac{2}{c-1}}\norm{\vU^\top\widetilde{\vU} \vy}_2$, hence it is maximized at this value of $\beta^\star$. Plugging in this value of $\beta^\star$ and ignoring the constant terms of $c$, we obtain the following lower bound on flatness for inner layer weights $\{\theta_j\}_{j\in [m]}$.
\begin{align}
        \min_{\va\in \bR^{m}, \vU\va=\vy}\Upsilon(\vw)\gtrsim d^{\frac{c}{c+1}}\br{\frac{\norm{\widetilde{\vU}^{\frac{1}{2}}\vy}_{2}^2}{\norm{\vU^\top \widetilde{\vU}\vy}_{\frac{2c}{c-1}}}}^{\frac{2c}{c+1}} \label{eq:flat_arb} 
\end{align}

\paragraph{Case II :  $c=1$.}
We need to minimize $\ell_1$ norm of weights subject to the constraint $\vU\va=\vy$. 
The dual problem takes the following form,
\begin{align*}
    \max_{\vr\in \bR^n}\min_{\va\in \bR^d}\sum_{j\in [m]}\abs{a_j}(1 - \sign(a_j)(\vU^\top \vr)_j) - \ip{\vr}{\vy}.
\end{align*}
If $\exists j\in [m]$ such that $\abs{(\vU^\top \vr)_j}>1$, then each $\abs{a_j}$ is multiplied by a negative term, thus it can be minimized by setting $\abs{a_j}\to \infty$, and the objective would be $-\infty$ making the problem infeasible. If $\norm{\vU^\top \vr}_\infty \leq 1, \forall j\in [m]$, then, the coefficient of each $\abs{a}_j$ is non-negative. This is minimized for $\va = \vzero$. Therefore, solving the inner minimization in terms of $\va$ gives us the following bound.
\begin{align*}
    \max_{\vr\in \bR^{n}} -\ip{\vr}{\vy}, \quad \text{ s.t. } \quad \norm{\vU^\top \vr}_\infty \leq 1
\end{align*}
We can again convert this to a $1$-dimensional optimization problem by setting $\vr = -\beta \widetilde{\vU} \vy$ for some $\beta >0$. From the constraint, we obtain $\beta \leq \norm{\vU^\top \widetilde{\vU}\vy}_{\infty}^{-1}$. Therefore, a lower bound on this function is $\frac{\norm{\widetilde{\vU}^{\frac{1}{2}}\vy}_2^2}{\norm{\vU^\top\widetilde{\vU}\vy}_{\infty}}$, which can also be obtained by taking the limit of $c\to 1$ for Eq~\eqref{eq:flat_arb}. 

\paragraph{Final Bound.}
We now try to find the lower bound on RHS of Eq~\eqref{eq:flat_arb}. To do this, we consider the minimum $\ell_2$ norm interpolator of $\va$ defined as, 
\begin{align*}
    \va_{\min, \ell_2} = \argmin_{\va\in \bR^{m}, \vU\va = \vy} \norm{\va}_2.
\end{align*}
If we fix the matrix of activations $\vU$, the constraint $\vU\va = \vy$ corresponds to an interpolator of overparametrized linear regression for $\va$ as $m\geq n$. The minimum $\ell_2$-norm interpolator~\cite{bartlett_benign_2020} for overparmetrized linear regression is obtained by solving the above problem. Further, from Lem.~\ref{lem:eigen_U}, the $\sigma_{\min}(\vU)>0$, so $\vU$ is full-rank, and therefore, $\vU\vU^\top \in \bR^{n\times n}$ is invertible. In this case, the minimum $\ell_2$-norm interpolator for $\va$, is given by~\cite{bartlett_benign_2020},
\begin{align*}
    \va_{\min, \ell_2} = \vU^\top\widetilde{\vU}\vy.
\end{align*}
Additionally, the $\ell_2$ norm of this interpolator is given by,
\begin{align*}
    \norm{\va_{\min, \ell_2}}_2^2 &= \norm{\vU^\top\widetilde{\vU}\vy}_2^2 \\&= \vy^\top \widetilde{\vU} \vU\vU^\top \widetilde{\vU} \\
    &= \vy^\top \widetilde{\vU}\vy \\
    &= \norm{\widetilde{\vU}^{\frac{1}{2}}\vy}_2^2.
\end{align*}

Therefore, the RHS of Eq~\eqref{eq:flat_arb} takes the following form,
\begin{align*}
d^{\frac{c}{c+1}}\br{\frac{\norm{\widetilde{\vU}^{\frac{1}{2}}\vy}_{2}^2}{\norm{\vU^\top \widetilde{\vU}\vy}_{\frac{2c}{c-1}}}}^{\frac{2c}{c+1}} =   d^{\frac{c}{c+1}} \br{\frac{\norm{\va_{\min, \ell_2}}_2^2}{\norm{\va_{\min, \ell_2}}_{\frac{2c}{c-1}}}}^{\frac{2c}{c+1}}.
\end{align*}

If $\frac{\norm{\va_{\min, \ell_2}}_\infty}{\norm{\va_{\min, \ell_2}}_2} = \cO(\frac{1}{\sqrt{m}})$, then, by using Riesz-Thorin Interpolation Theorem~\citep{SteinShakarchi2011}, we have,
\begin{align*}
    \norm{\va_{\min, \ell_2}}_{\frac{2c}{c-1}} \leq \norm{\va_{\min, \ell_2}}_{2}^{\frac{c-1}{c}}\norm{\va_{\min, \ell_2}}_{2}^{\frac{1}{c}} \leq m^{-\frac{1}{2c}}\norm{\va_{\min, \ell_2}}_2.
\end{align*}
Plugging this in, we obtain the required final expression.
\begin{align*}
\min_{\va \in \bR^{m}, \vU\va=\vy}    \Upsilon(\vw) = d^{\frac{c}{c+1}} \br{\frac{\norm{\va_{\min, \ell_2}}_2^2}{\norm{\va_{\min, \ell_2}}_{\frac{2c}{c-1}}}}^{\frac{2c}{c+1}} \asymp d^{\frac{c}{c+1}} m^{\frac{1}{c+1}}\norm{\va_{\min, \ell_2}}_{2}^{\frac{2c}{c+1}}.
\end{align*}
This completes the proof. 
\end{proof}

\begin{proof}[\textbf{Proof of Lem.~\ref{lem:weak_bound_l2}}]
Note that, by Cauchy-Schwarz,
\begin{align*}
    \norm{\va_{\min, \ell_2}}_2^2 = \vy^\top \widetilde{\vU} \vy \geq \lambda_{\max}(\widetilde{\vU})\norm{\vy}_2^2 = \sigma_{\max}^{-2}(\vU) \norm{\vy}^2
\end{align*}
From Lem.~\ref{lem:eigen_U}, with probability, $1-\delta$, $\sigma_{\max}^{-2}(\vU)\gtrsim \frac{1}{mn}$.

Therefore, to bound $\norm{\va_{\min, \ell_2}}$, we need to bound $\norm{\vy}_2^2$. Decomposing it into the signal and noise terms $\vy^\star$ and $\vn$ respectively, we obtain,
\begin{align}
    \norm{\vy}_2^2 = \norm{\vy^\star + \vn}_2^2 = \norm{\vy^\star}_2^2 + \norm{\vn}_2^2 + 2\ip{\vn}{\vy^\star}.\label{eq:norm_y}
\end{align}
We will bound each of the $3$ terms in the above equation.
For the first term $\norm{\vy^\star}_2^2$. Note that it is the $\ell_2$ norm of an $n$-dimensional each. Each coordinate of $\vy^\star$ is an independent random variable $\sigma^\star(\Theta^\star\vx_i)$. As $\vx_i$ is Gaussian and $\Theta^\star$ is an orthonormal matrix, $\Theta^\star \vx_i \sim \cN(\vzero, \bI_{m^\star})$. Further, from Assump.~\ref{assump:link}, $\sigma^\star(\Theta^\star\vx_i)$ is a $(\frac{c}{2}, \cO(1))$- Sub-Weibull random variable, as $\norm{\bI_{m^\star}}_F = \sqrt{m^\star} = \cO(1)$, and $\E{(\sigma^\star)^2(\Theta^\star \vx_i)} = \psi^\star = \Theta(1)$. Further, $\norm{\vy^\star}_2^2$ is a quadratic form of Sub-Weibull random variables, so from Lem.~\ref{lem:hw_sub_weibull}, with probability $1-\delta$,
\begin{align*}
    \norm{\vy^\star}_2^2 \asymp n.
\end{align*}

As for the second term which is an inner product, using Sub-Gaussian concentration~\cite[Chapter~2]{Wainwright_2019} for $\vn$ with probability $1-\delta$,
\begin{align*}
    2\ip{\vn}{\vy^\star} \geq -2\zeta \norm{\vy^\star}_2 \sqrt{2\log(1/\delta)} \gtrsim -\zeta \sqrt{n}. 
\end{align*}
For the third term $\norm{\vn}_2^2$, which is a quadratic form of the gaussian random variables $\{\xi_i\}_{i\in [n]}$, by Hanson-Wright~~\cite[Thm.~1]{rudelson_vershynin_hw},  with probability $1-\delta$, we have,
\begin{align*}
    \norm{\vn}_2^2 \geq \zeta^2 n - \zeta^2 \nu_1 \max\{\sqrt{n\log(2/\delta), \log(2/\delta)}\} \gtrsim n \zeta^2.
\end{align*}

Plugging these bounds into Eq~\eqref{eq:norm_y}, along with the bounds on $\sigma_{\max}(\vU)$, with probability, $1-4\delta$, we have, 
\begin{align*}
    \norm{\va_{\min, \ell_2}}_2^2 = \vy^\top \widetilde{\vU}\vy \gtrsim \frac{1+\zeta^2}{m} - \frac{\zeta}{m\sqrt{n}} \gtrsim \frac{1}{m}.
\end{align*}
Note that $\zeta = o(1)$.

To prove the condition between, $\ell_2$ and $\ell_\infty$ norms of $\va_{\min, \ell_2}$. Note that $\sigma_{\max}(\vU)$ corresponds to the all-ones singular vector. Therefore, to achieve the lower bound on $\norm{\va_{\min, \ell_2}}_2\asymp\frac{1}{\sqrt{m}}$, we need $\va_{\min, \ell_2}$ to be almost parallel to $\vone_{m}$. Since, $\norm{\vone_{m}}_{\infty}=\frac{\norm{\vone_m}_2}{\sqrt{m}}$, any vector almost parallel to $\vone_{m}$ also satisfies the required condition.This completes the proof.

\end{proof}
In the next section, we obtain a tighter bound than Lem.~\ref{lem:weak_bound_l2} for the case of bad interpolators.
\subsection{Flatness of Bad Interpolators: Proof of Thm.~\ref{thm:flattest_bad}}
\label{sec:proof_flattest_bad}
In this section, we provide lower bounds on flatness after rescaling, $\Upsilon(\vw)$, for the class of bad interpolators in Def.~\ref{def:bad_min}. Before we describe our lower bounds, we verify that bad interpolators defined in Def.~\ref{def:bad_min} do exist.

\paragraph{Existence of Bad Interpolators.}
In Def.~\ref{def:bad_min}, we provide conditions on the inner-layer weights $\{\theta_j\}_{j\in [m]}$ and the population loss. However, we do not verify that interpolators verifying these conditions actually exist. In this section, we first show that satisfying the first condition in Def.~\ref{def:bad_min} ensures that such an interpolator always exists. Further, we provide sufficient conditions on $\rho^\star$ in Def.~\ref{def:bad_min}, such that any interpolator satisfying the third condition satisfies the second condition, i.e., its population loss is large.

Note that an interpolator does not exist for any choice of inner-layer weights. An easy counter-example to this is setting all inner layer weights $\theta_j = \theta_{\perp}^\star$, where $\theta_{\perp}^\star \in \bS^{d-1}, \theta_{\perp}^\star \perp \theta^\star$. For piece-wise polynomial activations, these weights will a.s. not interpolate in the presence of label noise ($\zeta>0$). The following Lem. shows that for a well-separated set of inner-layer weights, according to the first condition in Def.~\ref{def:bad_min}, an interpolator always exists.

\begin{lemma}[Interpolation]\label{lem:interpolation}
    For $\vw \in \bR^{m(d+1)}$, with inner layer weights $\{\theta_j\}_{j\in [m]}$ such that $\theta_j\in \bS^{d-1}, \,\forall j\in [m]$ and $\ip{\theta_j}{\theta_{j'}} \leq \rho^2, \,\forall j\neq j'\in [m]$, for some $\rho\in [0,1)$, then $\exists \{a_j\}_{j\in [m]}$, with $a_j\in \bR$ such that: 
    \begin{itemize}
        \item for piece-wise polynomial activations with $c'\neq0, c''\neq 0$, $h(\vw, \vx_i) = y_i, \forall i\in [n]$ a.s. 
        \item for piece-wise polynomial activations with either $c'=0$ or $c''=0$, $h(\vw, \vx_i) = y_i, \forall i\in [n]$ with probability $1 - \check{\nu}m^{1-\rho_0^{-2}}\rho_0^{-\rho_0^{-2}}$, for some constant $\check{\nu}>0$ and $\rho_0 = \rho(\sqrt{1-\rho^2})^{-1}$.
    \end{itemize}
\end{lemma}
The proof of this Lem. is provided in App.~\ref{sec:lem_ws_bad_proof} and uses Slepian's lemma for the second part. For piece-wise polynomial activations where none of its pieces are identically $0$, we only need the inner-layer weights to be non-identical, i.e., $\theta_j\neq \pm\theta_j,\,\forall j\neq j'\in [m]$ for interpolation. When one of the pieces of the piece-wise polynomial activation is $0$, which includes ReLU and ReLU$^c$, we cannot have all inner-layer weights very similar, otherwise all of the activations $\sigma(\ip{\theta_j}{\vx_i})$ can be identically $0$ with high probability for a given feature $\vx_i\in [n]$, while it's label $y_i\neq 0$. For well-separated inner-layer weights, for instance $\rho\leq \frac{1}{\sqrt{2}}$, even activations like ReLU$^c$, interpolation is possible with probability atleast $1 - \breve{\nu}m^{-1}$ for some constant $\breve{\nu}>0$. Setting $\rho'$ in Def.~\ref{def:bad_min} to $\rho$ in the above lemma, we prove that such an interpolator always exists.

Note that we have only shown that inner-layer weights corresponding to those in Def.~\ref{def:bad_min} lead to an interpolator. To ensure that this interpolator is actually bad, i.e., its population loss is large, we first prove a sufficient condition.
\begin{proposition}[Sufficient Conditions for Large Population Loss]\label{prop:bad_min_pop_loss}
For an interpolator $\vw$ the third condition in Def.~\ref{def:bad_min}, if $\Phi \succ 0$, $\psi$ is coordinate-wise non-decreasing function, $\phi$ is non-decreasing functions, and 
\begin{align*}
2m\psi^2(\rho_1^\star\vone_{m^\star}) \leq (1-2\kappa)\lambda_{\min}(\phi)(\zeta^2 + \psi^\star)
\end{align*},
then second condition of Def.~\ref{def:bad_min} is also satisfied, i.e, $F(\vw) - F^\star \geq \kappa(\zeta^2 + \psi^\star)$, where $\kappa \leq \frac{1}{2}$.
\end{proposition}
We assume that $F^\star = o(1)$, so that the above condition can hold. The proof of this Prop. is provided in App.~\ref{sec:rem_bad_min_pop_loss_proof}. The above sufficient condition provides an upper bound on $\rho^\star$ in terms of the functions $\psi$ and $\phi$ and the width. Crucially, a small $\rho^\star$ will ensure that the population loss will always be large. This corresponds to the inner-layer weights  have very low alignment with  the true direction.

Prop.~\ref{prop:bad_min_pop_loss} allows us to prove that for small $\rho$, these interpolators are also bad. Combining Lem.~\ref{lem:interpolation} with Prop.~\ref{prop:bad_min_pop_loss}, we have shown that bad interpolators defined in Def.~\ref{def:bad_min} do exist for specific values of $\rho^\star, \rho'$ and $\kappa$. We corroborate the existence result with an actual example of a bad interpolator according to Def.~\ref{def:bad_min}. 
\begin{example}[Bad Interpolator for Def.~\ref{def:bad_min}]\label{ex:bad_min}
    Consider the subspace perpendicular to $\vecspan(\{\theta_j^\star\}_{j\in [m^\star]})$ of dimension $d-m^\star$. We can always pick $m$ unit norm vectors $\{\theta_j\}_{j\in [m]}$ from this subspace for any $\rho' > 0$. We can find $\va\in \bR^{m}$ such that $\vw$ interpolates according to Lem.~\ref{lem:interpolation}. Further, for this bad interpolator, $\kappa \coloneq \frac{1}{2} - \frac{1}{2}\br{\Ee{\vb\sim \cN(0, \bI_{m^\star})}{\sigma^\star(\vb)}}^2(\psi^\star)^{-1}$.
\end{example}
Proof for this example is provided in App.~\ref{sec:ex_bad_min_proof}.

\begin{proof}[\textbf{Proof of Thm.~\ref{thm:flattest_bad}}]
The core contribution of this proof is the following lower bound on $\norm{\va_{\min, \ell_2}}_2$ that is tighter than Lem.~\ref{lem:weak_bound_l2}.

\begin{lemma}[Lower bound on $\norm{\va_{\min, \ell_2}}_2$ for Bad Interpolators(Def.~\ref{def:bad_min})]\label{lem:tight_bound_l_2}
    If $\zeta^2 = o(1)$, $\rho_1= o((m^\star)^{-\frac{1}{4}}$, and the conditions of Lem.~\ref{lem:eigen_U} holds, then with probability $1-7\delta$,
    \begin{align*}
        \norm{\va_{\min,\ell_2}}_2^2 \gtrsim  \sum_{i=1}^n \sigma_i^{-2}(\vU) 
         \gtrsim \frac{n}{m}.
    \end{align*}
    Further, $\frac{\norm{\va_{\min, \ell_2}}_{\infty}}{\norm{\va_{\min, \ell_2}}_{2}} = \cO(\frac{1}{\sqrt{m}})$.
\end{lemma}

Combining Lem.~\ref{lem:tight_bound_l_2} with Lem.~\ref{lem:min_l2_lb}, with probability $1-12\delta$, for a bad interpolator $\vw$ defined in Def.~\ref{def:bad_min}, we have,
\begin{align*}
    \Upsilon(\vw) \gtrsim d^{\frac{c}{c+1}} m^{-\frac{1}{c+1}} \br{\sqrt{\frac{n}{m}}}^{\frac{2c}{c+1}} = (dn)^{\frac{c}{c+1}} m^{-\frac{c-1}{c+1}}.
\end{align*}
This completes the proof.
\end{proof}

In the remainder of this section, we provide the proof for  Lem.~\ref{lem:tight_bound_l_2}.

\begin{proof}[\textbf{Proof of Lem.~\ref{lem:tight_bound_l_2}}]

Note that, 
\begin{align*}
    \norm{\va_{\min, \ell_2}}_2^2 = \vy^\top \widetilde{\vU} \vU\vU^\top \widetilde{\vU} \vy = \vy^\top \widetilde{\vU}\vy.
\end{align*}
To bound this sum, we first remove the contribution of the noise terms $\vn$. We use the following Lem. for this task.
\begin{lemma}[Removing contribution of label noise]\label{lem:flattest_bad_noise}
With probability $1-2\delta$,
\begin{align*}
\norm{\va_{\min, \ell_2}}_2^2 \geq \frac{1}{2}(\vy^\star)^\top\widetilde{\vU}\vy^\star - \zeta^2\frac{n}{m}
\end{align*}
\end{lemma}
The proof of this Lem. is provided in App.~\ref{sec:flattest_bad_noise_proof}.

The remaining term in Lem.~\ref{lem:flattest_bad_noise} depends on the true signal $\vy^\star$ and the activation matrix $\widetilde{\vU}$. A lower bound for this term can be obtained from Lem.~\ref{lem:weak_bound_l2}. However, this lower bound is loose for the set bad interpolators in Def.~\ref{def:bad_min}.

To obtain a tight lower bound for Def.~\ref{def:bad_min}, we utilize the observation that the new labels $\vy^\star$ and the new activation $\vU$ are obtained from random variables that have very low correlation from $3^{rd}$ of Def.~\ref{def:bad_min}. Therefore, there is a large component of $\vy^\star$ that is independent of the activation matrix $\vU$.

To quantify this bound, we first derive the conditional distribution of the labels $\vy^\star$ conditioned on the activations $\vU$. Let $\vU_i\in \bR^{m}$ be the $i^{th}$ row of $\vU$, $\forall i\in [n]$. Then, due to independence of samples, only the $i^{th}$ coordinate of $\vy^\star$,  $\vy_i^\star$, is not independent of $\vU_i$. Therefore, to characterize the distribution of $\vy^\star\cond \vU$, we only need to characterize the distribution of $\vy_i^\star \cond \vU_i$. The following Lem. provides this distribution.

\begin{lemma}[Distribution of $\vy_i^\star\cond \vU_i$]\label{lem:cond_distr}
If Assump.~\ref{assump:link} holds, for all $i\in [n]$,
    \begin{align*}
        &\textbf{Mean}: \norm{\E{\vy_i^\star\cond \vU_i}}_2 =\Theta\br{\norm{\widetilde{\mu}_i}_2^{c} + \norm{\varsigma}_F^{\frac{c}{2}}},\quad \quad \textbf{Variance} : \Var{\vy_1^\star\cond \vU_1}  = \Omega\br{\norm{\varsigma}_F \norm{\widetilde{\mu}_i}_2^{2c-2} +  \norm{\varsigma}_F^{c}},\\
        &\textbf{Tails} : \sigma^\star(\vb)\text{ is  a }(\frac{c}{2}, \cO(\norm{\widetilde{\mu}_i}_2^c + \norm{\varsigma}_F^{\frac{c}{2}}))\text{-Sub-Weibull random vector.}
    \end{align*}
    where $\tilde{\mu}_i \coloneq \Theta^\star \Theta^\top (\Theta\Theta^\top)^\dagger \Theta \vx_i \sim \cN(0, \bI_{m^\star}-\varsigma)$,$\varsigma = \bI_{m^\star} - \Theta^\star \Theta^\top (\Theta\Theta^\top)^\dagger \Theta (\Theta^\star)^\top$ and $A^\dagger$ is the pseudo-inverse of $A$ for a square matrix $A$ with real entries.
\end{lemma}

The proof of this Lem. is provided in App.~\ref{sec:cond_distr_proof}. We use Assump.~\ref{assump:link} to obtain all these bounds as $\vy^\star$ is obtained from the link function $\sigma^\star$.

Let $\vy_{\vU}^\star \coloneq \E{\vy^\star\cond \vU}$. Then, conditioned on $\vU$, using Hanson-Wright for Sub-Weibull random variables from Lem.~\ref{lem:hw_sub_weibull}, with probability $1-\delta$, we have,
\begin{equation}\label{eq:flattest_first}
\begin{aligned}
    (\vy^\star)^\top \widetilde{\vU} \vy^\star \geq& \sum_{i\in [n]}\widetilde{\vU}_{i,i}\Var{\vy_i^\star \cond \vU_i} +  (\vy_{\hat{\vU}}^\star)^\top \widetilde{\vU}\vy_{\hat{\vU}}^\star \\
    &\quad -  \bar{\nu}_{\frac{c}{2}} \max\{\norm{\vK\widetilde{\vU}\vK}_F\sqrt{\log(2/\delta)}, \norm{\vK\widetilde{\vU}\vK}_2(\log(2/\delta))^{c}\}
\end{aligned}
\end{equation}
where the matrix $\vK\in \bR^{n\times n}$ satisfies, $\vK \lesssim  \vK'  + \norm{\varsigma}_F^{\frac{c}{2}} \bI_n$ with the matrix $\vK'\in \bR^{n\times n}$ being the diagonal matrix, $\vK' \coloneq\diag(\norm{\widetilde{\mu}_1}_2^c, \norm{\widetilde{\mu}_2}_2^c, \cdots, \norm{\widetilde{\mu}_n}_2^c)$. Further, we can apply the bounds on $\Var{\vy_i^\star\cond\vU_i}$ and $\norm{\vy_{\vU}^\star}$ from Lem.~\ref{lem:cond_distr} to simplify each of the terms above. We use the following Lem. to obtain appropriate bounds for each of the terms in the above equation. Note that to remove the conditioning on $\vU$, we compute an upper bound on the expected value of probability of error with the expectation over $\vU$. Since it is a constant $\delta$, its expectation over $\vU$ is also $\delta$, and thus the above bound holds with probability $1-\delta$ even without the conditioning.

\begin{lemma}[Simplification]\label{lem:simplification}
Note that,
\begin{align*}
    &\sum_{i\in [n]}\tilde{\vU}_{i,i} \Var{\vy_i^\star\cond \vU_i} \gtrsim \norm{\varsigma}_F\trace((\vK')^{\frac{2c-2}{c}})\sigma_{\max}^{-2}(\vU) + \norm{\varsigma}_F^c \trace(\widetilde{\vU}),\\
    &\vy_{\vU}^\star \widetilde{\vU} \vy_{\vU}^\star \gtrsim (\trace((\vK')^2) + n\norm{\varsigma}_F) \sigma_{\max}^{-2}(\vU) +\norm{\varsigma}_F^{c} \trace(\widetilde{\vU}),\\
    &\norm{\vK\widetilde{\vU}\vK}_F \leq (\sqrt{\trace((\vK')^4)} + \norm{\varsigma}_F^c)\sqrt{\sum_{i=1}^n \sigma_{i}^{-4}(\vU)}.
\end{align*}
\end{lemma}
The proof of this Lem. is provided in App.~\ref{sec:simplification_proof}. In addition to the above simplification, we use the fact that $\norm{A}_2 \leq \norm{A}_F$ for any PSD matrix $A$. Applying the above simplification to Eq.~\eqref{eq:flattest_first}, with probability $1-\delta$ conditioned on $\vU$, we obtain,
    \begin{align*}
        \vy^\star \widetilde{\vU}\vy^\star \gtrsim& \norm{\varsigma}_F^c (\trace(\widetilde{\vU}) + n\sigma_{\max}^{-2}(\vU)) + \sigma_{\max}^{-2}(\vU)(\trace((\vK')^{\frac{2c-2}{c}})\norm{\varsigma}_F + \trace((\vK')^4))\\
        &\quad - \bar{\nu}_{\frac{c}{2}}\sqrt{\sum_{i=1}^n \sigma_{i}^{-4}(\vU)}(\sqrt{\trace((\vK')^4)} + \norm{\varsigma}_F^c)(\log(2/\delta))^c.
    \end{align*}
We need to bound the terms of $\vK$ and the terms of activation matrices, $\vU$ and $\widetilde{\vU}$. From Lem.~\ref{lem:eigen_U}, with probability  $1-\delta$, we have, $\trace(\widetilde{\vU}) = \sum_{i=1}^n \sigma_{i}^{-2}(\vU) \asymp \frac{n}{m}$, $\sqrt{\sum_{i=1}^n \sigma_{i}^{-4}(\vU)} \asymp\frac{\sqrt{n}}{m}$, and $\sigma_{\max}^{-2}(\vU)\asymp \frac{1}{mn}$. Plugging these bounds into the above equation, with probability $1-2\delta$, conditioned on $\vU$, we obtain,
\begin{equation}\label{eq:flattest_second}
    \begin{aligned}
        \vy^\star \widetilde{\vU}\vy^\star \gtrsim& \norm{\varsigma}_F^c \frac{n}{m} + \frac{1}{mn}(\trace((\vK')^{\frac{2c-2}{c}})\norm{\varsigma}_F + \trace((\vK')^2))\\
        &\quad - \bar{\nu}_{\frac{c}{2}}\frac{\sqrt{n}}{m}(\sqrt{\trace((\vK')^4)} + \norm{\varsigma}_F^c)(\log(2/\delta))^c.
    \end{aligned}
\end{equation}
Note that the only random variable in the lower bound is obtained from the matrix $\vK'$. The following lemma provides a bound on this.
\begin{lemma}[Bounds on $\vK'$]\label{lem:k_bound}
For any $\varphi\geq 0$, with probability $1-\delta$, 
    \begin{align*}
        \trace((\vK')^{\varphi}) \asymp n\norm{\bI_{m^\star} - \varsigma}_F^{\frac{c\varphi}{2}}
    \end{align*}
\end{lemma}
The proof of this Lem. is provided in App.~\ref{sec:k_bound_proof}.
We can apply the above Lem. for $\varphi = \frac{2(c-1)}{c},2$ and $4$ to in Eq.~\eqref{eq:flattest_second}, with probability $1-5\delta$, we have,
\begin{equation}\label{eq:flattest_third}
    \begin{aligned}
    \vy^\star \widetilde{\vU}\vy^\star \gtrsim& \norm{\varsigma}_F^c \frac{n}{m} + \frac{\norm{\bI_{m^\star}-\varsigma}_F^{c-1}}{mn}(\norm{\varsigma}_F + \norm{\bI_{m^\star}-\varsigma}_F)\\
    &\quad - \bar{\nu}_{\frac{c}{2}}\frac{\sqrt{n}}{m}(\sqrt{n}\norm{\bI_{m^\star}-\varsigma}_F^{c} + \norm{\varsigma}_F^c)(\log(2/\delta))^c.
    \end{aligned}
\end{equation}
We can now see sufficient conditions for $\vy^\star \widetilde{\vU}\vy^\star \gtrsim \frac{n}{m}$. This is $\norm{\bI_{m^\star} - \varsigma}_F = o(1)$.

We first find an upper bound on $\norm{\bI_{m^\star} - \varsigma}_F$.

\paragraph{Upper bound on $\norm{\bI_{m^\star} - \varsigma}_F$}
Note that,
\begin{align*}
    \norm{\bI_{m^\star} - \varsigma}_F &= \norm{\Theta^\star \Theta^\top (\Theta\Theta^\top)^{\dagger}\Theta (\Theta^\star)^\top}_F .
\end{align*}

Let $\Theta^\top(\Theta\Theta^\top)\Theta(\Theta^\star)^\top\coloneq P_{\Theta}$. Note that $P_{\Theta} \in \bR^{d\times d}$ projects any $d$-dimensional vector to the subspace of $\vecspan(\{\theta_j\}_{j\in [m]})$. Therefore, the second term can be bounded in the following way,
\begin{align*}
    \norm{\Theta^\star \Theta^\top (\Theta\Theta^\top)^{\dagger}\Theta (\Theta^\star)^\top}_F &= \norm{\Theta^\star P_{\Theta} (\Theta^\star)^\top}_F = \sqrt{\sum_{j,j'\in [m^\star]} \ip{\theta_{j}^\star}{P_{\Theta}\theta_{j'}^\star}^2}\\
    &= \sqrt{\sum_{j\in [m^\star]} \ip{\theta_{j}^\star}{P_{\Theta}\theta_{j}^\star}^2}.
\end{align*}
We use the fact that $\ip{\theta_{j}^\star}{\theta_{j'}^\star} = 0$ if $j\neq j'$. Note that $P_{\Theta}\theta_j^\star$ is a vector in $\vecspan(\{\theta_j\}_{j\in [m]})$. Therefore, by Def.~\ref{def:bad_min},
\begin{align*}
 \ip{\theta_j^\star}{P_{\Theta}\theta_j^\star}^2 \leq (\rho^\star)^2\norm{P_{\Theta}\theta_j^\star}^2 = \rho^4
\end{align*}
We use the fact that $\norm{P_{\Theta}\theta_j^\star} = \abs{\ip{\vv}{\theta_j^\star}}\leq \rho^\star$ for some $\vv\in \bS^{d-1}$ and $\vv\in \vecspan(\{\theta_j\}_{j\in [m]})$ 

Using these bounds, we find that,
\begin{align*}
    \norm{\Theta^\star \Theta^\top (\Theta\Theta^\top)^{\dagger}\Theta (\Theta^\star)^\top}_F \leq \sqrt{(\rho^\star)^4 \sum_{j\in [m^\star]}}= (\rho^\star)^2 \sqrt{m^\star}.
\end{align*}
From the condition in Thm.~\ref{thm:flattest_bad}, $\rho^\star = o((m^\star)^{-\frac{1}{4}}$, so $\norm{\bI_{m^\star} - \varsigma}_F = o(1)$.
Therefore, plugging this bound into Eq.~\eqref{eq:flattest_third}, with probability $1-5\delta$, we have,
\begin{align*}
    (\vy^\star)^\top \widetilde{\vU}\vy^\star \gtrsim \frac{n}{m}.
\end{align*}

From Lem.~\ref{lem:flattest_bad_noise}, with probability $1-6\delta$, we have,
\begin{align*}
    \norm{\va_{\min, \ell_2}}_2^2 \gtrsim \frac{n}{m} - \zeta^2\frac{n}{m} \gtrsim \frac{n}{m}.
\end{align*}
Using $\zeta^2 = o(1)$ in the last step provides our required bound.

We now prove the bound on $\frac{\norm{\va_{\min, \ell_2}}_{\infty}}{\norm{\va_{\min, \ell_2}}_2}$. If $\ve_{j}\in \bR^{m}$ is the $j^{th}$ coordinate vector, then, we can bound,
\begin{align*}
    \norm{\va_{\min, \ell_2}}_{\infty} = \max_{j\in [m]}\abs{\ip{\ve_j}{\vU^\top \widetilde{\vU}\vy}} 
\end{align*}
First, conditioned on $\vU$, using Lemma~\ref{lem:sub_weibull_sum}, with probability $1-\delta$, by taking a union bound over all $m$ coordinates, we obtain,
\begin{align*}
\max_{j\in [m]} \abs{\ip{\ve_{j}}{\vU^\top\widetilde{\vU}}\vy} \lesssim \max_{j\in [m]}\br{\abs{\ip{\ve_{j}}{\vU^\top\widetilde{\vU}}\vy_{\vU}^\star}  + (\norm{\widetilde{\vU}\vU \ve_j}_{\infty}  + \zeta^2 \norm{\widetilde{\vU}\vU\ve_j}_2)\log(\frac{m}{\delta}) }
\end{align*}
We use the fact that $\rho^\star = \cO((m^\star)^{-\frac{1}{4}}$ and $\zeta^2 = o(1)$ to eliminate the first and the third term. Then, we bound the dominating second term with $\ell_{\infty}$ norm by $\ell_2$ norm. We also absorb the $\log(m/\delta)$ term inside $\lesssim$ notation. 

Finally, we require the following bound,
\begin{align*}
    \norm{\widetilde{\vU}\vU\ve_j}_2^2 = \ve_j^\top \vU^\top \widetilde{\vU}^2 \vU \ve_j = \widetilde{\vU}_{j,j} \leq \lambda_{\max}(\widetilde{\vU}).
\end{align*}
Using Lemma~\ref{lem:eigen_U}, we obtain, $\norm{\va_{\min, \ell_2}}_{\infty} \lesssim \frac{1}{\sqrt{m}}$. From our previous analysis, $\norm{\va_{\min, \ell_2}}_2 \gtrsim \sqrt{\frac{n}{m}}$, therefore, 
\begin{align*}
    \frac{\norm{\va_{\min, \ell_2}}_{\infty}}{\norm{\va_{\min, \ell_2}}_2}\lesssim \frac{1}{\sqrt{n}} =\cO(\frac{1}{\sqrt{m}}).
\end{align*}
We finally use $n= \Theta(m)$.
This completes the proof.

\end{proof}

In the next section, we find conditions under which an interpolator $\vw$ achieves the minimum flatness $\Upsilon^\star$ and characterize its population loss.
\subsection{Population loss of Flattest Interpolators : Proof of Thm.~\ref{thm:gen_all_good}}
\label{sec:proof_flattest_pop}
Note that Assump.~\ref{assump:Lipschitz-like} quantifies sufficient conditions to achieve the minimum flatness. There is a simple instance of our problem setting that achieves this condition, which serves as a motivation for both Assump.~\ref{assump:Lipschitz-like} and Prop.~\ref{rem:necessary_flatness}. We state it here for completeness.

\begin{theorem}[Flattest Good Minima for Learning Activation Without Noise]\label{thm:flattest_good}
For $\zeta=0$, the following interpolator $\vw\in \bR^{m(d+1)}$ has $F(\vw) = 0$, and with high probability,
$\Upsilon(\vw) \asymp \Upsilon^\star$. 
\begin{itemize}[nosep]
    \item Single-Index : $\sigma^\star = \sigma$, $a_j = \frac{1}{m}$$, \theta_j=\theta^\star, \forall j\in [m]$,
    \item Sum of Single-Index  : $\widetilde{\sigma}_{j'}^\star = \sigma$, $a_{j} = \frac{a_{j'}^\star m^\star}{m}$ and $\theta_{j}=\theta_{j'}^\star, \forall j\in \bs{\frac{(j'-1)m}{m^\star} + 1, \frac{jm}{m^\star}}, \forall j'\in [m^\star]$, where $m$ is divisible by $m^\star$.
\end{itemize}
\end{theorem}
The proof of this Thm. is provided in App.~\ref{sec:proof_thm_flattest_good}. 

Before stating the proof of Thm.~\ref{thm:gen_all_good}, we first justify the bound on $\va^\star$ in Assump.~\ref{assump:Lipschitz-like}.
\paragraph{Optimal outer-layer weights for a single-index model.}
For a single index-model in Section~\ref{sec:setup}, the population loss is minimized when all $\theta_j = \theta^\star$. Let $\sum_{j=1}^m a_j = \bar{a}$, then, the output of the network in this case is $h(\vw, \vx) = \bar{a}\sigma(\ip{\theta^\star}{\vx}), \forall \vx \in \bR^d$.
Then, the population loss is, from App.~\ref{sec:prelims}, is given by.
\begin{align*}
    2F(\vw) = \zeta^2 + \Ee{\vx}{(\sigma^\star(\ip{\theta^\star}{\vx})-\bar{a}\sigma(\ip{\theta^\star}{\vx}))^2} = \zeta^2 + \psi^\star +\bar{a}^2 \phi^\star -2\bar{a}\widetilde{\psi}(1).
\end{align*}
The value of $\bar{a}$ that minimizs the above quadratic in $\bar{a}$ is given by $\bar{a}^\star = \frac{\widetilde{\psi}(1)}{\phi^\star}$.
Plugging $\bar{a}^\star$ into the optimal population loss, we obtain,
\begin{align*}
    2F^\star = \zeta^2 + \psi^\star - \frac{\widetilde{\psi}^2(1)}{\phi^\star}
\end{align*}
For our Assump.~\ref{assump:Lipschitz-like}, the approximation error is $\Delta \coloneq \psi^\star - \frac{\widetilde{\psi}^2(1)}{\phi^\star}$. Further, the coefficients satisfy the following equation by Cauchy-Schwartz,
\begin{align*}
    \bar{a}^\star = \frac{\Ee{b\sim\cN(0,1)}{\sigma^\star(b)\sigma(b)}}{\phi^\star} \leq \frac{\sqrt{\Ee{b\sim \cN(0,1)}{(\sigma^\star)^2(b)}\Ee{b\sim \cN(0,1)}{\sigma^2(b)}}}{\phi^\star} = \frac{\sqrt{\psi^\star \phi^\star}}{\phi^\star} = \sqrt{\frac{\psi^\star}{\phi^\star}}. 
\end{align*}

Now, we prove Thm.~\ref{thm:gen_all_good}. Note that this Thm. contains an existence and generalization result. We restate these two results separately below. Note that we need

\begin{lemma}[Existence of a Flattest Interpolator]\label{lem:good_flat_exists}
    Suppose Assump.~\ref{assump:activation}, ~\ref{assump:link} and \ref{assump:Lipschitz-like} hold, and $m \geq 2n$. Then, with high probability, $\exists \hat{\vw}^\star = \begin{bmatrix}
        (\hat{\va}^\star)^\top, (\hat{\theta}_1^\star)^\top, \ldots, (\hat{\theta}_m^\star)^\top
    \end{bmatrix}$ such that $\hat{\theta}_j^\star \in \bS^{d-1},\,\forall j\in [m]$, $\hat{\vw}^\star$ interpolates and $\norm{\hat{\va}^\star}_\infty\leq\frac{1}{m}\br{\sqrt{\frac{\psi^\star}{\phi^\star}} + \gamma n^{-\min\{\epsilon_1,\epsilon_2\}}}$.%
\end{lemma}

\begin{theorem}[Population Loss of Flattest Interpolator]\label{thm:gen_all_good_orig}
Suppose Assump.~\ref{assump:activation} and ~\ref{assump:link} hold, and $m=\Omega(n^c)$ and $m=\cO(\poly(n)$. Then, for an interpolator $\vw\in \bR^{m(d+1)}$ with unit norm inner-layer weights, if $\norm{\va}_{\infty} \leq \frac{1}{m}\br{\sqrt{\frac{\psi^\star}{\phi^\star}} + \gamma n^{-\epsilon}}$ for some constants $\gamma=\cO(1)$ and $\epsilon>0$, then, with high probability, $F(\vw) \lesssim n^{-\min\{\frac{1}{2}, \epsilon\}}$.    
\end{theorem}

 Essentially, interpolation with approximation error and label noise increases $\norm{\va}_\infty$ and thus flatness $\Upsilon(\vw)$ for any interpolator $\vw$ by a term proportional to $n m(\Delta^2 + \zeta^2)$. Hence, we require these terms to be small to achieve a small approximation error. Further, we plug in $\epsilon = \min\{\epsilon_1, \epsilon_2\}$ in Thm.~\ref{thm:gen_all_good_orig} to obtain the final result. 

We now provide the proof of Lem.~\ref{lem:good_flat_exists} and Thm.~\ref{thm:gen_all_good_orig}, and defer the proof of Prop.~\ref{rem:necessary_flatness} to App.~\ref{sec:rem_necessary_flatness_proof}.

\begin{proof}[\textbf{Proof of Lem.~\ref{lem:good_flat_exists}}]
\label{sec:thm_good_flat_exists_proof}
We divide the $m$ layer weights into $2$ parts. The first part is used to learn the signal and the rest is used to learn noise and approximation error.

We assume that $m$ is even, so that $\frac{m}{2}\in \bN$. 
For the first $\frac{m}{2}$ layer-weights, we set it to exactly the optimal approximation in Assump.~\ref{assump:Lipschitz-like}.
\begin{align*}
    \hat{\theta}_{j} = \breve{\theta}_{j'},\quad  \forall j\in \bs{\frac{m(j'-1)}{2\breve{m}}+1, \frac{mj'}{2\breve{m}}}, \forall j'\in [\breve{m}].
\end{align*}
Further, we set the first $\frac{m}{2}$ outer-layer weights to also match those in Assump.~\ref{assump:Lipschitz-like} with an appropriate scaling by $m$. Define $\hat{\va}\in \bR^{m}$ such that,
\begin{align*}
    \hat{a}_j = \frac{2\breve{m}\breve{a}_{j'}}{m},\quad \forall j\in \bs{\frac{m(j'-1)}{2\breve{m}}+1, \frac{mj'}{2\breve{m}}}, \quad \forall j'\in [\breve{m}].
\end{align*}
Note that this choice matches the optimal weights corresponding to the sum of single-index case in Thm.~\ref{thm:flattest_good}. Summing up the first $\frac{m}{2}$ layer weights, we obtain, 
\begin{align*}
    \sum_{j=1}^{\frac{m}{2}} \frac{2\breve{m}\breve{a}_j}{m}\sigma(\ip{\hat{\theta}_j}{\vx_i}) = \frac{m}{2\breve{m}}\frac{2\breve{m}}{m} \sum_{j=1}^{\breve{m}}\breve{a}_j \sigma(\ip{\breve{\theta}_j}{\vx_i}). 
\end{align*}
Define the vector $\vv\in \bR^n$, as
\begin{align*}
(\vv)_i = \sigma^\star(\Theta^\star\vx_i) - \sum_{j=1}^{\breve{m}}\breve{a}_j \sigma(\ip{\breve{\theta}_j}{\vx_i}).  
\end{align*}
Then, from Assump.~\ref{assump:Lipschitz-like}, each coordinate of $\vv$ has absolute value at most $\Delta$.

To ensure interpolation, we need $\vU\va = \vy$. If $\breve{\vU}\in \bR^{n\times \frac{m}{2}}$ corresponds to the activations of the last $\frac{m}{2}$ weights, and $\widetilde{\va}\in \bR^{\frac{m}{2}}$ is their corresponding outer-layer weights, we can ensure interpolation as long as,
\begin{align*}
    \breve{\vU}\widetilde{\va} = \vv + \vn
\end{align*}
Note that $\frac{m}{2}\geq n$, so we have enough overparametrization to fit any $n$-dimensional target. We select the last $\frac{m}{2}$ inner-layer weights such that we satisfy Assump.~\ref{assump:width} and the minimum separation conditions in Lem.~\ref{lem:interpolation} for them. This ensures that Lem.~\ref{lem:eigen_U} holds for $\breve{\vU}$ as well.

We will set $\widetilde{\va}$ as the minimum $\ell_2$ norm interpolator. Thus,
\begin{align*}
    \widetilde{\va} = \breve{\vU}^\top(\breve{\vU}\breve{\vU}^\top)^{-1}(\vv+\vn).
\end{align*}

Now, we characterize the $\norm{\va}_\infty$ of this interpolator. For the first $\frac{m}{2}$ outer-layer weights, from Assump.~\ref{assump:Lipschitz-like}, their $\ell_\infty$ norm is bounded by $\frac{2\breve{m}}{m} \frac{1}{2\breve{m}}\sqrt{\frac{\psi^\star}{\phi^\star}}\leq \sqrt{\frac{\psi^\star}{\phi^\star}}$.

For the last $\frac{m}{2}$ layer weights, we have, 
\begin{align*}
    \norm{\widetilde{\va}}_\infty = \norm{\widetilde{\va}}_2 \leq  \sqrt{(\vv+\vn)^\top(\breve{\vU}\breve{\vU}^\top)^{-1}(\vv + \vn)}.
\end{align*}
As $\vn$ is independent of $\vv$, conditioned on $\vv$, $\vn +\vv \sim \cN(\vv, \bI_n)$.
By applying Hanson-Wright~~\cite[Thm.~1]{rudelson_vershynin_hw} for the gaussian random vector $\vn$, with probability $1-\delta$,
\begin{align*}
    (\vn +\vv)^\top (\breve{\vU}\breve{\vU}^\top)^{-1}(\vn +\vv) &\leq \zeta^2 \trace((\breve{\vU}\breve{\vU}^\top)^{-1})+ \vv (\breve{\vU}\breve{\vU}^\top)^{-1}\vv \\
    &\quad + \zeta^2 \nu \max\bc{\norm{(\breve{\vU}\breve{\vU}^\top)^{-1}}_F\sqrt{\log(2/\delta)}, \norm{(\breve{\vU}\breve{\vU}^\top)^{-1}}_2\log(2/\delta)}\\
&\lesssim \zeta^2 \frac{n}{m} + \norm{\vv}_2^2 \frac{1}{m}+ \zeta^2\frac{\sqrt{n}}{m} \lesssim \zeta^2 \frac{n}{m}+\Delta^2 \frac{n}{m}.
\end{align*}
From Lem.~\ref{lem:eigen_U} and the bounds computed in App.~\ref{sec:proof_flattest_bad}, we know that with probability $1-\delta$, $\trace((\breve{\vU}\breve{\vU}^\top)^{-1}) \asymp \frac{n}{m}, \sigma_{\min}^{-2}(\vU) \asymp\frac{1}{m}, \norm{(\breve{\vU}\breve{\vU}^\top)^{-1}}_2 \leq \norm{(\breve{\vU}\breve{\vU}^\top)^{-1}}_F$ and $\norm{(\breve{\vU}\breve{\vU}^\top)^{-1}}_F\asymp \frac{\sqrt{n}}{m}$. We plug all these bounds into the equation above. 

Now, we plug in the values of $\Delta$ and $\zeta$ from Assump.~\ref{assump:Lipschitz-like}.
Therefore, with probability $1-2\delta$, 
\begin{align*}
    \norm{\widetilde{\va}}_2^2 \lesssim \frac{n(\zeta^2 + \Delta^2)}{m} \lesssim \frac{n^{-2\min\{\epsilon_1, \epsilon_2\}}}{m^2}. 
\end{align*}
Plugging in the bounds on $\Delta^2$ and $\zeta^2$ from Assump.~\ref{assump:Lipschitz-like}, we obtain,
\begin{align*}
    \norm{\widetilde{\va}}_\infty \lesssim \frac{n^{-\min\{\epsilon_1, \epsilon_2\}}}{m}
\end{align*}

Therefore, with high probability,
\begin{align*}
    \norm{\va}_\infty \leq m^{-1}\br{\sqrt{\frac{\psi^\star}{\phi^\star}} + \gamma n^{-\min\{\epsilon_1, \epsilon_2\}}}
\end{align*}
for some constant $\gamma>0$.

Plugging this expression in Lem.~\ref{lem:flattest_rescaling}, with high probability, we obtain,
\begin{align*}
\Upsilon(\vw) \lesssim d^{\frac{c}{c+1}}m^{-\frac{c-1}{c+1}} \br{1 + \gamma' n^{-\frac{2c\min\{\epsilon_1, \epsilon_2\}}{c+1}}}
\end{align*}
for some constant $\gamma'>0$. This completes the proof.
\end{proof}

\begin{proof}[\textbf{Proof of Thm.~\ref{thm:gen_all_good_orig}}]

From App.~\ref{sec:prelims}, the population risk is given by,
\begin{align*}
    F(\vw) = \frac{1}{2}\br{\zeta^2 + \psi^\star + \va^\top \Phi\va - 2\ip{\va}{\widetilde{\Psi}}}
\end{align*}
We need an upper bound on $F(\vw)$ under the constraint 
$\norm{\va}_\infty \leq m^{-1}\br{\sqrt{\frac{\psi^\star}{\phi^\star}} + \gamma n^{-\epsilon}}$.

We first bound the quadratic term in $\va$.
\begin{align*}
    \va^\top \Phi\va \leq& \norm{\va}_2^2 \lambda_{\max}(\Phi) \leq \norm{\va}_2^2 m \phi^\star \\
    \leq& m^2 \norm{\va}_\infty^2 \phi^\star = \frac{m^2\psi^\star ( 1+ 3\gamma n^{-\epsilon})}{m^2}\\
    \leq &\psi^\star ( 1+ 3\gamma n^{-\epsilon})
\end{align*}
Since each element of the symmetric matrix $\Phi\in \bR^{m\times m}$ is $\leq \phi^\star$, its maximum eigenvalue is $m\phi^\star$.

We now need to maximize the linear term $-\ip{\va}{\widetilde{\Psi}}$ in terms of $\va$ under the constraint on $\norm{\va}_\infty$, and interpolation ($\vU\va = \vy$). This is equivalent to minimizing the negative of this linear term under the same constraints.

Therefore, we need to find a valid lower bound on the following objective,
\begin{align*}
    \min_{\va\in \bR^{m}} \ip{\va}{\widetilde{\Psi}}\quad \text{ s.t. } \norm{\va}_\infty \leq m^{-1}\br{\sqrt{\frac{\psi^\star}{\phi^\star}} + \gamma n^{-\epsilon}} \quad \text{ and } \vU\va = \vy
\end{align*}

Note that this is a linear objective with convex and linear constraints, hence strong duality holds for this problem. Using lagrange multipliers  $\vr\in \bR^n$ and $\bar{r}\in \bR_{+}$ for the equality and $\ell_\infty$ norm constraints respectively, the primal objective is given by,
\begin{align*}
    \min_{\va\in \bR^n}\max_{\vr\in \bR^n, \bar{r}\in \bR_{+}} \ip{\va}{\widetilde{\Psi}} + \bar{r}\br{\norm{\va}_\infty - m^{-1}\br{\sqrt{\frac{\psi^\star}{\phi^\star}} + \gamma n^{-\epsilon}}} +  \ip{\vr}{\vU\va - \vy}
\end{align*}
The dual of this problem is given by,
\begin{align*}
    \max_{\vr\in \bR^n, \bar{r}\in \bR_{+}}\min_{\va\in \bR^n} \ip{\va}{\widetilde{\Psi}} + \bar{r}\br{\norm{\va}_\infty - m^{-1}\br{\sqrt{\frac{\psi^\star}{\phi^\star}} + \gamma n^{-\epsilon}}} +  \ip{\vr}{\vU\va - \vy}
\end{align*}
Since the objective function is a convex function in $\va$, we can minimize it.
Collecting the terms of only $\va$, we need to minimize,
\begin{align*}
    \ip{\va}{\widetilde{\Psi} + \vU^\top \vr} + \bar{r}\norm{\va}_\infty \geq \norm{\va}_\infty\br{\bar{r} - \norm{\widetilde{\Psi} + \vU^\top \vr}_1} 
\end{align*}
If the term multiplied by $\norm{\va}_\infty$ is $<0$, then we can set $\norm{\va}_\infty \to \infty$, so that the dual problem is not feasible. For it to be feasible, we need the coefficient of $\norm{\va}_\infty$ to be non-negative, in which case we setting $\norm{\va}_\infty = 0$ minimizes the problem in terms of $\va$.
This requires, $\bar{r} \geq \norm{\widetilde{\Psi} + \vU^\top \vr}_1$.
Therefore, the dual problem becomes,
\begin{align*}
\max_{\vr\in \bR^n, \bar{r}\in \bR_{+}} -\ip{\vr}{\vy} -\bar{r}m^{-1}\br{\sqrt{\frac{\psi^\star}{\phi^\star}} + \gamma n^{-\epsilon}}, \quad \text{ s.t. } \bar{r} \geq \norm{\widetilde{\Psi} + \vU^\top \vr}_1
\end{align*}
By strong duality, the maximum of this linear objective has the same value as the minimum of the primal objective.This problem again does not have a closed-form solution. If we plug in a value of $\vr$ and $\bar{r}$, we get a lower bound to the optimal dual objective and thus also a lower bound to optimal primal objective.

We set $\vr = -\frac{\vy}{n}$ and $\bar{r} = \norm{\widetilde{\Psi} + \vU^\top \vr}_1$. The lower bound on the objective is ,
\begin{align*}
    \frac{1}{n}\norm{\vy}_2^2 - \frac{\norm{\widetilde{\Psi}- \frac{1}{n}\vU^\top \vy}_1}{m}\br{\sqrt{\frac{\psi^\star}{\phi^\star}} + \gamma n^{-\epsilon}} 
\end{align*}

We need a lower bound on $\norm{\vy}_2^2$ and an upper bound on $\norm{\widetilde{\Psi}+ \vU^\top \vr}_1$.

\paragraph{Lower bound on $\norm{\vy}_2^2$}
Note that $\norm{\vy}_2^2 = \norm{\vy^\star + \vn}_2^2$. Conditioned on $\vy^\star$, $\vy\sim \cN(\vy^\star, \zeta^2\bI_n)$. Therefore, using Hanson-Wright~\cite[Thm.~1]{rudelson_vershynin_hw}, conditioned on $\vy^\star$, with probability $1-\delta$,
\begin{align*}
    \norm{\vy}_2^2 \geq \norm{\vy^\star}_2^2 + n\zeta^2 - n\zeta^2\sqrt{\log(2/\delta)} \geq \norm{\vy^\star}_2^2 
\end{align*}
For $\zeta^2 = \cO(n^{-1 - 2\epsilon_1})$.
Without conditioning, the probability of error is $\Ee{\vy^\star}{\delta\cond\vy^\star} = \delta$.

Since $\norm{\vy^\star}_2^2$ is a quadratic form of independent $(\frac{c}{2}, \cO(1))$-Sub-Weibull random variables, using Hanson-Wright (Lem.~\ref{lem:hw_sub_weibull}), with probability $1-\delta$,
\begin{align*}
    \norm{\vy^\star}_2^2 \gtrsim n\psi^\star - \sqrt{n\log(2/\delta)} 
\end{align*}

\paragraph{Upper bound on $\norm{\widetilde{\Psi} - \frac{1}{n}\vU^\top \vy}_1$}
Let's consider the $j^{th}$ coordinate of this vector using the definition of $\widetilde{\Psi}$.
\begin{align*}
    &\E{\sigma^\star(\Theta^\star\vx)\sigma(\ip{\theta_j}{\vx})} - \frac{1}{n}\sum_{i=1}^n \sigma(\ip{\theta_j}{\vx_i})(\sigma^\star(\Theta^\star\vx_i) + \xi_i)\\
    &=\underset{I_j}{\underbrace{\E{\sigma^\star(\Theta^\star\vx)\sigma(\ip{\theta_j}{\vx})} - \frac{1}{n}\sum_{i=1}^n \sigma(\ip{\theta_j}{\vx_i})\sigma^\star(\Theta^\star\vx_i)}}  + \underset{II_j}{\underbrace{\frac{1}{n}\sum_{i=1}^n \sigma(\ip{\theta_j}{\vx_i})\xi_i}}.
\end{align*}
This is the difference between the empirical mean and the expectation of a random variable  $\sigma(\ip{\theta_j}{\vx_i})(\sigma^\star(\Theta^\star\vx_i) + \xi_i)$. We bound the two terms separately.

\textbf{Bound on $I_j$}

Note that $\sigma(\ip{\theta_j}{\vx_i})\sigma^\star(\Theta^\star\vx_i)$ is a product of two $(\frac{c}{2}, \cO(1))$-Sub-Weibull random variables. Therefore, independent of their correlation, this is a $(c, \cO(1))$-Sub-Weibull random variable (Lem.~\ref{lem:sub_weibull_sum_prod}). 

Using Lem.~\ref{lem:sub_weibull_sum}, with probability $1-\delta$, 
\begin{align*}
    I_j \lesssim \sqrt{\frac{\log(2m/\delta)}{n}},\quad \forall j\in [m]
\end{align*}
The term $\log(2m/\delta)$ comes due to a union bound over all $j\in [m]$ with probability of error $\delta/m$ for each $j\in [m]$.

\textbf{Bound on $II_j$}

Conditioning on $\{\vx_i\}_{i\in [n]}$, with probability $1-\delta$, due to concentration of the gaussian $\xi_i$'s~\cite[Chapter~2]{Wainwright_2019}, with probability $1-\delta$, we have,
\begin{align*}
    II_j \leq \zeta\frac{\sqrt{\log(2m/\delta)\sum_{i=1}^n \sigma^2(\ip{\theta_j}{\vx_i})}}{n} 
\end{align*}
We again take a union bound over all $j\in [m]$, giving the term of $\log(2m/\delta)$.

Without conditioning, the probability of error is also $\delta$ for this bound.

Note that $\sigma^2(\ip{\theta_j}{\vx_i})$ is a $(c, \cO(1))$-Sub-Weibull random variable (Def.~\ref{def:sub_weibull}), therefore, using Lem.~\ref{lem:sub_weibull_sum}, with probability $1-\delta$,
\begin{align*}
    \sum_{i=1}^n \sigma^2(\ip{\theta_j}{\vx_i}) \lesssim n, \quad, \forall j\in [m]
\end{align*}

This provides the following bound on $II_j$,
\begin{align*}
    II_j \lesssim \frac{\zeta}{\sqrt{n}} \lesssim \frac{1}{n}
\end{align*}
For the final step we use $\zeta \lesssim n^{-\frac{1}{2}}$.

The final upper bound on $\norm{\widetilde{\Psi} - \frac{1}{n}\vU^\top \vy}_1$, with probability $1-3\delta$, is,
\begin{align*}
    \norm{\widetilde{\Psi} - \frac{1}{n}\vU^\top \vy}_1 \leq& \sum_{j=1}^m \abs{(\widetilde{\Psi} - \frac{1}{n}\vU^\top \vy)_j} \leq \sum_{j=1}^m (\abs{I_j} + \abs{II_j})\\
    \lesssim& \frac{m}{\sqrt{n}}
\end{align*}

\paragraph{Combining the bounds}
Combining the upper and lower bounds, the lower bound on the optimal dual objective with probability $1-5\delta$, is
\begin{align*}
    \frac{1}{n}\norm{\vy}_2^2 - \frac{\norm{\widetilde{\Psi}- \frac{1}{n}\vU^\top \vy}_1}{m}\br{\sqrt{\frac{\psi^\star}{\phi^\star}} + \gamma n^{-\epsilon}} 
  \gtrsim\psi^\star - \frac{\gamma'}{\sqrt{n}} .
\end{align*}

\paragraph{Final Bound on $F(\vw)$}
Plugging in the bounds on the terms of $\va$, the upper bound on population loss with probability $1-5\delta$ is the following,
\begin{align*}
    F(\vw) \leq&\, \frac{1}{2}\br{\zeta^2 + \psi^\star + \psi^\star(1 + 3\gamma n^{-\epsilon}) - 2\psi^\star + \frac{2\gamma'}{\sqrt{n}}}\\
    \leq&\, \frac{1}{2}\br{\zeta^2 + 3\psi^\star \gamma n^{-\epsilon} + 2\gamma'n^{-\frac{1}{2}} }\\
    \lesssim&\, n^{-\min\bc{\frac{1}{2}, \epsilon}}.
\end{align*}
We use the fact that $\psi^\star = \Theta(1)$ and plug in the value of $\zeta^2$.

Finally, note that $F^\star \geq 0$, as the squared loss is non-negative for any $\vw\in \bR^{m(d+1)}$. Therefore, the bound on $F(\vw)$ is also the bound on the excess risk, $F(\vw) - F^\star$.

\end{proof}

\section{Technical Tools and Missing Proofs}

\label{sec:misc}

\subsection{Sub-Weibull Distributions}
\label{sec:sub_weibull}
A comprehensive treatment of Sub-Weibull random variables is provided in ~\citep{Vladimirova_2020,sub_weibull,sharper_sub_weibull,Zhang_2021}. We restate the main results that we use for Sub-Weibull random variables.

We first provide definition and properties of Sub-Weibull random variables, mostly adapted from ~\cite[Corollary~6.1]{Zhang_2021}.
\begin{definition}[$(\varrho, K)$-Sub-Weibull Random Variable]\label{def:sub_weibull} For a random variable $v$ the following statements are equivalent.
\begin{itemize}
    \item $v$ is a $(\varrho, K)$-Sub-Weibull random variable for $\varrho\geq \frac{1}{2}$ and $K>0$.
    \item The tails of $v$ satisfy $\P{\abs{v-\E{v}}\geq t}\leq \exp\br{-\br{\frac{t}{K_1}}^{\frac{1}{\varrho}}}$ for all $t> 0$ where  $K_1 = \Theta(K)$.
    \item The moments of $v$ satisfy $\norm{v}_k \coloneq (\E{\abs{v - \E{v}}^k})^{\frac{1}{k}} \leq K_2 k^{\varrho}$ for all $k\geq \min\{1, \varrho^{-1}\}$, where $K_2 = \Theta(K)$.
    \item The MGF of $\abs{v - \E{v}}^{\varrho}$ satisfies, $\E{\exp((\lambda \abs{v - \E{v}})^\varrho)}\leq \exp\br{(\lambda K_3)^\varrho}, \forall \abs{\lambda}\leq K_3^{-1}$ where $K_3 = \Theta(K)$.
    \item Orlicz Norm: $K = \inf\bc{K_4\in (0,\infty):\E{\exp\br{\frac{\abs{v - \E{v}}}{K_4}}^\varrho}\leq 2}$.
\end{itemize}
\end{definition}
The following Lem. provides guarantees on sum and products of Sub-Weibull random variables. This has been adapted from ~\citep{sharper_sub_weibull}. 
\begin{lemma}[Sum and Product of Sub-Weibull Random Variables]\label{lem:sub_weibull_sum_prod}
If $v_i$ is a $(\varrho_i, K_i)$-Sub-Weibull random variable $\forall i \in [r]$ then,
\begin{itemize}
    \item $\sum_{i=1}^r v_i$ is a $(\varrho, K')$-Sub-Weibull random variable if $\varrho=\varrho_i, \forall i\in [r]$ and $K' = \Theta(\max_{i\in [r]} K_i))$ if all $v_i$'s are dependent and $K'=\Theta(\sum_{i=1}^r K_i)$ if all $v_i$'s are independent.
    \item $\prod_{i=1}^r v_i$ is a $(\varrho, K'')$-Sub-Weibull random variable where $\varrho = \prod_{i=1}^r \varrho_i$ and $K'' = \prod_{i=1}^r K_i$.
\end{itemize}
\end{lemma}

The following Lem. provides concentration for sum of independent Sub-Weibull random variables. This has been adapted from ~\cite[Prop.~3]{sharper_sub_weibull}.
\begin{lemma}[Concentration of sum of Sub-Weibull Random Variables]\label{lem:sub_weibull_sum}
If $\{v_i\}_{i\in [r]}$ are $r$ iid $(\varrho, K)$-Sub-Weibull random variables for $\varrho\geq \frac{1}{2}$, then with probability $1-\delta$,
\begin{align*}
    \abs{\frac{1}{n}\sum_{i=1}^n v_i - \E{v_1}} \leq \cO\br{K\sqrt{\frac{\log(1/\delta)}{n}}}
\end{align*}
\end{lemma}
This Lem. provides bounds on quadratic forms for Sub-Weibull random variables and has been adapted from ~\cite{sub_weibull}.
\begin{lemma}[Hanson-Wright for Sub-Weibull distributions]\label{lem:hw_sub_weibull}
If $\vv\in \bR^{n}$ is a random vector with  $\vv_i$ being an independent $(\varrho, K_i)$-Sub-Weibull random variable. Then, for any deterministic symmetric matrix $\bar{\vU}\in \bR^{n\times n}$, with probability $1-\delta$,
\begin{align}
\abs{\vv^\top \bar{\vU}\vv - \E{\vv^\top\bar{\vU}\vv}} \leq \bar{\nu}_\varrho\; \max\{\norm{\vK \bar{\vU}\vK}_F \sqrt{\log(2/\delta)}, \norm{\vK\bar{\vU}\vK}_2 (\log(2/\delta))^{2\varrho}\}
\end{align}
for some absolute constant $\bar{\nu}_\varrho$ depending on $\varrho$ and $\vK\in \bR^{n\times n}$ is given by $\vK = \diag(K_1, K_2, \ldots, K_n)$.
\end{lemma}
\begin{proof}
    The proof simply replaces $\bar{\vU}$ by $\vK \bar{\vU}\vK$ which makes $\widetilde{\vv} = \vK^{-1} \vv$ a vector of iid $(\varrho, 1)$-Sub-Weibull random variables. Further, the term $\alpha$ in ~\cite[Thm.~2.1]{sub_weibull} is equal to $\frac{1}{\varrho}$ in our notation.
\end{proof}

\begin{lemma}[Bai-Yin for Sub-Weibull Random Matrices with Independent Columns]\label{lem:bai_yin_subweibull} For a matrix $\bar{\vU}\in \bR^{n\times m}$ with independent rows, each element $\bar{\vU}_{i,j}$ being a $0$-mean $(\varphi, \cO(1))$-Sub-Weibull random variable, satisfying Assump.~\ref{assump:width}, with probability $1-\delta$,
\begin{align*}
    \sigma_{\max}(\bar{\vU}) \lesssim \sqrt{m}+\sqrt{n}, \quad, \sigma_{\min}(\bar{\vU}) \gtrsim \sqrt{m} - \sqrt{n}
\end{align*}
\end{lemma}
\begin{proof}
    For $\sigma_{\max}$, define two $\frac{1}{2}$-covers, $\cC_{\frac{1}{2}}(\bS^{n-1})$ and $\cC_{\frac{1}{2}}(\bS^{m-1})$. Note that, 
    \begin{align*}
        \sigma_{\max}(\bar{\vU}) = \max_{\vv_1\in \bS^{n-1}, \vv_2\in \bS^{m-1}} \vv_1^\top \bar{\vU}\vv_2^\top \leq 4 \max_{\vv_1\in \cC_{\frac{1}{2}}(\bS^{n-1}), \vv_2\in\cC_{\frac{1}{2}}(\bS^{m-1})} \vv_1^\top \bar{\vU}\vv_2^\top. 
    \end{align*}
    For a fixed $\vv_1, \vv_2$, we have $\E{\vv_1^\top\bar{\vU}\vv_2} = 0$. Further,
    \begin{align*}
        \vv_1^\top \bar{\vU}\vv_2 = \sum_{i=1}^n (\vv_1)_i \ip{\bar{\vU}_i}{\vv_2}.
    \end{align*}
    Note that $\ip{\bar{\vU}_i}{\vv_2}$ is a sum of dependent Sub-Weibull variables, hence, it is $(\varphi, \cO(\vv_2^\top \bar{\Phi}\vv_2))$- Sub-Weibull. From Assump.~\ref{assump:width}, we have, $\vv_2^\top\bar{\Phi}\vv_2 \leq \lambda_{\max}(\bar{\Phi}) = \cO(1)$. Further, the term $\vv_1^\top \bar{\vU}\vv_2$ is a sum of $n$ iid $(\varphi, \cO(1))$-Sub-Weibull random variables. Hence, by Lem.~\ref{lem:sub_weibull_sum}, and taking a union bound over the size of the covers, from ~\citep[Chapter~5]{Wainwright_2019}, we have, with probability $1-\delta$, 
    \begin{align*}
        \vv_1^\top \bar{\vU}\vv_2 \leq \widetilde{C}_1 (\sqrt{m} + \sqrt{n}) + \widetilde{C}_2 (\log(1/\delta))^c
    \end{align*}
    for some constants $\widetilde{C}_1,\widetilde{C}_2>0$.
    By using the bounds from the cover, $\sigma_{\max}(\bar{\vU}) \lesssim \sqrt{m} + \sqrt{n}$.

For $\sigma_{\min}$, we use the small-ball method~\citep{koltchinskii_15}.
Note that, 
\begin{align*}
    \sigma_{\min}(\bar{\vU}) = \inf_{\vv\in \bS^{n-1}} \norm{\bar{\vU}^\top \vv}_2
\end{align*}
Note that, $\bar{\vU}^\top \vv = \sum_{i=1}^n (\vv)_i \bar{\vU}_i$. Note that each $\bar{\vU}_i$ is a Sub-Weibull random vector with $\E{\bar{\vU}_i} = 0$. Further, $\bar{\vU}_i$ are iid. By the Paley-Zygmund inequality, we first obtain a small ball condition on each coordinate of $\bar{\vU}^\top\vv$. Note that, $(\bar{\vU}^\top\vv)_j = \sum_{i=1}^n \bar{\vU}_{i,j}\vv_i$. Therefore, 
\begin{align*}
    \P{(\bar{\vU}^\top\vv)_j^2 \geq  \vartheta \bar{\Phi}_{j,j}} \geq (1-\vartheta)^2\frac{\bar{\Phi}_{j,j}^2}{\E{(\bar{\vU}^\top \vv)_j^4}}
\end{align*}
Since these are sum of independent Sub-Weibull entries and $\lambda_{\min}(\bar{\Phi}) =\Omega(1)$, all the terms above can be bounded by absolute constants. Theefore, we obtain, $\forall j\in [m]$, $\exists$ $\kappa,\eta >0$ such that,
\begin{align*}
    \P{(\bar{\vU}^\top\vv)_j^2 \geq  \kappa} \geq \eta.
\end{align*}

Now, applying this small-ball bound to ~\citep{koltchinskii_15}, we obtain, with high probabilty that,
\begin{align*}
    \sigma_{\min}(\bar{\vU}) \geq \kappa\eta \sqrt{m} - \widetilde{C}_3 \sqrt{n} 
\end{align*}
where $\widetilde{C}_3>0$ is some constant, and the $\sqrt{n}$ term is obtained from Rademacher complexity of halfspaces in $\bS^{n-1}$.
\end{proof}

\subsection{Proofs for App.~\ref{sec:prelims}}

\subsubsection{Proof of Prop.~\ref{prop:activation_valid}}
\label{sec:link_func}

\paragraph{Properties of $\sigma$}

First, for a piece-wise polynomial $\sigma$ defined in Assump.~\ref{assump:activation}, we compute the expectation, variance and tail behavior, for $\sigma(b)$ where $\vb\sim \cN(\mu, \zeta^2)$. We first compute this for the case of $m^\star = 1$

\textbf{Mean of $\sigma$}
We use the transformation $b' = \frac{b-\mu}{\varsigma}$
\begin{align*}
    \Ee{b\sim \cN(\mu, \varsigma^2)}{\sigma(b)} = c'\int_{-\frac{\mu}{\varsigma}}^\infty \abs{\mu + \varsigma b'}^c \phi(b') db' + c''\int_{-\infty}^{-\frac{\mu}{\varsigma}} \abs{\mu + \varsigma b'}^c \phi(b') db'
\end{align*}
Note that $\abs{b_1 + b_2}^c \leq \nu_c (\abs{b_1}^c + \abs{b_2}^c)$ for the constant $\widetilde{\nu}_c \leq 2^{c-1}$ for any $c\geq 1$.
\begin{align*}
    \Ee{b\sim \cN(\mu, \varsigma^2)}{\sigma(b)} \leq&\, \widetilde{\nu}_c\abs{\mu}^c\left(\abs{c'}\int_{-\frac{\mu}{\varsigma}}^\infty \phi(b') db' +\abs{c''} \int_{-\infty}^{-\frac{\mu}{\varsigma}}\phi(b') db'\right) \\
    &\quad+ \widetilde{\nu}_c\varsigma^c\left(\abs{c'}\int_{-\frac{\mu}{\varsigma}}^\infty \abs{b'}^c\phi(b') db' +\abs{c''} \int_{-\infty}^{-\frac{\mu}{\varsigma}}\abs{b'}^c \phi(b') db'\right) \\
    =&\,  \widetilde{\nu}_c\abs{\mu}^c\br{\abs{c'} + (\abs{c''} - \abs{c'})\Phi\br{-\frac{\mu}{\varsigma}}} \\
    &\quad+ \widetilde{\nu}_c\varsigma^c 2^{\frac{c}{2}-1} \pi^{-\frac{1}{2}}\br{(\abs{c'}+\abs{c''})\Xi_1\br{\frac{c+1}{2},\frac{2\mu^2}{\varsigma^2}} + 2\abs{c'}\Xi_2\br{\frac{c+1}{2},\frac{2\mu^2}{\varsigma^2}}} 
\end{align*}
Here, $\Xi_1(b_1, b_2) = \int_{b_2}^\infty t^{b_1 - 1} dt$ and $\Xi_2(b_1, b_2) = \int_{0}^{b_2} t^{b_1 - 1}dt$ represent the upper and the lower incomplete Gamma functions for $b_2>0, b_1>0$. We further use inequalities to bound these gamma functions. For any $b_1, b_2>0$, the following hold.
\begin{equation}\label{eq:gamma_bounds}
\begin{aligned}
    \Xi_1(b_1, b_2) &\leq\, \Xi_1(b_1) \leq \ceil{b_1}! = \cO(b_1^{b_1} e^{-b_1}),\quad \Xi_2(b_1, b_2) &\leq\, \frac{b_2^{b_1}}{b_1}
\end{aligned}
\end{equation}

Plugging in these upper bounds, and setting $c, c', c'' = \Theta(1)$, we obtain, 
\begin{align}\label{eq:exp_ub}
    \Ee{b\sim \cN(\mu, \varsigma^2)}{\sigma(b)} \leq \cO\br{\abs{\mu}^{c} + \varsigma^c}
\end{align}

The lower bound for the expectation is also of the order of $\Omega(\abs{\mu}^c + \abs{\sigma}^c)$.
For the case of $\mu=0$, the lower bound is exactly of the order of $\Omega(\varsigma^c)$
\begin{align*}
    \abs{\Ee{b\sim \cN(0, \varsigma^2)}{\sigma(b)}} = \varsigma^{c}\frac{c'+c''}{2}\Ee{b'\sim \cN(0,1)}{\abs{b'}^c} = \varsigma^{c}\frac{(\abs{c'+c''})2^{\frac{c}{2}}}{2\sqrt{\pi}}\Xi_1\br{\frac{c+1}{2}} = \Omega(\varsigma^{c})
\end{align*}
Here, $\Xi_1(b_1) = \int_0^\infty t^{b_1-1} e^{-t} dt$ is the complete Gamma function.

The only case when $\Ee{b\sim \cN(0,\sigma)}{\sigma(b)} = 0$ is when $\sigma$ is an odd-degree monomial, i.e., $c'+c'' = 0$.

For a general mean $\mu$, using Jensen's inequality for the convex function $\abs{b}^c$ with $c\geq 1$,
\begin{align*}
    \abs{\Ee{b\sim \cN(\mu, \varsigma^2)}{\sigma(b)}} \geq \min\{\abs{c'}, \abs{c''}\} \abs{\Ee{b\sim \cN(\mu, \varsigma^2)}{\abs{b}^c}} \geq \min\{\abs{c'}, \abs{c''}\} \abs{\Ee{b\sim \cN(\mu, \varsigma^2)}{b}}^c = \Omega(\abs{\mu}^c)
\end{align*}

\textbf{Variance of $\sigma$}

For a $0$-mean differentiable function $\widetilde{\sigma}:\bR\to \bR$, using Cauchy-Schwartz, we have,
\begin{align*}
 (\Ee{b\sim \cN(\mu, \varsigma^2)}{\widetilde{\sigma}(b)(b-\mu)})^2&\leq \Ee{b\sim \cN(\mu, \varsigma^2)}{(b-\mu)^2}\cdot \Ee{b\sim \cN(\mu, \varsigma^2)}{\widetilde{\sigma}^2(b)} = \varsigma^2 \Varr{b\sim \cN(\mu, \varsigma^2)}{\widetilde{\sigma}(b)}\\
 \varsigma^4 (\Ee{b\sim \cN(\mu, \varsigma^2)}{\widetilde{\sigma}'(b)})^2 &\leq \varsigma^2 \Varr{b\sim \cN(\mu, \varsigma^2)}{\widetilde{\sigma}(b)}\\
 \varsigma^4 (\Ee{b\sim \cN(\mu, \varsigma^2)}{\widetilde{\sigma}'(b)})^2 &\leq\Varr{b\sim \cN(\mu, \varsigma^2)}{\widetilde{\sigma}(b)}
\end{align*}
We use Stein's identity for the second inequality.

For a differentiable function $\sigma$, we can set $\widetilde{\sigma}$ to be $\sigma(b) - \Ee{b'\sim \cN(\mu, \varsigma}{\sigma(b)}$ and $\sigma^\star(b) - \Ee{b'\sim \cN(\mu, \varsigma}{\sigma^\star(b)}$. Note that adding or subtracting a constant to $\sigma$ doesn't change its first derivative.

Using the above inequality, we can lower bound the variance of $\sigma$.
\begin{align*}
    \Varr{b\sim \cN(\mu, \varsigma^2)}{\sigma(b)} \geq \varsigma^2\abs{\Ee{b\sim \cN(\mu, \varsigma^2)}{\sigma'(b)}}^2 = \Omega\br{\varsigma^2\abs{\mu}^{2c-2} +  \varsigma^{2c}\}} 
\end{align*}
Since $\sigma'(b)$ is also a piece-wise polynomial, but of degree $c-1$, we can use the previous bounds on the mean of piece-wise polynomial.

\textbf{Tails of $\sigma$}
We will now compute a concentration inequality for the tail of $\sigma(b)$ for $b\sim \cN(\mu, \varsigma^2)$. 

Since $\sigma$ is a piece-wise polynomial function of degree $c$, $\sigma(b)$ for $b\sim \cN(\mu, \varsigma^2)$ is a $(\frac{c}{2}, \cO(\abs{\mu}^c + \varsigma^c))$-Sub-Weibull~\citep{Vladimirova_2020}. Note that the $k^{th}$ moment of $(\varrho, K)$-Sub-Weibull distribution is upper bounded by $  K( k)^\varrho$. As piece-wise polynomial activations are $c$-degree monomials of gaussian $\cN(\mu, \varsigma^2)$, their $k^{th}$ moment is the $(ck)^{th}$ moment of a gaussian which grows as $k^{\frac{c}{2}}$, with the constant $K = \cO((\abs{\mu} + \varsigma)^c) = \cO(\abs{\mu}^c + \abs{\varsigma}^c)$.

\paragraph{Mean, variance and tails of $\hat{\sigma}$:} From Assump.~\ref{assump:link}, the mean, variance and tails of the unknown link function $\sigma^\star$ match those of the activation function, if $m^\star =1$. This directly corresponds to the single-index link function in Section~\ref{sec:setup}. To extend this notion to Sum of Single-Index link functions in Section~\ref{sec:setup}, we need to find compute the mean, variance and tails of for the random variable $q=\sum_{j=1}^{m^\star} a_j^\star \sigma(b_j)$ for a random vector $\vb \sim \cN(\mu, \varsigma)$ with $\mu\in \bR^{m^\star}$ and $\varsigma\in \bR^{m^\star \times m^\star}$ being a Positive Semi-Definite matrx.

\textbf{Mean.}
\begin{align*}
    \E{q} = \sum_{j=1}^{m^\star} a_j^\star \E{\sigma(b_j)}
\end{align*}
As each $b_j \sim \cN(\mu_j, \varsigma_{j,j})$ $\forall j\in [m^\star]$, using the bounds on mean of $\sigma$ from previous section, we have,
\begin{align*}
    \abs{\E{\sigma(b_j)}} = \Theta(\abs{\mu}^c, (\varsigma_{j,j})^{\frac{c}{2}})
\end{align*}
Therefore,
\begin{align*}
    \abs{\E{q}} = \Theta(\sum_{j=1}^{m^\star} \abs{\mu_j}^c + \sum_{j=1}^{m^\star} (\varsigma_{j,j})^{\frac{c}{2}})
\end{align*}
As $m^\star = \cO(1)$, and all $a_j^\star \neq 0$ and $\abs{a_j^\star} = \cO(1)$, we have the above bound.
Additionally, since $1\leq m^\star= \cO(1)$, by the equivalence of $\ell_p$-norms, we have, 
\begin{align*}
    \sum_{j=1}^{m^\star} \abs{\mu_j}^c =\Theta(\norm{\mu}_2^c), \quad \sum_{j=1}^{m^\star} (\varsigma_{j,j})^{\frac{c}{2}} = \norm{\varsigma}_F^{\frac{c}{2}}
\end{align*}

\textbf{Variance.}
For variance, we use a similar bound based on Stein's method. Define $\vq'\in \bR^{m^\star}$ as the vector with its $j^{th}$ coordinate being $\sigma'(b_j)$, $\forall j\in [m^\star]$.
\begin{align*}
\norm{\E{(q-\E{q})(\vb - \mu)}}_2^2 \leq& \Var{\vb}\Var{q}\\
\norm{\varsigma \vq'}_2^2 \leq & \trace(\varsigma) \Var{q} \leq \sqrt{m^\star}\norm{\varsigma}_F\Var{q} 
\end{align*}
As $\sigma'$ is also a piece-wise polynomial, and from previous bounds on the mean of piece-wise polynomials, we have,
\begin{align*}
    \Var{q} \geq \norm{\varsigma}_F \Omega(\norm{\mu}_2^{c-1} + \norm{\varsigma}_F^{\frac{c-1}{2}})
\end{align*}

\textbf{Tails.}
Here, each $\sigma(b_j)$ is a $(\frac{c}{2}, \cO(\abs{\mu_j}^c + (\varsigma_{j,j})^{\frac{c}{2}})$-Sub-Weibull random variable. From Def.~\ref{lem:sub_weibull_sum_prod}, $q$ is a $(\frac{c}{2}, \cO(\max_{j\in [m^\star]}\abs{\mu_j}^c + \max_{j\in [m^\star]}(\varsigma_{j,j})^{\frac{c}{2}}))$-Sub-Weibull random variable.
Since $m^\star = \cO(1)$, we have,
\begin{align*}
 \max_{j\in [m^\star]} \abs{\mu_j}^c = \Theta(\norm{\mu}_2^c), \quad    \max_{j\in [m^\star]}(\varsigma_{j,j})^{\frac{c}{2}} = \Theta(\norm{\varsigma}_F^{\frac{c}{2}})
\end{align*}
Therefore, $q$ is a $(\frac{c}{2}, \cO(\frac{c}{2}, \cO(\norm{\mu}_2^c + \norm{\varsigma}_F^{\frac{c}{2}})$ random variable.

Through the above calculations, we have shown that the sum of activations satisfies Assump.~\ref{assump:link}.

\subsubsection{Proof of Lem.~\ref{lem:eigen_U}}
\label{sec:eigen_U_proof}
Note that $\vU$ is a tall random matrix with iid rows and dependent columns. Let $\bar{\phi} \coloneq \E{\vU_{i,j}}$. If we set $\bar{\vU}_{i,j} \coloneq \vU_{i,j}-\bar{\phi}$, then $\E{\bar{\vU}_{i,j}}=0, \E{\bar{\vU}^\top\bar{\vU}^\top}=\bar{\Phi}$. This decomposition can be written as,
\begin{align*}
    \vU = \bar{\vU} + \bar{\phi}\vone_n^\top \vone_m
\end{align*}

From Prop.~\ref{prop:activation_valid}, $\abs{\bar{\phi}} = \Theta(1)$ and each $\vU_{i,j}$ is a $(\frac{c}{2}, \cO(1))$-Sub-Weibull random variable. Therefore, each $\bar{\vU}_{i,j}$ is a $0$-mean, $(\frac{c}{2}, \cO(1))$-Sub-Weibull random variable.

Using Lem.~\ref{lem:bai_yin_subweibull}, with probability $1-\delta$, $\sigma_i(\bar{\vU}) \in [\sqrt{m} -  \sqrt{n}, \sqrt{m}+\sqrt{n}]$. Since $m\geq n$, $\sigma_{\min}(\bar{\vU}) >0$. For $\vU$, note that it has a non-zero mean component $\bar{\phi}\vone_n^\top \vone_m^\top$. Since $\abs{\bar{\phi}} = \Theta(1)$, the largest singular vector of $\vU$ is along the all-ones vector. The singular value corresponding to this vector obtained from the mean component is $\sqrt{mn}$. 
\begin{align*}
    \sigma_{\max}(\vU) \leq \cO(\sigma_{\max}(\bar{\vU}) + \sigma_{\max}(\vone_n^\top\vone_m)) = \cO( \sqrt{m} + \sqrt{n} + \sqrt{m n})
\end{align*}
Therefore, $\sigma_{\max} = \cO(\sqrt{mn})$. Note that all the other singular values of $\vU$ are obtained from singular values of $\bar{\vU}$ as the mean corresponds to a rank-$1$ matrix. Therefore, $\sigma_{i}(\vU) \in [\sqrt{m} -  \sqrt{n}, \sqrt{m}+\sqrt{n}]$, $\forall i\in [n]$. Therefore, for $m \gtrsim n, $
\begin{align*}
    \abs{\sigma_{\max}(\vU)} \asymp \frac{1}{\sqrt{mn}}, \quad \sigma_{i}(\vU) \asymp \frac{1}{\sqrt{m}}, \forall i\in \{2,3,\ldots, n\}.
\end{align*}
This completes the proof.

\subsection{Proof of Lem.~\ref{lem:bounds_b_theta}}
\label{sec:b_theta_bounds}
First, consider the expression for $B(\theta)$ for any $\theta\in \bS^{d-1}$,
\begin{align*}
&B(\theta) =2B_1^\frac{1}{c+1}(\theta)B_2^\frac{c}{c+1}(\theta)\\
&B_1(\theta) = \frac{1}{n}\sum_{i=1}^n \sigma^2(\ip{\theta}{\vx_i}), \quad B_2(\theta) = \frac{1}{n}\sum_{i=1}^n (\sigma')^2(\ip{\theta}{\vx_i})\norm{\vx_i}_2^2
\end{align*}

\paragraph{Bound on $B_1$}
Note that $\sigma^2(\ip{\theta}{\vx_i})$ is $(c, \cO(1))$-Sub-Weibull(Def.~\ref{def:sub_weibull})  as $\ip{\theta}{\vx_i}\sim \cN(0,1)$. Further, $\E{\sigma^2(\ip{\theta}{\vx_i}} = \phi^\star = \Theta(1)$ (App.~\ref{sec:link_func}). Therefore, the mean of $n$ independent Sub-Weibull random variables can be bounded by using concentration. For a fixed $\theta\in \bS^{d-1}$, using Lem.~\ref{lem:sub_weibull_sum}, with probability $1-\delta$,
\begin{align*}
    \abs{B_1(\theta) - \phi^\star} \lesssim \sqrt{\frac{\log(2/\delta)}{n}}
\end{align*}

\paragraph{Bound on $B_2$}
\begin{align*}
    B_2(\theta) &= \frac{1}{n} (\sigma')^2(\ip{\theta}{\vx_i})\norm{\vx_i}_2^2 = \frac{1}{n} (\sigma')^2(\ip{\theta}{\vx_i})\br{\ip{\theta}{\vx_i}^2 + \norm{(\bI_d - \theta\theta^\top)\vx_i}_2^2}\\
    &=\underset{B_{2,1}(\theta)}{\underbrace{\frac{1}{n} (\sigma')^2(\ip{\theta}{\vx_i})\ip{\theta}{\vx_i}^2}} + \underset{B_{2,2}(\theta)}{\underbrace{\frac{1}{n} (\sigma')^2(\ip{\theta}{\vx_i})\norm{(\bI_d - \theta\theta^\top)\vx_i}_2^2}}
\end{align*}
We bound the two terms separately. Note that we have decomposed $\vx_i$ in the direction along $\theta$ and the direction perpendicular to $\theta$.

For the first term, since $\sigma'$ is a $(c-1)$-degree piecewise-polynomial function of $\ip{\theta}{\vx_i}$, \,   
$(\sigma')^2(\ip{\theta}{\vx_i})\ip{\theta}{\vx_i}^2$ is a $2c$-degree piece-wise polynomial function of $\ip{\theta}{\vx_i}$. This random variable has mean $\mu_{B_{2,1}} = \Theta(1)$(App.~\ref{sec:link_func}), and is $(c,\cO(1))$-Sub-Weibull (Def.~\ref{def:sub_weibull}). For a fixed $\theta\in \bS^{d-1}$, using Lem.~\ref{lem:sub_weibull_sum}, with probability $1-\delta$, 

\begin{align*}
        \abs{B_{2,1}(\theta) - \mu_{B_{2,1}}} \lesssim \sqrt{\frac{\log(2/\delta)}{n}}
\end{align*}

For the second term, note that $(\bI_d - \theta\theta^T)\vx \sim \cN(0, \bI_d - \theta\theta^\top)$ and this random variable is independent of $\ip{\theta}{\vx}$. Note that the norm of this variable is Sub-Exponential with mean $d-1$, and Sub-Exponential parameter proportional to $d$. Note that this Sub-Exponential variable is $(1, d)$-Sub-Weibull(Def.~\ref{def:sub_weibull}). Since $(\sigma')^2$ itself is $(c-1, \cO(1))$-Sub-Weibull (Def.~\ref{def:sub_weibull}), and independent of this Sub-Exponential variable, the product of these two variables is $(c, \cO(d))$-Sub-Weibull (Lem.~\ref{lem:sub_weibull_sum_prod}). Let $\Ee{b\sim \cN(0,1)}{(\sigma')^2(b)}\coloneq \mu_{B_{2,2}} = \Theta(1)$. Using Lem.~\ref{lem:sub_weibull_sum}, with probability $1-\delta$, for a fixed $\theta \in \bS^{d-1}$, we have,
\begin{align*}
        \abs{B_{2,2}(\theta) - (d-1)\mu_{B_{2,1}}} \lesssim (d-1)\sqrt{\frac{\log(2/\delta)}{n}}
\end{align*}

\paragraph{Lower Bound on $\min_{j\in [m]}B(\theta_j)$}
We use the bounds for $B_1$ and $B_2$ for each $\theta = \theta_j$ $\forall j\in [m]$, and take a union bound over the probability of error. We set the probability of error for each $\theta_j$ to be $\frac{\delta}{m}$ instead of $\delta$. Therefore, with probability $1-3\delta$,

\begin{align*}
\min_{j\in [m]}        B(\theta_j) \gtrsim \br{\phi^\star - \frac{\sqrt{\log(m)}}{\sqrt{n}}}^\frac{1}{c+1}\br{\mu_{B_{2,1}} + (d-1)\mu_{B_{2,2}} - \frac{d\sqrt{\log(m)}}{\sqrt{n}}}^\frac{c}{c+1} \gtrsim d^{\frac{c}{c+1}}
\end{align*}

\paragraph{Upper Bound on $\max_{j\in [m]}B(\theta_j)$}
Using the same union bound as before, with probability $1-3\delta$,

\begin{align*}
\max_{j\in [m]}        B(\theta_j) \lesssim \br{\phi^\star + \frac{\sqrt{\log(m)}}{\sqrt{n}}}^\frac{1}{c+1}\br{\mu_{B_{2,1}} + (d-1)\mu_{B_{2,2}} +\frac{d\sqrt{\log(m)}}{\sqrt{n}}}^\frac{c}{c+1} \lesssim d^{\frac{c}{c+1}}
\end{align*}

\subsection{Proofs for App.~\ref{sec:proof_flattest_bad}}
\subsubsection{Proof of Prop.~\ref{prop:bad_min_pop_loss}}
\label{sec:rem_bad_min_pop_loss_proof}
From App.~\ref{sec:prelims}, the population loss for any weight vector $\vw$ is given by,
\begin{align*}
    2F(\vw) = \zeta^2 +\psi^\star + \va^\top \Phi \va - 2\ip{\widetilde{\Psi}}{\va}.
\end{align*}
Note that this is a quadratic in $\va$. If $\Phi \succ 0$, this quadratic is minimized in $\va$ at $\bar{\va} = \Phi^{-1}\widetilde{\Psi}$. Therefore, a lower bound on the population loss independent of $\va$ is given by,

\begin{align*}
    2F(\vw) \geq \zeta^2 +\psi^\star -  \widetilde{\Psi}\Phi^{-1}\widetilde{\Psi} \geq \zeta^2 \geq \zeta^2 +\psi^\star - \lambda_{\min}^{-1}(\Phi)\norm{\widetilde{\Psi}}_2^2 \geq \zeta^2 + \psi^\star - m\lambda_{\min}^{-1}(\Phi) \max_{j\in [m]}\widetilde{\psi}^2(\Theta^\star \theta_j) 
\end{align*}
 Each coordinate of $\widetilde{\Psi}$ is $\widetilde{\psi}(\Theta^\star \theta_j)$. If $\widetilde{\psi}$ is coordinate-wise an increasing function, and each coordinate of the vector $\Theta^\star\theta_j$ lies in $[-\rho^\star, \rho^\star]$, $\widetilde{\psi}(\Theta^\star\theta_j) \leq \widetilde{\psi}(\rho^\star \vone_{m^\star}), \forall j\in [m]$.
 Plugging in the sufficient condition from Prop.~\ref{prop:bad_min_pop_loss} completes the proof.
\subsubsection{Proof of Example~\ref{ex:bad_min}}
\label{sec:ex_bad_min_proof}
If we select $\theta_j\perp \theta_{j'}^\star, \forall j\in [m], j'\in [m^\star]$, then $\ip{\theta_j}{\vx}$ is independent of $\ip{\theta_{j'}}{\vx}, \forall j\in [m], j'\in [m^\star]$. This implies that $h(\vw, \vx)$ is independent of $\sigma^\star(\Theta^\star\vx)$.
Therefore, the population loss for such an interpolator is given by the following.
\begin{align*}
    2F(\vw) &= \zeta^2 + \E{(y - h(\vw, \vx))^2} = \zeta^2 + \psi^\star + \E{h^2(\vw,\vx)} - 2\E{h(\vw, \vx)}\E{\sigma^\star(\Theta^\star\vx)} \\
    &= \zeta^2 + \psi^\star + \E{h^2(\vw, \vx)} - (\E{h(\vw, \vx)})^2 - (\E{\sigma^\star(\Theta^\star\vx)})^2 + (\E{h(\vw, \vx)} - \E{\sigma^\star(\Theta^\star\vx)})^2 
\end{align*}
The second equation is obtained by completing the squares while the first equation is obtained by independence of $h(\vw, \vx)$ and $\sigma^\star(\Theta^\star\vx)$.
Note that $\Var{h(\vw, \vx)} = \E{h^2(\vw, \vx)} - (\E{h(\vw, \vx})^2 \geq 0$. Additionally, the term inside the square is non-negative. This provides the required lower bound on population risk.
\begin{align*}
    2F(\vw) = \zeta^2 + \psi^\star - (\E{\sigma^\star(\Theta^\star\vx)})^2 =  \zeta^2 + \br{1-\frac{(\E{\sigma^\star(\Theta^\star(\vx))})^2}{\psi^\star}}\psi^\star
\end{align*}
If we set $\kappa \coloneq \frac{1}{2}\br{1-\frac{(\E{\sigma^\star(\Theta^\star\vx)})^2}{\psi^\star}}$, we get the required result. Note that $\Theta^\star \vx \sim \cN(0, \bI_{m^\star}$.

\subsubsection{Proof of Lem.~\ref{lem:interpolation}}
\label{sec:lem_ws_bad_proof}
 If $\vU\in \bR^{n\times m}$ is the matrix of activations, such that the $(i,j)^{th}$ element of it represented by $\vU_{i,j} = \sigma(\ip{\theta_j}{\vx_i})$. Then, for interpolation, $h(\vw, \vx_i) = y_i, \forall i\in [n]$ can be represented by the linear system,
    \begin{align}
        \vU \va = \vy \label{eq:interp}
    \end{align}
    Since $m\geq n$, for fixed $\theta_j\in \bR^{d}, j\in [m]$, if $\vU$ is full row-rank, then we can always find an $\va\in \bR^{m}$ that satisfies the above equation.

    Note that the only case when $\vU$ is not invertible is if its rows are linearly dependent. In this case, there exist real coefficients $\{q_i\}_{i\in [n]}$ such that not all of them are simultaneously $0$ and $\sum_{i=1}^n q_i \vU_{i,j} = 0, \forall j\in [m]$. Without loss of generality, assume that $q_1 \neq 0$. Then, $\vU_{1, j} = \frac{-\sum_{i=2}^{n}q_i U_{i,j}}{q_1}, \forall j\in [m]$.

    \paragraph{Case I: $\sigma(b)=0$ iff $b=0$.}
    For piece-wise polynomial activation, we need to compute the probability of the set $H(b) = \{\widetilde{b} \in \bR:\sigma(\widetilde{b}) = b\}$. If $c'\neq0, c''\neq0$, then  $\abs{H(b)}$ is either $0,1$ or $2$ values. It is $0$ if $b$ is not in the range of $\sigma$, it is $1$ if $b=0$ or $c'c''<0$ as then $\sigma$ is invertible, and it is $2$ if $c'c''>0$ and $b\neq0$ as then $\sigma$ is invertible on only half of the real line. However, $\ip{\theta_j}{\vx_1}\sim\cN(0,1)$ $\forall j\in [m]$ is a continuous probability distribution. Therefore, the probability of set $H(b)$ is $0$ as the probability of $b$ being in a finite set under a continuous distribution is always $0$.

    \paragraph{Case II: $\sigma(b) = 0$ for $b\leq 0$.}
    If $b=0$, then the set $H(b)$ is either $\{\ip{\theta_j}{\vx}\geq 0\}$ or $\{\ip{\theta_j}{\vx}\leq 0\}$, when the piece-wise polynomial activation has one of its pieces $0$. Therefore, the only case for $\vU$ being non-invertible with finite probability for $q_1=0$ is when $\vU_{1,j} = 0,\forall j\in [m]$. 
    This corresponds to $\sign(\ip{\theta_j}{\vx_1}), \forall j\in [m]$.
   
     We will calculate the probability of this event for $\zeta=0$ now.  Without loss of generality,  we assume $\sigma(b) = 0$ for $b\leq 0$. The same proof works for $\sigma(b) = 0$ for $b\geq 0$ by replacing all random variables by their negations.

    Define $\widetilde{Q}_{j} = \ip{\theta_j}{\vx_1}, \forall j\in [m]$. Then, $\widetilde{Q}_j \sim \cN(0, 1)$ and $\E{\widetilde{Q}_j \widetilde{Q}_{j'}} = \ip{\theta_j}{\theta_{j'}}$. Consider the set of random variables defined by $\bar{Q}_j = \rho Q_0 + \sqrt{1-\rho^2} Q_j, \forall j\in [m]$, where $Q_j \overset{iid}{\sim}\cN(0,1), \forall j \in [m]\cup\{0\}$. Then, $$\E{\bar{Q}_j} = 0, \E{\bar{Q}_j^2} = 1 \text{~~~and~~~} \E{\bar{Q}_j \bar{Q}_{j'}} = \rho^2.$$ By Slepian's Inequality~\cite{slepian}, we have,
    \begin{align*}
&\Pp{\vx_1\sim \cN(0, \bI_d)}{\ip{\theta_j}{\vx_1} \leq 0 , \forall j\in [m]}=\P{\widetilde{Q}_j \leq 0, \forall j\in [m]} \leq \P{\bar{Q}_j \leq 0, \forall j\in [m]}\\
&= \P{Q_j \leq \frac{-\rho Q_0}{\sqrt{1-\rho^2}}, \forall j\in [m]}= \int_{-\infty} \P{Q_j \leq \frac{-\rho t}{\sqrt{1-\rho^2}}, \forall j\in [m] \cond Q_0 = t} g(t) dt\\
&=  \int_{-\infty} \prod_{j=1}^m\P{Q_j \leq \frac{-\rho t}{\sqrt{1-\rho^2}} \cond Q_0=t} g(t) dt\\
&=\int_{-\infty}^{\infty}g(t) G^m(\frac{-\rho t}{\sqrt{1-\rho^2}}) dt   
    \end{align*}
    Here, $g(t) = \frac{1}{\sqrt{2\pi}}e^{-\frac{t^2}{2}}$ and $G(t) = \Pp{b\sim \cN(0,1)}{b \leq t} = \int_{-\infty}^t g(b) db$ are the gaussian pdf and cdf respectively.
    We break the integral into $3$ parts $t\geq 0$, $t\in (-\widetilde{b}, 0)$ and $t\in (-\infty, -\widetilde{b})$, for some $\widetilde{b}>0$.
    For the first part.
    \begin{align*}
        \int_{0}^\infty g(t)G^m(\frac{-\rho t}{\sqrt{1-\rho^2}}) dt \leq \int_{0}^\infty g(t) G^m(0) dt \leq 2^{-(m+1)}
    \end{align*}
    For the second part, we use the fact that $G(b) \leq 1 - \frac{b}{(b^2+1)\sqrt{2\pi}}e^{-\frac{b^2}{2}}\leq 1-\frac{1}{2\sqrt{2\pi}}e^{-\frac{b^2}{2}}$, for any $b\geq 0$, using the Mill's ratio.
    \begin{align*}
        &\int_{-\widetilde{b}}^0 g(t) G^m(\frac{-\rho t}{\sqrt{1-\rho^2}}) dt \leq  G^m(\frac{\rho \widetilde{b}}{\sqrt{1-\rho^2}})\int_{-\widetilde{b}}^0 g(t)dt \leq  G^m(\frac{\rho \widetilde{b}}{\sqrt{1-\rho^2}}) \\
        &\leq \left(1- \frac{1}{\sqrt{2\pi}}e^{-\frac{\rho^2 \widetilde{b}^2}{2(1-\rho^2)}}\right)^m \leq \exp\left(-m \frac{1}{\sqrt{2\pi}}e^{-\frac{\rho^2 \widetilde{b}^2}{2(1-\rho^2)}}\right)
    \end{align*}
    For the third part, using gaussian tail inequality, 
    \begin{align*}
        \int_{-\infty}^{-\widetilde{b}}g(t) \Phi^m(\frac{-\rho t}{\sqrt{1-\rho^2}})dt \leq \int_{-\infty}^{-\widetilde{b}} g(t) dt \leq \exp(-\frac{\widetilde{b}^2}{2})
    \end{align*}

    We would want to find $\bar{b}>0$ that minimizes the sum of second and third part, as the first part is already small. Set $\rho_0 = \frac{\rho}{\sqrt{1-\rho^2}}$. Since this function is not necessarily convex in $\bar{b}$, we find $\bar{b}$ such that both the second and third parts are equal.
    \begin{align*}
        &\exp\br{-\frac{m}{\sqrt{2\pi}}\exp\br{-\frac{\rho_0^2\bar{b}^2}{2}}} = \exp\br{-\frac{\bar{b}^2}{2}}\\
        &\exp\br{-\frac{\rho_0^2\bar{b}^2}{2}} = \frac{\bar{b}\sqrt{\pi}}{m\sqrt{2}}\\
        &\bar{b} = \frac{1}{\rho_0}\sqrt{2(\log(m) +\log(\sqrt{2}\rho_0) -\log(\sqrt{\pi}) - \log(\bar{b}))}
    \end{align*}
    Since the leading order term is $\frac{1}{\rho_0}\sqrt{2\log(m)}$, we can set  
    \begin{align*}
     \bar{b} = \frac{1}{\rho_0}\sqrt{2(\log(m) +2\log(\rho_0) -\log(\sqrt{\pi}) - \log(\sqrt{\log(m)}))}   
    \end{align*}
    Then, the bound on the high probability term is, $\widetilde{\cO}\br{\left(m\rho_0\right)^{-\frac{1}{\rho_0^2}}}$.
    
Note that the probability that none of the rows of $\vU$ are all $0$'s is at least $1 - \check{\nu}m\left(m\rho_0\right)^{-\frac{1}{\rho_0^2}}$, for some constant $\check{\nu}>0$,  using a union bound over $m$ rows.

\subsubsection{Proof of Lem.~\ref{lem:flattest_bad_noise}}
\label{sec:flattest_bad_noise_proof}
Note that, $\vy = \vy^\star + \vn$. By triangle inequality for $\ell_2$ norm, for any two vectors $\vb_1, \vb_2 \in \bR^{m}$, $\norm{\vb_1}_2^2 \geq \frac{1}{2}\norm{\vb_1 - \vb}_2^2 - \norm{\vb_2}_2^2$. 

Therefore,
\begin{align*}
\norm{\va_{\min, \ell_2}}_2^2 \geq \frac{1}{2}\norm{\widetilde{\vU}^{\frac{1}{2}}\vy^\star}_2^2 - \norm{\widetilde{\vU}^{\frac{1}{2}}\vn}_2^2 
\end{align*}
We can bound the term of noise separately.
Note that,
\begin{align*}
    \norm{\widetilde{\vU}^{\frac{1}{2}}\vn}_2^2 \leq \lambda_{\max}(\widetilde{\vU}) \norm{\vn}_2^2 = \sigma_{\min}^{-2}(\vU) \norm{\vn}_2^2 
\end{align*}
From Lem.~\ref{lem:eigen_U}, with probability $1-\delta$, $\sigma_{\min}(\vU) \gtrsim \frac{1}{\sqrt{m}}$. Further, note that $\zeta^{-2}\norm{\vn}_2^2$ is a $\chi_n^2$-random variable. By $\chi^2$ concentration from ~\cite[Lem.~1]{laurent_massart_06}, with probability $1-\delta$,
\begin{align*}
    \norm{\vn}_2^2 \leq \zeta^2 (n + 2\sqrt{n\log(1/\delta)} + 2\log(1/\delta)) \lesssim \zeta^2 n
\end{align*}

After eliminating the contribution of noise from $\norm{\va_{\min, \ell_2}}_2^2$, with probability $1-2\delta$, we obtain,
\begin{align*}
\norm{\va_{\min, \ell_2}}_2^2 \gtrsim \frac{1}{2}\norm{\widetilde{\vU}^{\frac{1}{2}}\vy^\star}_2^2 - \zeta^2 \frac{n}{m}    
\end{align*}
This completes the proof.
\subsubsection{Proof of Lem.~\ref{lem:cond_distr}}
\label{sec:cond_distr_proof}
Note that the matrix $\vU$ is obtained by applying a deterministic function $\sigma(\cdot)$ on the matrix $\hat{\vU}$ element-wise. Since $\widetilde{\vU}$ is obtained from $\vU$, it is also a deterministic function of the matrix $\hat{\vU}\in \bR^{n\times m}$. Let $\hat{\vU}_i=\Theta\vx_i\in \bR^{n}$ denote the $i^{th}$ row of $\hat{\vU}$.  To compute the distribution of $\vy_i^\star \cond \vU_i$, it is enough for us to compute the distribution of $\vy^\star \cond \hat{\vU}_i$. Since $\vy_i^\star = \sigma^\star(\Theta^\star\vx_i)$, we first need to specify the distribution of $\Theta^\star\vx_i\cond \hat{\vU}_i$.

\paragraph{Distribution of $\Theta^\star\vx_1\cond \hat{\vU}_1$}
Note that,
\begin{align*}
    \begin{bmatrix}
        &\Theta^\star\vx_1\\
        &\hat{\vU}_1
    \end{bmatrix} \sim \cN\left(\vzero_{m+m^\star}, \begin{bmatrix}
        \bI_{m^\star} & \Theta^\star \Theta^\top\\
        \Theta(\Theta^\star)^\top & \Theta\Theta^\top
    \end{bmatrix}\right).
\end{align*}

Therefore, $\Theta^\star\vx_1\cond \hat{\vU}_1$ is also a gaussian distribution with mean and variance determined by the Schur's complement of the joint distribution. Note that here $\Theta\Theta^\top$ is not invertible as its rank is atmost $d$, while it is a $m\times m$ matrix. Therefore, we compute its pseudo-inverse in the Schur's complement.
\begin{align*}
    \Theta^\star\vx_1\cond \hat{\vU}_1\sim \cN(\Theta^\star \Theta^\top (\Theta\Theta^\top)^{\dagger}\hat{\vU}_1, \bI_{m^\star} - \Theta^\star \Theta^\top (\Theta\Theta^\top)^{\dagger}\Theta (\Theta^\star)^\top)
\end{align*}
Here, $A^{\dagger}$ is the pseudo-inverse of $A$ for a square matrix with real entries $A$. In this case, $\vy_i^\star\cond\vU = \sigma^\star(\Theta^\star\vx_i)\cond \vU$ where $\Theta^\star\vx_i \cond \vU \sim\cN(\widetilde{\mu}_i, \varsigma)$ with $\varsigma = \bI_{m^\star} - \Theta^\star \Theta^\top (\Theta\Theta^\top)^{\dagger}\Theta (\Theta^\star)^\top\in \bR^{m^\star \times m^\star}$ and $\widetilde{\mu}_i = \Theta^\star \Theta^\top (\Theta\Theta^\top)^{\dagger}\hat{\vU}_i, \forall i\in [n]$. Therefore, $\vy_i^\star \cond \vU$ is a $(\frac{c}{2}, \cO(\norm{\widetilde{\mu}_i}_2^c + \norm{\varsigma}_F^{\frac{c}{2}}))$-Sub-Weibull random variable, using Assump.~\ref{assump:link}. The other conditions on mean and variance are obtained by plugging in the corresponding expressions from Assump.~\ref{assump:link}. This completes the proof.

\subsubsection{Proof of Lem.~\ref{lem:simplification}}
\label{sec:simplification_proof}
\paragraph{Lower Bound on $(\vy_{\hat{\vU}}^\star)^\top \widetilde{\vU}\vy_{\hat{\vU}}^\star$  }
\begin{align*}
    (\vy_{\hat{\vU}}^\star)^\top \widetilde{\vU}\vy_{\hat{\vU}}^\star \geq \norm{\vy_{\hat{\vU}}^\star}_2^2 \lambda_{\min}(\widetilde{\vU}) \gtrsim (\trace((\vK')^2) + n\norm{\varsigma}_F^c)\sigma_{\max}^{-2}(\vU) 
\end{align*}
We use the fact that $\norm{\vy_{\hat{\vU}}^\star}^2 =\sum_{i=1}^n \norm{\widetilde{\mu}_i}_2^{2c} = \trace((\vK')^2)$.

\paragraph{Lower Bound on $\sum_{i\in [n]}\widetilde{\vU}_{i,i}\Var{\vy_i^\star \cond \vU_i}$}
\begin{align*}
    \sum_{i\in [n]}\widetilde{\vU}_{i,i}\Var{\vy_i^\star \cond \hat{\vU}} \gtrsim & \norm{\varsigma}_F\trace((\vK')^{\frac{2(c-1)}{c}}\widetilde{\vU}) + \norm{\varsigma}_F^{c}\trace(\widetilde{\vU})\\
    \gtrsim&\norm{\varsigma}_F(\trace((\vK')^{\frac{2(c-1)}{c}})\lambda_{\min}(\widetilde{\vU}) + \norm{\varsigma}_F^{c}\trace(\widetilde{\vU}))\\
    \gtrsim&\norm{\varsigma}_F(\trace((\vK')^{\frac{2(c-1)}{c}})\sigma_{\max}^{-2}(\vU) + \norm{\varsigma}_F^{c}\trace(\widetilde{\vU}))
    \end{align*}

\paragraph{Upper Bound on  $\norm{\vK\widetilde{\vU}\vK}_F$}
We  use triangle inequality for the norm,
\begin{align*}
\norm{\vK\widetilde{\vU}\vK}_F \lesssim \norm{\vK'\widetilde{\vU}\vK'}_F + \norm{\varsigma}_F^{c} \norm{\widetilde{\vU}}_F
\end{align*}
For the first term,
\begin{align*}
    \norm{\vK'\widetilde{\vU}\vK'}_F =& \sqrt{\trace(\vK'\widetilde{\vU} (\vK')^2 \widetilde{\vU}\vK')}= \sqrt{\trace((\vK')^4\widetilde{\vU}^2)}\\
    \leq& \sqrt{\trace((\vK')^4)}\sqrt{\trace(\widetilde{\vU}^2)}\leq \sqrt{\trace((\vK')^4)}\norm{\widetilde{\vU}}_F 
\end{align*}
We use Cauchy-Schwartz for Positive Semi-Definite matrices $\vK'$ and $\widetilde{\vU}$ for the inequality and commute the terms inside the trace.
Finally, we use an upper bound on $\norm{\widetilde{\vU}}_F$.
\begin{align*}
    \norm{\widetilde{\vU}}_F = \sqrt{\trace(\widetilde{\vU}^\top \widetilde{\vU}}) = \sqrt{\trace(\widetilde{\vU}^2)} = \sqrt{\sum_{i=1}^n \lambda_i^{2}(\widetilde{\vU})} = \sqrt{\sum_{i=1}^n \sigma_i^{-4}(\vU)} 
\end{align*}
We use the definition of frobenius norm in the first step. Then, we use the fact that $\widetilde{\vU}$ is symmetric. We then use the connection between trace and eigenvalues of a PD symmetric matrix $\widetilde{\vU}$ and the fact that $\lambda_{i}(\widetilde{\vU}) = \sigma_i^{-2}(\vU)$.

\subsubsection{Proof of Lem.~\ref{lem:k_bound}}
\label{sec:k_bound_proof}

From Lem.~\ref{lem:cond_distr}, $\norm{\widetilde{\mu_i}}_2^2$ is a sum of gaussian random variables, hence it is subexponetial and in turn $(1, \cO(\norm{\bI_{m^\star} - \varsigma}_F)$-Sub-Weibull. Raising a Sub-Weibull to the exponent $\frac{\varphi}{2}$ for some $\varphi>0$, we obtain another Sub-Weibull random variable. Therefore, $\norm{\widetilde{\mu_i}}_2^{\varphi}$ is also a $(\frac{\varphi}{2}, \cO(\norm{\bI_{m^\star} - \varsigma}_F^{\frac{\varphi}{2}})$-Sub-Weibull random variable (Def.~\ref{def:sub_weibull}). In addition to this, the Sub-Weibull random variable is a polynomial of sum of sub-Weibull random variables, hence its mean follows the same properties as that of the link function in App.~\ref{sec:link_func}, with its mean bounded by $\cO(\norm{\bI_{m^\star} - \varsigma}_F^{\frac{\varphi}{2}})$.

Since we want to bound the sum of these iid random variables, we can use concentration inequalties for sum of Sub-Weibull random variables. Here, $\trace((\vK')^{\varphi})$ for any $\varphi>0$ can be represented as 
\begin{align*}
    \trace((\vK')^{\varphi})  = \sum_{i=1}^n \norm{\widetilde{\mu}_i}_2^{c\varphi}
\end{align*}
Therefore, we have a sum of $\left(\frac{c\varphi}{2}, \norm{\bI_{m^\star} - \varsigma}_F^{\frac{c\varphi}{2}}\right)$-Sub-Weibull independent random variables. This sum is a $\left(\frac{c\varphi}{2}, n\norm{\bI_{m^\star} - \varsigma}_F^{\frac{c\varphi}{2}}\right)$-Sub-Weibull random variable by triangle inequality of the Orlicz norm, as these are independent random variables. Using Lem.~\ref{lem:sub_weibull_sum}, with probability $1-\delta$, 
\begin{align*}
    &\abs{\trace((\vK')^{\varphi}) - n \E{\norm{\widetilde{\mu}_1}_2^{c\varphi}}} \lesssim n \norm{\bI_{m^\star} - \varsigma}_F^{\frac{c\varphi}{2}} (\log(2/\delta))^{\frac{c\varphi}{2}}\\
    &\trace((\vK')^{\varphi}) \asymp n \norm{\bI_{m^\star} - \varsigma}_F^{\frac{c\varphi}{2}} \asymp n \norm{\bI_{m^\star} - \varsigma}^{\frac{c\varphi}{2}}.
\end{align*}
For $\varphi=0$, $\trace((\vK')^0) = \trace(\bI_n) = n$.

\subsection{Proofs for App.~\ref{sec:proof_flattest_pop}}

\subsubsection{Proof of Thm.~\ref{thm:flattest_good}}
\label{sec:proof_thm_flattest_good}

For \textbf{Single-Index} models defined in Section~\ref{sec:setup}, we consider a good interpolator given by $a_j = \frac{1}{m}, \theta_j = \theta^\star, \forall j\in [m]$. For this interpolator, $\forall \vx\in \bR^{d}$, we have, 
\begin{align*}
 h(\vw, \vx) = \sum_{j=1}^m \frac{1}{m}\sigma(\ip{\theta^\star}{\vx}) = y.  
\end{align*}
Therefore, this is indeed an interpolator, and it is good, as  $F_S(\vw) = F(\vw) = 0$. The flatness of this good interpolator is obtained by plugging in the value of $\va$ into Lem.~\ref{lem:flattest_rescaling}. 
\begin{align*}
    \norm{\va}_{\frac{2c}{c+1}}^{\frac{2c}{c+1}} = \sum_{j=1}^m m^{-\frac{2c}{c+1}} = m^{-\frac{c-1}{c+1}}.
\end{align*}
Therefore, $\Upsilon(\vw) \asymp d^{\frac{c}{c+1}}m^{-\frac{c-1}{c+1}} \asymp\Upsilon^\star$.

Similarly for the \textbf{Sum of Single-Index} case in Section~\ref{sec:setup}, for the given good interpolator., $\forall \vx\in \bR^{d}$, we have, 
\begin{align*}
 h(\vw, \vx) = \frac{m}{m^\star}\sum_{j\in [m^\star]} \frac{a_j^\star m^\star}{m}\sigma(\ip{\theta_j^\star}{\vx}) = \sum_{j\in [m^\star]}a_j^\star \sigma(\ip{\theta_j^\star}{\vx}) = y .
\end{align*}
Therefore, $F_S(\vw) = F(\vw) = 0$ and it is indeed a good interpolator. From Lem.~\ref{lem:flattest_rescaling}, the flatness can be computed from $\va$ as,
\begin{align*}
    \norm{\va}_{\frac{2c}{c+1}}^{\frac{2c}{c+1}} =& \br{\frac{m}{m^\star}}^{-\frac{c-1}{c+1}}\sum_{j\in [m^\star]} \abs{a_j^\star}^{\frac{2c}{c+1}} \lesssim \br{\frac{m}{m^\star}}^{-\frac{c-1}{c+1}} m^\star \br{\max_{j\in [m^\star]}\abs{a_j^\star}}^{\frac{2c}{c+1}}\lesssim m^{-\frac{c-1}{c+1}}
\end{align*}
We use the fact that $m^\star=\cO(1), \abs{a_j^\star}=\cO(1), \forall j\in [m^\star]$. This bound implies that $\Upsilon(\vw) \asymp d^{\frac{c}{c+1}}m^{-\frac{c-1}{c+1}}=\Upsilon^\star$.

\subsubsection{Proof of Prop.~\ref{rem:necessary_flatness}}
\label{sec:rem_necessary_flatness_proof}
For the sufficient condition, i.e., $\norm{\va}_\infty \implies \Upsilon(\vw) \asymp \Upsilon^\star$, we find an upper bound on $\norm{\va}_{\frac{2c}{c+1}}^{\frac{2c}{c+1}}$. Note that,
\begin{align*}
    \norm{\va}_{\frac{2c}{c+1}}^{\frac{2c}{c+1}} \lesssim m m^{-\frac{2c}{c+1}} = m^{-\frac{c-1}{c+1}}.
\end{align*}
Combining this with Lem.~\ref{lem:flattest_rescaling}, we prove the sufficient condition with high probability.

For the necessary condition, we need to show
\begin{align*}
    \Upsilon(\vw) \asymp \Upsilon^\star \implies \norm{\va}_{\frac{2c}{c+1}}^{\frac{2c}{c+1}} \asymp m^{-\frac{c-1}{c+1}}.
\end{align*}

For an interpolator $\vw\in \bR^{m(d+1)}$ with unit-norm inner-layer weights, we have $\vU\va = \vy$. Therefore, the outer-layer weights satisfy the following with high probability. 
\begin{align*}
    \norm{\va}_2 \geq \norm{\va_{\min, \ell_2}}_2 \gtrsim \frac{1}{\sqrt{m}}. 
\end{align*}
Here, we use Lem.~\ref{lem:weak_bound_l2}, which holds for all interpolators. 

We will show that a $\vw$ that is flattest and an interpolator is not possible unless $\norm{\va}_{\infty} =\cO(m^{-1})$.
By contradiction, assume that there is a subset of coordinates $\cQ\subset[m]$ of size $\abs{\cQ} = q$ such that $\abs{a_j} \asymp m^{-b}, \forall j\in \cQ$ for some $b>1$, and for all coordinates $j\notin \cQ$, $\abs{a_j} \leq \cO(m^{-1})$. Note that $\norm{\va}_\infty = m^{-b}$.  If $q = \Theta(m)$, then $\va$ is the flattest interpolator, i.e., $\norm{\va}_{\frac{2c}{c+1}}^{\frac{2c}{c+1}} \lesssim m^{-\frac{c-1}{c+1}}$, iff $b=1$. Therefore, we assume that $q=o(m)$. Since, $n=\Theta(m)$, $q=o(n)$.

Now, note that the vector $\va$ is close to a $q$-sparse vector. Let $\va = \va_{\cQ} + \va_{\cQ^c}$ be its decomposition into two vectors with different magnitude of coordinates. Note thar $\va_{\cQ}$ is $q$-sparse. To ensure interpolation, we require $\vU\va = \vy$, or $\vU \va_{\cQ} = \vy - \vU\va_{\cQ^c} = \widetilde{\vy}$. Additionally, note that $\widetilde{\vy}$ is an $n$-dimensional vector, and $\va_{\cQ}$ is a $q$-sparse vector. We will show that with high probability, we cannot interpolate. For a fixed $\va_{\cQ}$ that is $q$-sparse, 
\begin{align*}
    \Pr{\vU \va_{\cQ} - \widetilde{\vy} = 0} \leq \vartheta^n.
\end{align*}
Note that $\widetilde{\vy}$ has each coordinate independent, and $\vU$ also has independent rows. Further, each of these variables has non-zero mean, and a continuous distribution, so $\vartheta$ is their probability density function at $0$. Due to independence, we multiply it $n$ times.

Now, taking a union bound over set of $q$ sparse-vectors with each coordinate of magnitude $m^{-b}$, we obtain,
\begin{align*}
    \Pr{\vU \va_{\cQ} - \widetilde{\vy} = 0} \leq \exp(-n\log(1/\vartheta) + q\log(m/q)).
\end{align*}
Since, $q=o(n)$, this probability is bounded by $\exp(-\widetilde{C}'n)$ for some constant $\widetilde{C}'>0$. Therefore, for any $b<1$, with high probability, if $\vw$ is flattest, it will not interpolate. This provides a contradiction and thus $b\geq 1$, which provides us with the necessary condition.

\subsection{Flattest minima with inner-layer bias}
\label{sec:wen_comparison_proof}
One of the settings in ~\citep{wen2023how} uses a bias $\{b_j\}_{j\in [m]}$ with the inner layer weights. This increases the dimension of the weights, as $\vw\in \bR^{m(d+2)}$, and the output of the network is  $h(\vw, \vx) = \sum_{j=1}^m a_j \sigma(\ip{\theta_j}{\vx}+b_j), \forall \vx\in \bR^d$. The only change that this causes in Lem.~\ref{lem:flattest_rescaling}, is to replace the terms of $\ip{\theta_j}{\vx}$ in $B(\theta_j)$ by  $\ip{\theta_j}{\vx}+b_j$. From Section~\ref{sec:link_func}, $\sigma, \sigma^2, \sigma'$ are still $(\varrho, K)$-Sub-Weibull, but now their means, variances and Orlicz norm $K$ has increased by a term proportional to $\abs{b_j}^c$. Note that this holds for activations that do not have any piece being $0$, i.e., $c'\neq 0, c''\neq 0$ in Assump.~\ref{assump:activation}. Therefore, to minimize flatness, the optimal choice is to set $b_j = 0$. Further, even for the bounds on $\norm{\va}_{\frac{2c}{c+1}}^{\frac{2c}{c+1}}$ in proofs of Thm.s~\ref{thm:flattest_all}, \ref{thm:flattest_good}, \ref{thm:flattest_bad}, all the Sub-Weibull random variables have an additional term of $\abs{b_j}$ in both their means and Orlicz norms. They are minimized over $b_j$ if $b_j=0$. Therefore, for our setting even with inner-layer bias, the flattest minima are achieved when this inner layer bias is set to $0$.

\end{document}